\documentclass[final]{colt2025}

\title[Generalization error bound for denoising score matching]{Generalization error bound for denoising score matching\\ under relaxed manifold assumption}
\usepackage{times}
\usepackage{bbm}
\usepackage{mathrsfs}
\usepackage{comment}

\renewcommand{\leq}{\leqslant}

\renewcommand{\geq}{\geqslant}

\renewcommand{\tilde}{\widetilde}
\renewcommand{\hat}{\widehat}

\newcommand{\eps}{\varepsilon}

\newcommand{\dd}{{\mathrm{d}}}
\newcommand{\Var}{{\mathrm{Var}}}

\newcommand{\sfp}{\mathsf{p}}
\newcommand{\sfP}{\mathsf{P}}
\newcommand{\sfq}{\mathsf{q}}
\newcommand{\sfQ}{\mathsf{Q}}
\newcommand{\NN}{\mathsf{NN}}

\newcommand{\E}{\mathbb E}
\newcommand{\N}{\mathbb N}
\newcommand{\p}{\mathbb P}
\newcommand{\R}{\mathbb R}
\newcommand{\Z}{\mathbb Z}
\newcommand{\1}{\mathbbm 1}

\newcommand{\sR}{\mathscr R}
\newcommand{\V}{\mathscr V}

\newcommand{\bj}{{\mathbf j}}
\newcommand{\bk}{{\mathbf k}}

\newcommand{\cB}{{\mathcal{B}}}
\newcommand{\cC}{{\mathcal{C}}}
\newcommand{\cE}{{\mathcal{E}}}
\newcommand{\cF}{{\mathcal{F}}}
\newcommand{\cH}{{\mathcal{H}}}
\newcommand{\cK}{{\mathcal{K}}}

\newcommand{\cL}{{\mathcal{L}}}
\newcommand{\cM}{{\mathcal{M}}}
\newcommand{\cN}{{\mathcal{N}}}
\newcommand{\cO}{{\mathcal{O}}}
\newcommand{\cP}{{\mathcal{P}}}
\newcommand{\cQ}{{\mathcal{Q}}}
\newcommand{\cR}{{\mathcal{R}}}
\newcommand{\cS}{{\mathcal{S}}}
\newcommand{\cT}{{\mathcal{T}}}
\newcommand{\cU}{{\mathcal{U}}}
\newcommand{\cV}{{\mathcal{V}}}

\newcommand{\floor}[1]{{\lfloor #1 \rfloor}}
\newcommand{\ceil}[1]{\lceil #1 \rceil}

\def\argmin{\operatornamewithlimits{argmin}}

\newcommand{\relu}{\mathrm{ReLU}}
\newcommand{\sdata}{\sigma_{\mathrm{data}}}
\newcommand{\clip}[2]{\mathrm{clip}_{#2}(#1)}
\newcommand{\integral}[1]{\int\limits_{#1}}

\newtheorem{Th}{Theorem}[section]
\newtheorem{Lem}[Th]{Lemma}
\newtheorem{Def}[Th]{Definition}
\newtheorem{Prop}[Th]{Proposition}
\newtheorem{Co}[Th]{Corollary}
\newtheorem{Rem}[Th]{Remark}

\newtheorem{As}[Th]{Assumption}

\coltauthor{%
    \Name{Konstantin Yakovlev} \Email{kdyakovlev@hse.ru}\\
    \addr HSE University, Russian Federation
    \AND
    \Name{Nikita Puchkin} \Email{npuchkin@hse.ru}\\
    \addr HSE University, Russian Federation%
}

\begin{document}

\maketitle

\begin{abstract}%
    We examine theoretical properties of the denoising score matching estimate. We model the density of observations with a nonparametric Gaussian mixture. We significantly relax the standard manifold assumption allowing the samples step away from the manifold. At the same time, we are still able to leverage a nice distribution structure. We derive non-asymptotic bounds on the approximation and generalization errors of the denoising score matching estimate. The rates of convergence are determined by the intrinsic dimension. Furthermore, our bounds remain valid even if we allow the ambient dimension grow polynomially with the sample size.
\end{abstract}


\section{Introduction}

Denoising diffusion probabilistic models \citep{song2019generative, song2020score} provide a state-of-the art tool for generating high-quality
data including images, audio and video synthesis \citep{dhariwal2021diffusion, ho2022video, kong2021diffwave}. It is based on an insight from random processes theory that the Ornstein-Uhlenbeck process
\begin{equation}
    \label{eq:ornstein-uhlenbeck}
    \dd X_t = -X_t \, \dd t + \sqrt{2} \, \dd W_t, \quad 0 \leq t \leq T,
\end{equation}
with an initial condition $X_0 \sim \sfp_0^*$ admits an inverse one
\begin{equation}
    \label{eq:inverse_process}
    \dd Z_t = \left(Z_t + 2 \nabla \log \sfp_{T - t}^*(Z_t) \right) \dd t + \sqrt{2} \, \dd B_t,
    \quad Z_0 \sim \sfp_T^*.
\end{equation}
Here $W_t$ and $B_t$ are independent Wiener processes in $\R^D$ and $\sfp_t^*$ is the density of $X_t$, $t \in [0, T]$. Since the diffusion \eqref{eq:ornstein-uhlenbeck} converges to the standard Gaussian distribution $\cN(0, I_D)$ quite fast (see, e.g., \citep[Section 4.1]{bakry14}), one can model the initial density $\sfp_0^*$ running the inverse process \eqref{eq:inverse_process} with $Z_0 \sim \cN(0, I_D)$. The problem is that the \emph{score function}
\[
    s^*(y, t) = \nabla \log \sfp_t^*(y),
    \quad (y, t) \in \R^D \times [0, T],
\]
is unknown and should be estimated from i.i.d. samples $Y_1, \dots, Y_n \sim \sfp_0^*$. In practice, one usually replaces \eqref{eq:inverse_process} with
\begin{equation}
\label{eq:inverse_process_normal_start}
        \dd \hat Z_t = \left(\hat Z_t + 2 \, \hat s(\hat Z_t, T - t) \right) \dd t + \sqrt{2} \, \dd B_t,
    \quad \hat Z_0 \sim \cN(0, I_D),
\end{equation}
where $\hat s(y, t)$ is an estimate of $s^*(y, t)$. Recently, a lot of researchers (for instance, \cite{bortoli21, de2022convergence, lee2023convergence, chen2023improved, chen2023the, chen2023sampling, benton2024nearly, li2024adapting} to name a few) considered generative diffusion models through the lens of Markov processes theory. As one should have expected, the rates of convergence of the density of $\hat Z_{T - t_0}$, where $t_0 > 0$ is a fixed number referred to as \emph{stopping time}, to the target distribution obtained in those papers heavily depend on the accuracy of estimation of $s^*(y, t)$.
For this reason, the problem of score estimation attracted attention of many statisticians.

In \citep{sriperumbudur17, wibisono2024optimal, zhang2024minimax}, the authors tackled this problem via widespread tools from nonparametric statistics. Namely, \cite{sriperumbudur17} considered an infinite-dimensional exponential family of probability densities, which are parametrized by functions in a reproducing kernel Hilbert space, while \cite{wibisono2024optimal} and \cite{zhang2024minimax} used kernel smoothing. On the other hand, \cite{oko2023diffusion} studied theoretical properties of the denoising score matching estimate (see Section \ref{sec:preliminaries} for the definition), which is often used in practice. The authors took a class of feed-forward neural networks with ReLU activations as a class of admissible scores. Under the condition that $\sfp_0^*$ is supported on the cube $[-1, 1]^D$ and bounded away from zero on this set, \cite{oko2023diffusion} derived approximation and generalization error bounds. Unfortunately, the rates of convergence in \cite{sriperumbudur17, wibisono2024optimal, zhang2024minimax, oko2023diffusion} deteriorate extremely fast as the dimension $D$ grows and are not applicable in real-world scenarios. 

At the same time, in various applications the data distribution has a nice low-dimensional structure \citep{bengio2013representation, pope2021the} despite its high-dimensional representation. For this reason, one can hope for more optimistic rates of convergence depending on the effective dimension rather than on the ambient one. For instance, \cite{chen2023score} attempted to  escape the curse of dimensionality imposing additional assumptions on the underlying density $\sfp_0^*$. They showed that the risk of the denoising score matching estimate with probability at least $1 - 1/n$ does not exceed
\[
    \cO\left( \frac1{t_0} \left( n^{-2 / (d + 5)} + D n^{-(d + 3) / (d + 5)}  \right) \mathrm{polylog}(\log D, \log(1 / t_0), \log n) \right).
\]
Here $t_0$ is a stopping time and $d$ is an intrinsic dimension which may be much smaller than $D$. However, this upper bound estimate (as well as the one in the subsequent work \citep{boffi2024shallow}) was obtained under very restrictive assumption that $\sfp_0^*$ is supported on a linear $d$-dimensional subspace. 
In the recent papers \citep{tang2024adaptivity} and \citep{azangulov2024convergence} the authors considered a much more general and challenging setup when the distribution of $Y_1, \dots, Y_n$ is supported on a smooth low-dimensional manifold $\cM$. They reported that the excess risk of their estimates converges to zero at the rates depending on the manifold smoothness, smoothness of $\sfp^*_0$, and the dimension of $\cM$. However, the works \citep{tang2024adaptivity, azangulov2024convergence} have several significant issues.
\begin{itemize}
    \item In \citep{tang2024adaptivity} the hidden constants in the rates of convergence depend exponentially on the ambient dimension $D$. This makes the derived rates of convergence useless when $D$ is of order $\log n$. In practice, $D$ is usually much larger than $\log n$.
    \item The estimate of \cite{azangulov2024convergence} exploits the fact that $Y_1, \dots, Y_n$ lie exactly on the manifold $\cM$. However, this assumption seems to be unrealistic. In contrary, in several papers \citep{pmlr-v235-daras24a, daras2024ambient, kawar2024gsurebased, weiminbai2024an} the authors note that adding moderate noise to initial samples improves quality of image generation and prevents mode collapse. 
    \item Both \cite{tang2024adaptivity} and \cite{azangulov2024convergence} assume that $\cM$ has no boundary and a positive reach. Furthermore, they suppose that the density $\sfp_0^*$ (with respect to the volume measure on $\cM$) is bounded away from zero and infinity. In \citep[p. 2]{zhang2024minimax}, the authors write: ``The density lower bound greatly simplifies the proof of the score estimation error bound; however, it excludes natural distribution classes, such as multi-modal distributions or mixtures with well-separated components.''
    \item Both \cite{tang2024adaptivity} and \cite{azangulov2024convergence} refer to \citep[Theorem C.4]{oko2023diffusion} when derive the estimation error bound. However, Theorem C.4 in \citep{oko2023diffusion} has a critical flaw. To be more precise, the last implication in (69) on page 41 does not hold.
    Hence, the rates of convergence in all the papers \citep{oko2023diffusion, tang2024adaptivity, azangulov2024convergence} are incorrect.
\end{itemize}

\medskip
\noindent
\textbf{Our Contribution.} \quad
In this paper, we attempt to overcome the  aforementioned drawbacks in the existing results on score estimation.

\begin{itemize}
    \item We suggest a statistical model with relaxed manifold assumption. In contrast to \cite{tang2024adaptivity, azangulov2024convergence}, we do not require the samples $Y_1, \dots, Y_n \sim \sfp_0^*$ to lie exactly on a low-dimensional manifold. Instead, we assume that
    \[
        Y_i = g^*(U_i) + \sdata \; \xi_i,
        \quad 1 \leq i \leq n,
    \]
    where $\sdata \geq 0$, $g^* : [0, 1]^d \rightarrow \R^D$ is an unknown continuous map and $U_1, \dots, U_n \sim \mathrm{Un}([0, 1]^d)$ and $\xi_1, \dots, \xi_n \sim \cN(0, I_D)$ are independent random elements. This model admits situations when the image of $g^*$ has zero reach and when the density $g^*(U_1)$ (with respect to the volume measure on the image of $g^*$) is not bounded away from zero or infinity.
    \item 
    We show that the denoising score matching estimate enjoys a rate of convergence $\cO(n^{-2\beta / (2\beta + d)})$ (up to some logarithmic factors) depending on the intrinsic dimension $d$ and the smoothness of $g^*$. We also carefully track how the hidden constant behind $\cO(\cdot)$ depends on the ambient dimension $D$ and stopping time $t_0$.
\end{itemize}

\medskip

\noindent
\textbf{Notation.}
\quad
Throughout the paper, $\Z_+$ stands for the set of non-negative integers. For any $\beta > 0$, $\lfloor \beta \rfloor$ denotes the largest integer strictly less than $\beta$. For a multi-index $\bk = (k_1, \dots, k_r) \in \Z_+^r$ and a vector $v = (v_1, \dots, v_r) \in \R^r$, we define $v^\bk = v_1^{k_1} v_2^{k_2} \dots v_r^{k_r}$ and $|\bk| = k_1 + k_2 + \ldots + k_r$. Multi-indices are always displayed in bold. For any $R > 0$, we denote a centered Euclidean ball of radius $R$ by $\cB(0, R)$. For any two sets $A, B \subset \R^r$ and any $c \in \R$, we introduce
\[
    A \oplus B = \{a + b : a \in A, b \in B\}
    \quad \text{and} \quad
    c A = \{ c a: a \in A\}.
\]
The notation $f \lesssim g$ and $g \gtrsim f$ means that $f = \cO(g)$. If $f \lesssim g$ and $g \lesssim f$, we simply write $f \asymp g$.  Besides, we often replace $\max\{a, b\}$ and $\min\{a, b\}$ by shorter expressions $a \vee b$ and $a \wedge b$, respectively.

\medskip

\noindent
\textbf{Paper structure.}
\quad
The rest of the paper is organized as follows. In Section \ref{sec:preliminaries}, we introduce necessary definitions and notations. In Section \ref{sec:main_results}, we present main results of the present paper and provide comparison with concurrent work. Some open questions and directions for future work are discussed in Section \ref{sec:conclusion}. Finally, we summarize key ideas of the proofs of main results (Theorem \ref{th:score_approximation} and Theorem \ref{thm:estim_enh_zero}) in Section \ref{sec:proof_sketches} while rigorous derivations are deferred to appendix.

\section{Preliminaries and notations}
\label{sec:preliminaries}

\noindent
\textbf{Denoising score matching.}
\quad
Score matching approach aims to minimize
\begin{equation}
    \label{eq:sm_objective}
    \int\limits_{t_0}^T \E_{X_t} \left\| s(X_t, t) - s^*(X_t, t) \right\|^2 \dd t,
\end{equation}
where $X_t$ obeys the Ornstein-Uhlenbeck process \eqref{eq:ornstein-uhlenbeck},
over a class $\cS$ of admissible score functions. The parameter $t_0 \geq 0$ is called \emph{stopping time}. Since the score $s^*(X_t, t)$ is unknown, \cite{vincent2011connection} suggested to replace the objective \eqref{eq:sm_objective} by $\E_{X_0} \ell(s, X_0)$, where
\[
    \ell(s, X_0) = \int\limits_{t_0}^T \E \left[ \left\| s(X_t, t) - \nabla_{X_t} \log \sfp_t^*(X_t \,\vert\, X_0) \right\|^2 \, \big\vert \, X_0 \right] \dd t
\]
and $\sfp_t^*(X_t \,\vert\, X_0)$ stands for the conditional density of $X_t$ given $X_0$.  In contrast to \eqref{eq:sm_objective}, the loss function $\ell(s, X_0)$ admits an explicit expression. Indeed, due to the properties of the Ornstein-Uhlenbeck process, we have that
\begin{equation}
    \label{eq:cond_distr}
    (X_t \,\vert\, X_0) \sim \cN(m_t X_0, \sigma_t^2 I_D),
    \quad \text{where $m_t = e^{-t}$ and $\sigma_t^2 = 1 - e^{-2t}$.}
\end{equation}
Then it is straightforward to observe that
\[
    \ell(s, X_0) = \int\limits_{t_0}^T \E \left[ \left\| s(X_t, t) + \frac{X_t - m_t X_0}{\sigma_t^2} \right\|^2 \, \Big\vert \, X_0 \right] \dd t.
\]
The intuition of \cite{vincent2011connection} is based on a simple observation that for any $t > 0$ and any $s : \R^D \times [t_0, T] \rightarrow \R$ we have
\[
    \E_{X_t} \|s(X_t, t) - s^*(X_t, t)\|^2
    = \E_{X_0} \, \E_{X_t} \left[ \|s(X_t, t) - \nabla_{X_t} \log \sfp_t^*(X_t \,\vert\, X_0) \|^2 \, \big\vert \, X_0 \right] + C,
\]
where $C$ does not depend on $s$. In particular, this yields that
\begin{align}
\label{eq:vincent_int_t0_T}
    \int\limits_{t_0}^T \E_{X_t} \left\| s(X_t, t) - s^*(X_t, t) \right\|^2 \dd t
    = \E_{X_0} \ell(s, X_0) - \E_{X_0} \ell(s^*, X_0).
\end{align}
for all $s : \R^D \times [t_0, T] \rightarrow \R$. Given i.i.d. samples $Y_1, \dots, Y_n \sim \sfp_0^*$, the denoising score matching estimate is defined as an empirical risk minimizer
\begin{equation}\label{eq:emp_risk_min}
    \hat s \in \argmin\limits_{s \in \cS} \left\{ \frac1n \sum\limits_{i = 1}^n \ell(s, Y_i) \right\}.
\end{equation}

\medskip

\noindent
\textbf{Norms.}
\quad Throughout the paper, we denote the Euclidean norm of a vector $v$ by $\|v\|$. We also use the notation 
$\|v\|_\infty$ for the maximal absolute value of its entries, while $\|v\|_0$ stands for the number of non-zero entries of $v$.
For a matrix $A$ and a tensor $\cT$ of order $k$, their operator norms are defined as
\[
    \|A\| = \sup\limits_{\|u\| = \|v\| = 1} u^\top A v
    \quad \text{and} \quad
    \|\cT\| = \sup\limits_{\|u_1\| = \ldots = \|u_k\| = 1} \left\{ \sum\limits_{i_1, \dots, i_k} \cT_{i_1, \dots, i_k} u_{1, i_i} \dots u_{k, i_k} \right\}.
\]
Similarly, $\|A\|_\infty$ and $\|A\|_0$ stand for the maximal absolute value of entries of $A$ and the number of its non-zero entries, respectively. Finally, for a vector-valued function $f$ defined on a set $\Omega$, we denote
\begin{align*}
    &
    \|f\|_{L^\infty(\Omega)} = \sup_{x\in \Omega} \|f(x)\|
    \quad \text{and} \quad
    \|f\|_{L^p(\Omega)} =  \left\{ \int_{\Omega} \|f(x)\|^p \, \dd x \right\}^{1/p},
    \quad
    p \geq 1.
\end{align*}

\medskip

\noindent
\textbf{Smoothness classes.}
\quad
Let $f: \Omega \mapsto \R$ be an arbitrary function defined on a set $\Omega \subseteq \R^{r}$. For a multi-index $\bk = (k_1, \dots, k_r) \in 
\Z^r_+$, we define the corresponding partial derivative $\partial^\bk f$ as
\[
    \partial^{\bk} f(x) = \frac{\partial^{|\bk|} f}{\partial x_1^{k_1} \ldots \partial x_r^{k_r}}.
\]
Given $\beta > 0$ and $H > 0$, say that $f$ belongs to a H\"older class $\cH^\beta(\Omega, \R, H)$ if and only if
\[
    \max\limits_{\substack{\bk \in \Z^r_+\\ 1 \leq |\bk| \leq \lfloor\beta\rfloor}} \left\|\partial^\bk f\right\|_{L^\infty(\Omega)} \leq H
    \quad \text{and} \quad
    \max\limits_{\substack{\bk \in \Z^r_+\\ |\bk| = \lfloor\beta\rfloor}} \sup\limits_{\substack{x, y \in \Omega\\ x \neq y}} \; \frac{|\partial^\bk f(x) - \partial^\bk f(y)|}{\min\{1, \|x - y\|_\infty\}^{\beta - \lfloor\beta\rfloor}} \leq H.
\]
We also say that a vector-valued function $h : \Omega \rightarrow \R^m$ lies in a H\"older class $\cH^\beta(\Omega, \R^m, H)$ if and only if every component of $h$ is in $\cH^\beta(\Omega, \R, H)$.

\medskip

\noindent
\textbf{Neural networks.}
\quad
In the present paper, we focus on feed-forward neural networks with the activation function $\relu(x) = x \vee 0$. For a  vector $b = (b_1, \dots, b_r) \in \R^r$,
we define the shifted activation function $\relu_b: \R^r \rightarrow \R^r$ as 
\begin{equation*}
    \sigma_{b}: (x_1, \dots, x_r) \mapsto \bigl( \relu(x_1 - b_1), \dots, \relu(x_r - b_r) \bigr).
\end{equation*}
Given a positive integer $L$ and a vector $W = (W_0, W_1, \dots, W_{L}) \in \mathbb N^{L+1}$,  a neural network of depth $L$ and architecture $W$ is a function $f: \R^{W_0} \rightarrow \R^{W_L}$ of the form
\begin{equation}
    \label{eq:nn}
    f(x) = -b_L + A_L \circ \relu_{b_{L-1}} \circ A_{L-1} \circ \relu_{b_{L-1}} \circ \ldots \circ A_2 \circ \relu_{b_1} \circ A_1 x,
\end{equation}
where $A_j \in \R^{W_j \times W_{j - 1}}$ and $b_j \in \R^{p_j}$ for all $j \in \{1, \dots, L\}$. The maximum number of neurons in one layer $\|W\|_\infty$ is referred to as the width of the neural network. In what follows, we consider classes $\NN(L, W, S, B)$ of neural networks of the form \eqref{eq:nn} with at most $S$ non-zero weights and the weight magnitude $B$:
\begin{align*}
    \NN(L, W, S, B)
    = \Big\{ \text{$f$ is of the form \eqref{eq:nn}} :
    &
    \sum\limits_{j = 1}^L \left( \|A_j\|_0 + \|b_j\|_0 \right) \leq S
    \\&
    \text{and } \max\limits_{1 \leq j \leq L} \left\{ \|A_j\|_\infty \vee \|b_j\|_\infty \right\} \leq B \Big\}.
\end{align*}
In our proofs, we will extensively use the results on  concatenation and parallel stacking of neural networks described in \cite[Section B.1.1]{nakada20}.

\section{Main results}
\label{sec:main_results}

In this section, we present main results of our paper. Before we move to upper bounds on the accuracy of score approximation and estimation, let us elaborate on data distribution assumptions.

\begin{As}\label{as:data_distr}
    Given a \emph{generator} from the Hölder ball $g^* \in \cH^\beta([0, 1]^d, \R^D, H)$, $\|g^*\|_{L^\infty([0, 1]^d)} \leq 1$, and $\sdata \in [0, 1)$, the observed samples $Y_1, \dots, Y_n$ are i.i.d. copies of a random element $X_0 \in \R^D$ generated from the model
    \begin{align*}
        X_0 = g^*(U) + \sdata \; \xi,
    \end{align*}
    where $U \sim \mathrm{Un}([0, 1]^d)$ and $\xi \sim \cN(0, I_D)$ are independent.
\end{As}

In the case $\sdata > 0$, $X_0$ has a density with respect to the Lebesgue measure in $\R^D$ given by
\begin{align*}
    \sfp_0^*(y) = (\sqrt{2\pi}\sdata)^{-D}\integral{[0,1]^d}\exp\left(-\frac{\|y - g^*(u)\|^2}{2\sdata^2}\right) \dd u, \quad y \in \R^D.
\end{align*}

Assumption \ref{as:data_distr} suggests that the observations occupy a vicinity of a low-dimensional surface $\text{Im}(g^*)$. However, we allow  $Y_1, \dots, Y_n$ to slightly deviate from $\text{Im}(g^*)$ adding Gaussian noise $\sdata \xi$. This not only reflects common real-world scenarios but also corresponds to the situations when the noise is added manually to move from inherently discrete to absolutely continuous distributions (like in the dequantization trick, see, for example, \citep{dinh2017density, ho2019flow++}).
On the other hand, Assumption \ref{as:data_distr} ensures that the distribution of $Y_i$'s has a small entropic dimension (see, for instance, \citep[Section 2]{dudley68} and \citep[Definition 4]{chakraborty2024statistical}). This is a reason for rates of convergence depending on the intrinsic dimension $d$, rather than on the ambient one (see Theorem \ref{thm:estim_enh_zero} below).
In addition, Assumption \ref{as:data_distr} encompasses the cases when the data distribution has multiple modes, which is typical to real-world data, supported by empirical studies (see, for example, \cite{khayatkhoei2018disconnected, brown2023verifying}).
It is worth mentioning that the distribution of $g^*(U)$ has atoms with respect to the volume measure on the image of $g^*$ when $g^*$ is constant on a set of positive measure. This feature alleviates the need for lower and upper bound assumptions on the density of $g^*(U)$ (with respect to the volume measure) commonly used in several papers (see, for instance, \citep{oko2023diffusion, tang2024adaptivity, gatmiry2024learning}).
Furthermore, unlike prior works studying properties of generative diffusion models in the presence of a hidden low-dimensional manifold \citep{tang2024adaptivity, azangulov2024convergence}, we do not require the image of $g^*$ to have a positive reach. The difference becomes more evident if one takes into account that \cite{azangulov2024convergence} suppose that the reach of the underlying manifold $\cM$ is not just positive but also large enough (see their Assumption C(iii)). This, together with the condition $\cM \subset \cB(0, 1)$, puts significant restrictions on the shape of $\cM$.
Finally, we emphasize that the assumption $\sdata \leq 1$ is reasonable, as it, in conjunction with the bound $\|g^*\|_{L^\infty([0, 1]^d)} \leq 1$ on generator's $L^\infty$-norm, ensures a well-controlled signal-to-noise ratio, which is crucial for meaningful data analysis.

Assumption \ref{as:data_distr} and the conditional distribution property \eqref{eq:cond_distr} ensure that, for any $t > 0$, the density $\sfp_t^*$ along the forward process \eqref{eq:ornstein-uhlenbeck} is expressed as
\begin{equation}
    \label{eq:pt_density}
    \sfp_t^*(y) = (\sqrt{2\pi}\tilde{\sigma}_t)^{-D} \integral{[0, 1]^d} \exp\left\{-\frac{\|y - m_t g^*(u)\|^2}{2\tilde{\sigma}_t^2}\right\} \dd u, \quad \tilde\sigma_t^2 = m_t^2\sdata^2 + \sigma_t^2.
\end{equation}
Hence, the corresponding score function is given by
\begin{equation}
    \label{eq:true_score}
    s^*(y, t)
    = \nabla_y \log \sfp_t^*(y)
    =-\frac{y}{\tilde{\sigma}_t^2} + \frac{m_t}{\tilde{\sigma}_t^2}f^*(y, t),
\end{equation}
where 
\begin{align}\label{eq:f_star}
        f^*(y, t) = \frac{\integral{[0, 1]^d}g^*(u)\exp\left(-\frac{\|y - m_t g^*(u)\|^2}{2\tilde{\sigma}_t^2}\right) \dd u}{\integral{[0, 1]^d}\exp\left(-\frac{\|y - m_t g^*(u)\|^2}{2\tilde{\sigma}_t^2}\right)\dd u}.
\end{align}
It is easy to see that $f^*(y, t)$ is uniformly bounded.
Formally, Assumption \ref{as:data_distr} suggests that for any $t > 0$
\begin{align*}
    \|f(\cdot, t)\|_{L^\infty(\R^D)} \leq \|g^*\|_{L^\infty([0, 1]^d)} \leq 1.
\end{align*}
These findings about the score function structure  motivate us to consider the following class of score estimators.

\begin{Def}[the class of score estimators]\label{def:score_class}
    The class of neural score estimators $\cS(L, W, S, B)$ is defined as
    \begin{align}
        \cS(L, W, S, B)
        = \biggl\{s(y, t) := -\frac{y}{m_t^2\sigma^2 + \sigma_t^2} + \frac{m_t\clip{f(y, t)}{2}}{m_t^2\sigma^2 + \sigma_t^2}:
        \qquad
        &\notag
        \\
        f \in \NN(L, W, S, B), \; \sigma \in [0, 1) \biggr\},
        &
    \end{align}
\end{Def}
where, for any $z \in \R^D$, $\clip{z}{R}$ stands for componentwise clipping of $z$ at the level $R$:
\[
    \clip{z}{R} =
    \begin{cases}
        z, \quad \text{if $\|z\| \leq R$},\\
        \frac{R z}{\|z\|}, \quad \text{otherwise}.
    \end{cases}
\]
The use of componentwise clipping is justified as it does not limit the application of gradient-based learning methods due to the non-differentiable nature of the operation. This is also the case for neural  networks that use the ReLU activation function. Distinct from the approach in \citep{chen2023score}, our definition of the score estimator class in Definition \ref{def:score_class} refrains from imposing extra Lipschitz constraints, which is a significant relaxation in assumptions. Furthermore, unlike \cite{tang2024adaptivity}  necessitating uniform output boundedness across the entire estimator class, we do not impose such a restriction, thereby enhancing flexibility.

\subsection{Score approximation}

We move to main results of the paper. 
In this section we provide a quantitative expressive power of neural network class from Definition \ref{def:score_class} to approximate the true score function.

\begin{Th}[Approximation of the true score function]
    \label{th:score_approximation}
    Assume that $g^* \in \cH^\beta([0, 1]^d, \R^D, H)$ and let $s^*(y, t) = \nabla \log \sfp_t^*(y)$ be the corresponding score function. Fix an arbitrary $\eps \in (0, 1)$ such that
    \[
        D \eps \sqrt{\log(1 / \eps)} \leq \tilde{\sigma}_{t_0}^2,
        \quad H \eps \sqrt{D} \leq \tilde\sigma_{t_0},
        \quad \text{and} \quad
        \frac{H d^{\lfloor\beta\rfloor} \eps^\beta \sqrt{D}}{\lfloor\beta\rfloor!} \leq 1 \wedge \tilde\sigma_{t_0}.
    \]
    Then there exists a score function $s \in \cS(L, W, S, B)$ (see Definition \ref{def:score_class}) with $L \lesssim D^2 + \log^4(1/\eps)$, $\log B
    \lesssim D + \log^2(1/\eps)$, and
    \begin{align*}
        &
        \|W\|_{\infty}
        \lesssim D^2 \eps^{-d} \left(\frac{1}{t_0 + \sdata^2}\vee 1\right)\left(D + \log^2(1/\eps)\right)^3,
        \\&
        S \lesssim D^2 \eps^{-d} \left(\frac{1}{t_0 + \sdata^2}\vee 1\right)\left(D + \log^2(1/\eps)\right)^3 + D \eps^{-d} \left(D + \log^2\frac1{\eps} \right)^{2\binom{d + \lfloor\beta\rfloor}{d} + 5}
    \end{align*}
    such that
    \begin{align}
        \label{eq:score_approx_main_res_enhanced}
        \int\limits_{t_0}^T \E_{X_t} \left\|s^*(X_t, t) - s(X_t, t) \right\|^2 \, \dd t
        \lesssim \frac{D \eps^{2\beta}}{\tilde{\sigma}_{t_0}^2}.
    \end{align}
    Here $\lesssim$ stands for an inequality up to a multiplicative constant depending on $d$ and $\beta$.
\end{Th}
We provide a rigorous proof of Theorem \ref{th:score_approximation} in Appendix \ref{sec:score_approximation_proof} and a proof sketch in Section \ref{sec:th_score_approximation_proof_sketch}. In contrast to the literature on score estimation under the manifold hypothesis, such as \citep{tang2024adaptivity, azangulov2024convergence}, we do not need the density of $g^*(U)$, $U \sim \mathrm{Un}([0, 1]^d)$, (with respect to the volume measure on $\text{Im}(g^*)$) be bounded away from zero (but, of course, we still have to impose some regularity assumptions on $g^*$). This significant difference prevents straightforward adoption of the arguments presented in \citep{oko2023diffusion, tang2024adaptivity, azangulov2024convergence}, leading to a more intricate proof. Nonetheless, we can exploit properties of the data distribution specified in Assumption \ref{as:data_distr}.

A comparison with the work of \cite{azangulov2024convergence} reveals that an extension of their method to handle positive $\sdata$ is not straightforward.
This is because the construction of their estimator involves a pre-processing step, which requires the target distribution to be exactly supported on the manifold. Moreover, Assumption D in \citep{azangulov2024convergence} requires the smoothness parameter of the manifold to be large enough. In contrast, the proposed approach addresses these challenges, effectively handling positive $\sdata$ values and small values of $\beta$ without imposing such restrictive conditions.

Another important feature of Theorem \ref{th:score_approximation} distinguishing our result from \citep{zhang2024minimax, tang2024adaptivity}
is polynomial dependence on the ambient dimension $D$ of both the right-hand side of \eqref{eq:score_approx_main_res_enhanced} and the parameters $L$, $W$, $S$, and $\log B$. For instance, the bound of \cite{zhang2024minimax} (Corollary 3.7) explicitly depends on $t_0^{-D/2}$. In \citep{tang2024adaptivity}, the exponential dependence on $D$ is potentially hidden in the coefficients $a_{lki}$ (see pp. 21--22 in \citep{tang2024adaptivity}) as they hide multinomial coefficients which may be as large as $\cO(e^D)$. \cite{tang2024adaptivity}  do not clarify whether one can get better upper bounds on $a_{lki}$'s than $\cO(e^D)$. The same concerns the coefficients $a_{l_1, l_2, k, s, i}$ on page 26 in \citep{tang2024adaptivity}.

\subsection{Score estimation}
\label{sec:score_estimation}

It this section, we present a sample complexity bound for the score estimator derived through the empirical risk minimization \eqref{eq:emp_risk_min}, considering the specified configuration of the score class $\cS(L, W, S, B)$ (see Definition \ref{def:score_class}).

\begin{Th}
\label{thm:estim_enh_zero}
    Assume that the conditions of Theorem \ref{th:score_approximation} hold.
    Let also $T \geq 1$ and let the sample size be sufficiently large, that is, it fulfils
    \begin{align}
        \label{eq:sample_size_zero_lb}
        \sigma_{t_0}^2 n \geq \left(\frac{D\sqrt{\log (n\sigma_{t_0}^2)}}{\sigma_{t_0}^2\sqrt{2\beta + d}}\right)^{2\beta + d}
        \vee \left(\frac{H d^{\floor{\beta}}\sqrt{D}}{\floor{\beta}! \, \sigma_{t_0}}\right)^{2 + d / \beta}
        \vee \left( \frac{H \sqrt{D}}{\sigma_{t_0}} \right)^{2\beta + d}.
    \end{align}
    Then, for any $\delta \in (0, 1)$, with probability at least $(1 - \delta)$ the excess risk of an empirical risk minimizer \eqref{eq:emp_risk_min} over the class $\cS(L, W, S, B)$ (see Definition \ref{def:score_class}) with 
    \begin{align*}
        L &\lesssim D^2\log^4 n, \quad \|W\|_\infty \lesssim t_0^{-1} D^5(n\sigma_{t_0}^2)^{\frac{d}{2\beta + d}}\log^6 n, \\
        S &\lesssim t_0^{-1} D^{6 + 2\binom{d + \floor{\beta}}{d}}(n\sigma_{t_0}^2)^{\frac{d}{2\beta + d}}(\log n)^{10 + 4\binom{d + \floor{\beta}}{d}}, \quad \log B \lesssim D \log^2 n
    \end{align*}
    satisfies the inequality
    \begin{align*}
        \integral{t_0}^T\E_{X_t}\|\hat{s}(X_t, t) - s^*(X_t, t)\|^2 \, \dd t 
        \lesssim \frac{T^2 D^{12 + 2\binom{d + \floor{\beta}}{d}}}{\sigma_{t_0}^2}(n\sigma_{t_0}^2)^{-\frac{2\beta}{2\beta + d}} L(t_0, n) \log(4 / \delta),
    \end{align*}
    where
    \begin{align*}
        L(t_0, n) = (\log n)^{20 + 4\binom{d + \floor{\beta}}{d}} \log T \log D \log^3(1 / t_0).
    \end{align*}
    The hidden constant behind $\lesssim$ depends on $d$ and $\beta$ only. 
\end{Th}

We provide a complete proof of Theorem \ref{thm:estim_enh_zero} in Appendix \ref{sec:thm_estim_enh_zero_proof} and its sketch in Section \ref{sec:thm_estim_enh_zero_sketch}.
We would like to emphasize that the requirement that $T \geq 1$ is mild, since the approximation of the reversed process outlined in \eqref{eq:inverse_process_normal_start} is tight when $T$ is sufficiently large. In the case $\beta \geq 1$ our rate of convergence $\cO(n^{-2\beta / (2\beta + d)} \text{polylog}(n))$ decays faster than the minimax optimal generalization error bound $\cO(n^{-2 / (4 + d)})$ of \cite{wibisono2024optimal} (see their Theorems 1 and 3). This is not surprising, since \cite{wibisono2024optimal} studies minimax optimal rates over the class of sub-Gaussian Lipschitz scores $s^*$ while we consider a more special case.
As we mentioned in the introduction, the concurrent papers \citep{tang2024adaptivity, azangulov2024convergence} inherit the mistake of \cite{oko2023diffusion} in the proof of Theorem C.4 (see eq. (69) on page 41).
In particular, using the bounds $|\ell_j(x) - \ell^\circ(x)| \leq C_\ell$ and $\E_x \big[ \ell_j(x) - \ell^\circ(x) \big] \leq r_j^2$, \cite*{oko2023diffusion} mistakenly conclude that
\[
    \E_x \big( \ell_j(x) - \ell^\circ(x) \big)^2 \leq C_\ell \, r_j^2.
\]
The last inequality would be true if the difference of losses $\ell_j(x) - \ell^\circ(x)$ was non-negative almost surely. Unfortunately, this is not the case for the denoising score matching loss.
Hence, our paper provides the first upper bound on the generalization error of the denoising score matching estimate under rather general assumptions.
In our proof, we show that the excess losses $\ell(s, X_0) - \ell(s^*, X_0)$, $s \in \cS$, satisfy the Bernstein condition (see \eqref{eq:bernst_cond_positive_sdata}): 
\[
    \E_{X_0}(\ell(s, X_0) - \ell(s^*, X_0))^2
    \lesssim \left(\frac{D^3 T^2  \log(1 / \sigma_{t_0}^2) \log n}{\sigma_{t_0}^2}\right) \big( \E_{X_0}[\ell(s, X_0) - \ell(s^*, X_0)] \big)^{1 - 1 / \varkappa},
\]
where $\varkappa = 2 \vee \big( \log n + \log(\sigma_{t_0}^{-2}) \big)$. We would like to note that first factor in the right-hand side is
$\cO(\sigma_{t_0}^{-2})$, rather than $\cO(1)$. The linear dependence on $\sigma_{t_0}^{-2}$ affects the results on distribution estimation in the Kantorovich and total variation distances.
Besides, we carefully track dependence on the ambient dimension $D$ (in contrast to \citep{tang2024adaptivity}, where the authors ignore the terms which potentially may be as large as $\cO(e^D)$). Finally, our results remain valid under the relaxed manifold assumption (Assumption \ref{as:data_distr}), while the estimate of \cite{azangulov2024convergence} suffers from noise in the observations $Y_1, \dots, Y_n$.

Let us move to the linear case and compare the result of Theorem \ref{thm:estim_enh_zero} with the one of \cite{chen2023score}.
We set $\delta = n^{-1}$, $\sdata = 0$, $\beta = 1$ and take $t_0 \leq 1$, which leads to $\sigma_{t_0}^2 \asymp t_0$.
Thus, the generalization error bound from Theorem \ref{thm:estim_enh_zero} simplifies to
\begin{align*}
    \integral{t_0}^T \E_{X_t}\|\hat s(X_t, t) - s^*(X_t, t)\|^2 \dd t \lesssim \frac{T^2 D^{12 + 2\binom{d + \floor{\beta}}{d}}}{t_0}
(n\sigma_{t_0}^2)^{-\frac{2\beta}{2\beta + d}} L(t_0, n) \log(4 / \delta),
\end{align*}
with probability at least $1 - 1/n$.
With the same parameters specified, \cite{chen2023score} claims that, disregarding  logarithmic factors in all parameters excluding $n$, with the same confidence level one has
\begin{align*}
    \frac{1}{T - t_0}\integral{t_0}^T \E_{X_t} \|\hat{s}(X_t, t) - s^*(X_t, t)\|^2 \dd t
    \lesssim \frac{1}{t_0}\left(n^{-\frac{2 - 2d\log(\log n) / \log n}{d + 5}} + D n^{-\frac{d + 3}{d + 5}}\right)\log^3 n,
\end{align*}
From the comparison of the rates above we deduce that our estimate enjoys faster rate in terms of the sample size. It is not surprising that the dependence on the ambient dimension and the stopping time is worse, since \cite{chen2023score} impose severe assumptions on the data distribution. This allows them to consider scores of a special kind, while we must deal with a more general class. When comparing our work to that of \cite{wibisono2024optimal}, a critical distinction lies in the strong dependence on the ambient dimension.
Specifically, Theorem 5 \citep{wibisono2024optimal} suggests that the score estimation error scales as $(\log n)^{D / 2}$.
This exponential-like growth in error with respect to the ambient dimension severely undermines consistency of the estimate when $D \asymp \log n$.
In contrast, our bound given in Theorem \ref{thm:estim_enh_zero} exhibits a milder dependence on the ambient dimension, successfully addressing the aforementioned issue.

\subsection{Distribution Estimation}

This section focuses on estimating the true data distribution with density $\sfp^*_0$ by leveraging the score function $\widehat{s}$ learned from data.
For the sake of simplicity, we do not take a discretization error into consideration.
To evaluate a total variation (TV) distance between $X_0 \sim \sfp^*_0$ and $\widehat{Z}_{T - t_0}$ specified in the modified backward process \eqref{eq:inverse_process_normal_start}, we derive the following significant finding, which is an implication of our estimation theory encapsulated in Theorem \ref{thm:estim_enh_zero} and an auxiliary lemma from \citep{chen2023sampling}.

\begin{Th}
\label{lem:tv_bound}
    Assume that $\sdata > 0$ and suppose that the conditions of Theorem \ref{thm:estim_enh_zero} hold for $t_0 = \sdata^{(2\beta + d) / (3\beta + d)} n^{-\beta / (3\beta + d)}$ and
    $T \asymp \log n$.
    Then, for any $\delta \in (0, 1)$, with probability at least $1 - \delta$, the following holds:
    \begin{align*}
        \mathrm{TV}(\hat{Z}_{T - t_0}, X_0)
        \lesssim  D^{6 + \binom{d + \floor{\beta}}{d}} \sdata^{-\frac{1}{3\beta + d}} n^{-\frac{\beta}{6\beta + 2d}} (\log n)^{25 + 4\binom{d + \floor{\beta}}{d}} \log D \log^{1/2}(4 / \delta).
    \end{align*}
\end{Th}
The proof of Theorem \ref{lem:tv_bound} is postponed to Appendix \ref{sec:tv_bound_proof}.
Theorem \ref{lem:tv_bound} establishes that the target distribution can be accurately learned with a polynomial number of samples. This represents a significant enhancement over prior work by \cite{gatmiry2024learning}, which only guaranteed a quasi-polynomial sample complexity. Notably, while our approach necessitates higher-order smoothness conditions on the support of the distribution of means, \cite{gatmiry2024learning} achieve their results under more relaxed assumptions without such smoothness constraints. When comparing our result presented in Theorem \ref{lem:tv_bound} to those of \cite{wibisono2024optimal, zhang2024minimax}, a key distinction emerges regarding sensitivity to the ambient dimensionality.
More precisely, Corollary 6 in \citep{wibisono2024optimal} states that the sample complexity scales as $D^{D / 2}$.
Comparable limitations are evident in \cite[Theorem 3.8]{zhang2024minimax}.
By following their proof, we have found that the hidden degree of the polynomial logarithmic term is proportional to the ambient dimension.
These dependencies impose notable constraints when the dimension is logarithmic in the sample size, a scenario discussed in Section \ref{sec:score_estimation}.

\section{A note on minimax rates of convergence}
\label{sec:conclusion}

In conclusion, we would like to briefly discuss minimax optimal rates of convergence for score estimation under Assumption \ref{as:data_distr}. We are going to show that this question is far from being trivial. In Appendix \ref{sec:lem_analytic_proof}, we prove the following lemma.

\begin{Lem}
    \label{lem:analytic}
    Given an arbitrary $\sigma > 0$ and $g : [0, 1]^d \rightarrow \R^D$, let
    \[
        \sfp(y) = (\sqrt{2\pi} \sigma)^{-D} \int\limits_{[0, 1]^d} \exp\left\{ -\frac{\|y - g(u)\|^2}{2 \sigma^2} \right\} \dd u.
    \]
    Then, for any $k \in \N$, it holds that
    \[
        \sup\limits_{y \in \R^D} \left\| \nabla^{k} \left(\log \sfp(y) - \frac{\|y\|^2}{2 \sigma^2} \right) \right\| \leq \frac{2^{k - 1} (k - 1)!}{\sigma^{2k}} \max\limits_{u \in [0, 1]^d} \|g(u)\|^{k}.
    \]
\end{Lem}
Applying Lemma \ref{lem:analytic} with $g(u) = g^*(u)$ and $\sigma = \tilde\sigma_t$, $t_0 \leq t \leq T_0$, we observe that the function $\log \sfp_t^*(y)$ is $(Q, R)$-analytic with $Q \lesssim 1$ and $R \lesssim \tilde \sigma_t^{-2}$ (see, for example, \citep[Definition 1]{belomestny2023simultaneous} for the definition of $(Q, R)$-analytic functions). Choosing an appropriate system of functions (splines, wavelets, Hermite polynomials, etc.), one can construct an estimate $\tilde s(y, t)$ such that
\begin{equation}
    \label{eq:analytic_function_estimation_upper_bound}
    \E_{X_t} \left\| \tilde s(X_t, t) - s^*(X_t, t) \right\|^2 \lesssim \frac{\mathrm{polylog}(n)}n
    \quad \text{for any $t \in [t_0, T]$.}
\end{equation}
However, the hidden constant is likely to depend on $D$ exponentially, so the obtained upper bound becomes vacuous when $D \gtrsim \log n$. On the other hand, if one tries to prove a minimax lower bound of order $\Omega(n^{-2\beta / (2 \beta + d)})$ (complementary to the result of Theorem \ref{thm:estim_enh_zero}), he must take into account that the rate $\cO(n^{-2\beta / (2 \beta + d)})$ may be minimax optimal only in the case $D \gtrsim \log n$. Otherwise, one should expect different rates of convergence. For instance, according to \eqref{eq:analytic_function_estimation_upper_bound}, in the case $D = \cO(1)$ one can prove an upper bound $\cO(\mathrm{polylog}(n) / n)$. This poses a great challenge in deriving optimal rates of convergence in the minimax sense.

\section{Proof sketches of main results}
\label{sec:proof_sketches}

In this section we elaborate on main ideas used in the proofs of Theorems \ref{th:score_approximation} and \ref{thm:estim_enh_zero}. A reader can find rigorous derivations in Appendices \ref{sec:score_approximation_proof} and \ref{sec:thm_estim_enh_zero_proof}, respectively.

\subsection{Proof sketch of Theorem \ref{th:score_approximation}}
\label{sec:th_score_approximation_proof_sketch}

We split the proof into several steps for convenience.

\medskip
\noindent
\textbf{Step 1: local polynomial approximation.}\quad
We begin our analysis by noting that it is sufficient for our purposes to approximate a surrogate score function $s^\circ$ induced by a local polynomial approximation of $g^*$.
This is essential for our subsequent steps.
Our technical findings reveal that
\begin{align*}
    \int\limits_{t_0}^T \E_{X_t \sim \sfp_t^*} \|s^\circ(X_t, t) - s^*(X_t, t)\|^2
    \leq \frac{D H^2 \eps^{2\beta}}{4 (\sdata^2 + e^{2 t_0} - 1)} \left( \frac{ d^{\lfloor\beta\rfloor}}{\lfloor\beta\rfloor!} \right)^2.
\end{align*}

\medskip
\noindent
\textbf{Step 2: reduction to approximation on a compact set.}
\quad
Representing $ s^\circ(y, t) = -\frac{y}{\tilde\sigma_t^2} + \frac{m_t f^\circ(y, t)}{\tilde\sigma_t^2}$
and leveraging the property that the distribution $\sfp_t^*$ is light-tailed, we deduce that
\begin{align*}
    \int\limits_{t_0}^T \E_{X_t} \|s^\circ(X_t, t) - s(X_t, t)\|^2 \, \dd t
    \leq
    \integral{t_0}^T \frac{m_t^2}{\tilde\sigma_t^4}\left(D\eps^{2\beta} + \integral{\cK_t}\|f^\circ(y, t) - f(y, t)\|^2\sfp_t^*(y) \, \dd y\right) \, \dd t,
\end{align*}
where $\cK_t$ denotes a compact set containing points that are close to the image of $g^*$.

\medskip
\noindent
\textbf{Step 3: $f^\circ$ is a composition of simpler functions.}\quad 
We first recall that 
\[
    f^\circ(y, t) = 
    {\scriptstyle
    \left(\sum\limits_{\bj \in \{1, \dots, N\}^d} \; \int\limits_{\cU_\bj} g^\circ_\bj(u) \exp\left\{ -\frac{\|y - m_t g^\circ_\bj(u)\|^2}{2\tilde{\sigma}_t^2} \right\} \dd u \right) \Big/ \left( \sum\limits_{\bj \in \{1, \dots, N\}^d} \; \int\limits_{\cU_\bj} \exp\left\{ -\frac{\|y - m_t g^\circ_\bj(u)\|^2}{2\tilde{\sigma}_t^2} \right\} \dd u \right)}.
\]
We next note that the expression under the exponent could be expressed as follows:
\begin{equation*}
    \frac{\|y - m_t g^\circ_\bj(u)\|^2}{2 \tilde\sigma_t^2}
    = V_{\bj, 0}(y, t) + \cV(t) \left\|g^\circ_\bj(u) - g^*(u_\bj)\right\|^2 + \sum\limits_{\substack{\bk \in \Z^d_+ \\ 1 \leq |\bk| \leq \lfloor\beta\rfloor}} V_{\bj, \bk}(y, t) \, \frac{(u - u_\bj)^{\bk}}{\bk!},
\end{equation*}
where we introduced intermediate functions
\begin{equation*}
    \cV(t) = \frac{m_t^2}{2 \tilde\sigma_t^2},
    \qquad
    V_{\bj, 0}(y, t) = \frac{\|y - m_t g^\circ_\bj(u_\bj)\|^2}{2 \tilde\sigma_t^2},
\end{equation*}
and
\begin{equation*}
    V_{\bj, \bk}(y, t) = -\frac{m_t}{\tilde\sigma_t^2} \left. \partial_u^\bk \big((y - m_t g^*(u_\bj))^\top g^*(u) \big) \right|_{u = u_\bj},
    \quad \bk \in \Z^d_+, 1 \leq |\bk| \leq \lfloor\beta\rfloor.
\end{equation*}
Thus, each term in both the numerator and the denominator of $f^\circ$ representation is a composition of the introduced functions
with functions of $\binom{d + \lfloor\beta\rfloor}{d} + 1$ variables.
It is important to note that $\binom{d + \lfloor\beta\rfloor}{d} + 1$ can be substantially smaller than the ambient dimension $D$,
implying a reduction in the overall complexity of the resulting neural network.

\medskip
\noindent
\textbf{Step 4: approximation of $V_{\bj, 0}$, $V_{\bj, \bk}$ and $\cV$.}\quad
We remark that the functions under consideration can be approximated to $\varepsilon$-accuracy with a number of parameters that scales logarithmically with $1/\varepsilon$.
This logarithmic scaling has a positive impact on the complexity of the resulting network architecture.

\medskip
\noindent
\textbf{Step 5: approximation of the composition.}\quad
We first approximate
\begin{equation*}
    \int\limits_{\cU_\bj} \exp\left\{ -\frac{\|y - m_t g^\circ_\bj(u)\|^2}{2\tilde{\sigma}_t^2} \right\} \dd u
    \quad \text{and} \quad
    \int\limits_{\cU_\bj} g^\circ_\bj(u) \exp\left\{ -\frac{\|y - m_t g^\circ_\bj(u)\|^2}{2\tilde{\sigma}_t^2} \right\} \dd u.
\end{equation*}
The insight from Step 3 regarding the functions with a small number of arguments facilitates the application of a fundamental result on the approximation capabilities of ReLU neural networks \cite[Theorem 5]{schmidt-hieber20}.
Second, the approximation of the numerator and denominator in the expression for $f^\circ$ poses no additional difficulty, as a summation operation introduces an additional linear layer.

\medskip
\noindent
\textbf{Step 6: division approximation.}\quad 
It remains to approximate the division operation in order to complete the approximation of $f^\circ$.
It is crucial to ensure that both the numerator and denominator are approximated with sufficient \emph{relative} precision.
To accomplish this without relying on specific structural assumptions, such as a lower-bounded density (as used in \cite{tang2024adaptivity, azangulov2024convergence}),
we establish a lower bound for the denominator that is independent of such constraints.
This key step allows us to apply a division approximation result and complete the proof.

\subsection{Proof sketch of Theorem \ref{thm:estim_enh_zero}}
\label{sec:thm_estim_enh_zero_sketch}

The proof of Theorem \ref{thm:estim_enh_zero} consists of several steps.

\medskip

\noindent
\textbf{Step 1: Bernstein's condition.}\quad
We start with the observation that the Bernstein condition for the excess loss class could be easily verified.
Formally, for any $s \in \cS(L, W, S, B)$, it holds that
\begin{align*}
    &\E_{X_0}(\ell(s, X_0) - \ell(s^*, X_0))^2
    \\&
    \lesssim \left(\frac{D^2(T - t_0)(1 + \varkappa)\log(1 / \sigma_{t_0}^2)}{\sigma_{t_0}^2}\right)^{1 + 1 /\varkappa} \left\{\E_{X_0}[\ell(s, X_0) - \ell(s^*, X_0)]\right\}^{1 - 1 / \varkappa},
\end{align*}
where $\varkappa \geq 1$ and will be determined later in the proof.
The subsequent proof leverages Bernstein's inequality in conjunction with the $\eps$-net argument. To enhance clarity, we split the proof into several steps.

\medskip
\noindent
\textbf{Step 1: Bernstein's large deviation bound.} \quad
Given a fixed $s \in \cS(L, W, S, B)$, we invoke Bernstein's concentration inequality for unbounded random variables \cite{lecue2012oracle}.
This is feasible due to the fact that $\ell(s, X_0) - \ell(s^*, X_0)$ has finite $\psi_1$-norm.
Therefore, with probability at least $(1 -\delta / 2)$,
\begin{align*}
    \notag
    &\left| \E_{X_0} [\ell(s, X_0) - \ell(s^*, X_0)] - \widehat \E_{X_0}[\ell(s, X_0) - \ell(s^*, X_0)] \right| \\
    &
    \lesssim \sqrt{\frac{C_{\mathrm{b}}\{\E_{X_0}[\ell(s, X_0) - \ell(s^*, X_0)]\}^{1 - 1 / \varkappa}\log(4 / \delta)}{n}} + \frac{C_{\mathrm{b}}\log(4 / \delta)}{n},
\end{align*}
where $C_{\mathrm{b}}$ denotes the constant in the aforementioned Bernstein-type inequality, given by
\begin{align*}
    C_{\mathrm{b}} =  \frac{D^3T^2\log^2(\sigma_{t_0}^{-2}) \log n}{\sigma_{t_0}^{2 + 2 / \varkappa}} .
\end{align*}

\medskip
\noindent
\textbf{Step 2: $\eps$-net argument and a uniform bound.}\quad
For any $\tau > 0$ we form a set $\cS_\tau \subseteq \cS(L, W, S, B)$, such that
\begin{align*}
    \sup_{s \in \cS(L, W, S, B)}\inf_{s_\tau \in \cS_\tau} \left\{ |\E_{X_0}[\ell(s, X_0) - \ell(s_\tau, X_0)]| + |\widehat{\E}_{X_0}[\ell(s, X_0) - \ell(s_\tau, X_0)]| \right\} \leq \tau,
\end{align*}
with probability at least $(1 - \delta)$.
In addition, the following bound holds:
\begin{align*}
    \log |\cS_\tau| \lesssim SL \log(\tau^{-1} L(\|W\|_\infty + 1)(B \vee 1) D T \sigma_{t_0}^{-2}\log(n / \delta)) .
\end{align*}
This observation yields a uniform bound, guaranteeing that with probability at least $(1 - \delta)$,
\begin{align}
    \label{eq:sketch_unif_bound}
    &\notag
    \left| \E_{X_0}[\ell(s, X_0) - \ell(s^*, X_0)] - \widehat \E_{X_0}[\ell(s, X_0) - \ell(s^*, X_0)]  \right| \\
    &\quad \lesssim \tau^{1 - 1 / \varkappa} + \sqrt{\frac{C_{\mathrm{b}}\{\E_{X_0}[\ell(s, X_0) - \ell(s^*, X_0)]\}^{1 - 1 / \varkappa}\log(4 |\cS_\tau| / \delta)}{n}} + \frac{C_{\mathrm{b}}\log(4 |\cS_\tau| / \delta)}{n}
\end{align}
holds for all $s \in \cS(L, W, S, B)$ simultaneously.

\medskip
\noindent
\textbf{Step 3: final bound for $\widehat s$.}
\quad
Let $\bar{s}$ be the approximation of the true score from Theorem \ref{th:score_approximation} formulated for accuracy parameter $\eps \in (0, 1)$ satisfying the conditions of the theorem.
Then, for the empirical risk minimizer $\hat{s}$ from the uniform bound \eqref{eq:sketch_unif_bound}, we deduce that with probability at least $(1 - \delta)$,
\begin{align*}
    \widehat{\E}_{X_0}[\ell(\widehat{s}, X_0) - \ell(s^*, X_0)]
    &\leq \widehat{\E}_{X_0}[\ell(\bar{s}, X_0) - \ell(s^*, X_0)] \\
    &\lesssim \tau^{1 - 1 / \varkappa} + \frac{C_{\mathrm{b}}\log(4 |\cS_\tau| / \delta)}{n} + \left(\frac{D\eps^{2\beta}}{\sigma_{t_0}^2}\right)^{1 - 1/ \varkappa},
\end{align*}
Subsequently, combining this result with \eqref{eq:sketch_unif_bound} and setting $\tau = \eps^{2\beta}$, we conclude that with probability at least $(1 - \delta)$,
\begin{align*}
    \integral{t_0}^T\E_{X_t}\|\hat{s}(X_t, t) - s^*(X_t, t)\|^2 \, \dd t
    \lesssim \left( \left(\frac{D\eps^{2\beta}}{\sigma_{t_0}^2}\right)^{1 - 1 / \varkappa} + \frac{T^2 \eps^{-d} D^{12 + 2\binom{d + \floor{\beta}}{d}} }{t_0 \cdot (\sigma_{t_0}^2 n)^{1 / (1 + 1 / \varkappa)}} \right) L'(t_0, \eps) \log(4 / \delta) ,
\end{align*}
where the logarithmic factors are captured in the expression
\begin{align*}
    L'(t_0, \eps) = (\log(1 / \eps))^{18 + 4\binom{d + \floor{\beta}}{d}} \log T \log D \log^3(1 / t_0) \log^2 n.
\end{align*}
Then the choice $\eps = (n\sigma_{t_0}^2)^{-\frac{1}{2\beta + d}}$ and $\varkappa = 2 \vee ( \log n + \log(\sigma_{t_0}^{-2}))$ yields the desired result.

\acks{This work was supported by the
Ministry of Economic Development of the Russian Federation (code 25-139-66879-1-0003).}

\bibliography{biblio.bib}

\newpage

\tableofcontents

\newpage

\appendix

\section{Proof of Theorem \ref{th:score_approximation}}
\label{sec:score_approximation_proof}

The proof of Theorem \ref{th:score_approximation} is quite technical. For this reason, we split it into several steps.

\medskip

\noindent
\textbf{Step 1: local polynomial approximation.}
\quad
We start with a simple observation that it is enough to approximate a surrogate score function $s^\circ$ induced by a local polynomial approximation of $g^*$. This will play a crucial role on further steps. Let us introduce $N = \lceil 1/\eps \rceil$, and for any $\bj = (j_1, \dots, j_d) \in \{1, \dots, N\}^d$ we define
\[
    u_\bj = \frac{\bj}N
    \quad \text{and} \quad
    \cU_\bj = \left[ \frac{j_1 - 1}N, \frac{j_1}N \right] \times  \left[ \frac{j_2 - 1}N, \frac{j_2}N \right] \times \dots \times \left[ \frac{j_d - 1}N, \frac{j_d}N \right].
\]
Then the local polynomial approximation of $g^*$ is given by
\begin{equation}
    \label{eq:g_circ}
    g^\circ(u) = \sum\limits_{\bj \in \{1, \dots, N\}^d} g^\circ_\bj(u), \quad u \in [0, 1]^d,
\end{equation}
where
\begin{equation}
    \label{eq:gj_circ}
    g^\circ_\bj(u) = \sum\limits_{\substack{\bk \in \Z^d_+ \\ |\bk| \leq \lfloor\beta\rfloor}} \frac{\partial^\bk g^*(u_\bj)}{\bk!} (u - u_\bj)^{\bk} \; \1(u \in \cU_\bj), \quad \text{for all $\bj \in \{1, \dots, N\}^d$ and $u \in [0, 1]^d$.}
\end{equation}
It is straightforward to show that $g^\circ$ does not differ from $g^*$ too much. We provide an explicit quantitative bound in the following lemma.

\begin{Lem}
    \label{lem:g_approximation}
    Let $g^* \in \cH^\beta([0, 1]^d, \R^D, H)$ and let $g^\circ$ be as defined in \eqref{eq:g_circ}. Then it holds that
    \[
        \left\| g^* - g^\circ \right\|_{L^\infty([0, 1]^d)}
        = \max\limits_{u \in [0, 1]^d} \left\| g^*(u) - g^\circ(u) \right\|
        \leq \frac{H d^{\lfloor\beta\rfloor} \eps^\beta \sqrt{D}}{\lfloor\beta\rfloor!}.
    \]
\end{Lem}
The proof of Lemma \ref{lem:g_approximation} is postponed to Appendix \ref{sec:lem_g_approximation_proof}. Our next goal is to show that the closeness of $g^*$ and $g^\circ$ implies the proximity of corresponding score functions. Similarly to $s^*(y, t)$ (see \eqref{eq:true_score}), we denote
\[
    s^\circ(y, t) = \nabla_y \log \int\limits_{[0, 1]^d} \exp\left\{ -\frac{\|y - m_t g^\circ(u)\|^2}{2\tilde{\sigma}_t^2} \right\} \dd u.
\]
Then Proposition \ref{prop:azangulov} ensures that
\begin{align}
    \label{eq:s_circ_s_star_difference}
    \int\limits_{t_0}^T \int\limits_{\R^D} \|s^\circ(y, t) - s^*(y, t)\|^2 \, \sfp_t^*(y) \, \dd y \, \dd t
    &\notag
    \leq \frac{m_{t_0}^2 }{4 \tilde\sigma_{t_0}^2} \left\| g^* - g^\circ \right\|_{L^\infty([0, 1]^d)}^2
    \\&
    \leq \frac{D H^2 \eps^{2\beta}}{4 (\sdata^2 + e^{2 t_0} - 1)} \left( \frac{ d^{\lfloor\beta\rfloor}}{\lfloor\beta\rfloor!} \right)^2.
\end{align}
In the rest of the proof, we focus on approximation of $s^\circ(y, t)$.

\medskip

\noindent
\textbf{Step 2: reduction to approximation on a compact set.}
\quad
We can represent the surrogate score $s^\circ(y, t)$ in the following form:
\[
    s^\circ(y, t)
    = \nabla_y \log \int\limits_{[0, 1]^d} \exp\left\{ -\frac{\|y - m_t g^\circ(u)\|^2}{2\tilde{\sigma}_t^2} \right\} \dd u
    = -\frac{y}{\tilde\sigma_t^2} + \frac{m_t f^\circ(y, t)}{\tilde\sigma_t^2},
\]
where
\begin{equation}
    \label{eq:f_circ}
    f^\circ(y, t) =
    {\scriptstyle
    \left(\int\limits_{[0, 1]^d} g^\circ(u) \exp\left\{ -\frac{\|y - m_t g^\circ(u)\|^2}{2\tilde{\sigma}_t^2} \right\} \dd u \right) \Big/ \left( \int\limits_{[0, 1]^d} \exp\left\{ -\frac{\|y - m_t g^\circ(u)\|^2}{2\tilde{\sigma}_t^2} \right\} \dd u \right)}.
\end{equation}
Note that the conditions of the theorem and Lemma \ref{lem:g_approximation} imply that
\[
    \|g^\circ\|_{L^\infty([0, 1]^d)}
    \leq \|g^*\|_{L^\infty([0, 1]^d)} + \|g^\circ - g^*\|_{L^\infty([0, 1]^d)}
    \leq 1 + \frac{H d^{\lfloor\beta\rfloor} \eps^\beta \sqrt{D}}{\lfloor\beta\rfloor!}
    \leq 2.
\]
This means that $f^\circ$ takes its values in the Euclidean ball $\cB(0, 2)$ and $f^\circ(y, t) = \clip{f^\circ(y, t)}{2}$. Hence, if we manage to find a neural network $f(y, t) \in \NN(L, W, S, B)$ that approximates $f^\circ(y, t)$, then the score function
\[
    s(y, t) = -\frac{y}{\tilde\sigma_t^2} + \frac{m_t}{\tilde\sigma_t^2} \clip{f(y, t)}{2}
\]
belongs to $\cS(L, W, S, B)$ and
\begin{align*}
    \int\limits_{t_0}^T \int\limits_{\R^D} \|s^\circ(y, t) - s(y, t)\|^2 \, \sfp_t^*(y) \, \dd y \, \dd t
    &
    = \int\limits_{t_0}^T \int\limits_{\R^D} \frac{m_t^2}{\tilde\sigma_t^4} \|\clip{f^\circ(y, t)}{2} - \clip{f(y, t)}{2}\|^2 \, \sfp_t^*(y) \, \dd y \, \dd t
    \\&
    \leq \int\limits_{t_0}^T \int\limits_{\R^D} \frac{m_t^2}{\tilde\sigma_t^4} \left( 16 \wedge \|f^\circ(y, t) - f(y, t)\|^2 \right) \sfp_t^*(y) \, \dd y \, \dd t.
\end{align*}

Let $R_t > 0$ be a parameter to be defined a bit later. One can decompose the integral with respect to $y$ into the sum of two integrals over
\begin{align}
    \label{eq:k_m_t_def}
    \cK_t = \left\{ y \in \R^D \, : \, \min_{u \in [0, 1]^d} \|y - m_t g^*(u) \| \leq R_t \right\}
\end{align}
and its complement. Then, for any $t > 0$ it holds that
\begin{align*}
    &
    \int\limits_{\R^D} \left( 16 \wedge \|f^\circ(y, t) - f(y, t)\|^2 \right) \sfp_t^*(y) \, \dd y
    \\&
    \leq \int\limits_{\cK_t} \|f^\circ(y, t) - f(y, t)\|^2 \, \sfp_t^*(y) \, \dd y + 16 \int\limits_{\R^D \backslash \cK_t} \sfp_t^*(y) \, \dd y.
\end{align*}
The next lemma shows that the latter term in the right-hand side is negligible. 
\begin{Lem}
    \label{lem:pt_tails}
    Fix an arbitrary $t \in [t_0, T]$ and let $\cK_t \subset \R^D$ be as defined above in \eqref{eq:k_m_t_def}. Then the density $\sfp_t^*$ given by \eqref{eq:pt_density} satisfies
    \[
        \int\limits_{\R^D \backslash \cK_t} \sfp_t^*(y) \, \dd y
        \leq \exp\left\{-\frac1{16} \left( \frac{R_t^2 - D \tilde\sigma_t^2}{D \tilde\sigma_t^2} \wedge \frac{\sqrt{R_t^2 - D \tilde\sigma_t^2}}{\tilde\sigma_t} \right) \right\}.
    \]
\end{Lem}
The proof of Lemma \ref{lem:pt_tails} is moved to Appendix \ref{sec:lem_pt_tails_proof}. Setting
\begin{align}\label{eq:rt_def}
    R_t = \tilde\sigma_t \sqrt{D} + 16 \tilde\sigma_t \left( \sqrt{D \log\left( \frac{\eps^{-2\beta}}D \right)} \vee \log\left( \frac{\eps^{-2\beta}}D \right) \right) 
\end{align}
we obtain that
\begin{align}
    \label{eq:s_to_f_approx}
    &\notag
    \int\limits_{t_0}^T \int\limits_{\R^D} \|s^\circ(y, t) - s(y, t)\|^2 \, \sfp_t^*(y) \, \dd y \, \dd t
    \\&
    \leq D \eps^{2\beta} \int\limits_{t_0}^T \frac{m_t^2}{\tilde\sigma_t^4} \dd t
    + \int\limits_{t_0}^T \int\limits_{\cK_t} \frac{m_t^2}{\tilde\sigma_t^4} \|f^\circ(y, t) - f(y, t)\|^2 \sfp_t^*(y) \, \dd y \, \dd t.
\end{align}
    
Hence, the problem of score approximation reduces to approximation of the function $f^\circ(y, t)$ on a compact set
\begin{align}\label{eq:c_star_def}
    \cC_{[t_0, T]}^* = \left\{ (y, t) \in \R^{D} \times [t_0, T] : y \in \cK_t \right\},
\end{align}
where $\cK_t$ defined in \eqref{eq:k_m_t_def} is taken with
\[
    R_t = \tilde\sigma_t \sqrt{D} + 16 \tilde\sigma_t \left( \sqrt{D \log\left( \frac{\eps^{-2\beta}}D \right)} \vee \log\left( \frac{\eps^{-2\beta}}D \right) \right).
\]

\medskip

\noindent
\textbf{Step 3: $f^\circ$ is a composition of simpler functions.}
\quad
The main challenge in approximation of $f^\circ$ is that it has a form of a fraction (see \eqref{eq:f_circ}), where the denominator
\[
    \int\limits_{[0, 1]^d} \exp\left\{ -\frac{\|y - m_t g^\circ(u)\|^2}{2\tilde{\sigma}_t^2} \right\} \dd u
\]
is not bounded away from zero. In contrast to other papers on score estimation under the manifold hypothesis (for instance, \citep{tang2024adaptivity, azangulov2024convergence}), we do not require the smallest eigenvalue of $\nabla g^\circ(u)^\top \nabla g^\circ(u)$ to be bounded away from zero. This means that the density of $g^\circ(U)$, $U \sim \mathrm{Un}([0, 1]^d)$, with respect to the volume measure on $\text{Im}(g^\circ)$ may be unbounded. For this reason, we cannot rely on the argument of \citep{oko2023diffusion, tang2024adaptivity, azangulov2024convergence}, and this complicates the proof significantly. Nevertheless, $f^\circ$ still has a nice structure we are going to exploit. To be more precise, we will represent $f^\circ$ as a composition of simpler functions. Using \eqref{eq:g_circ} and \eqref{eq:gj_circ}, we rewrite \eqref{eq:f_circ} in the form
\[
    f^\circ(y, t) =
    {\scriptstyle
    \left(\sum\limits_{\bj \in \{1, \dots, N\}^d} \int\limits_{\cU_\bj} g^\circ_\bj(u) \exp\left\{ -\frac{\|y - m_t g^\circ_\bj(u)\|^2}{2\tilde{\sigma}_t^2} \right\} \dd u \right) \Big/ \left( \sum\limits_{\bj \in \{1, \dots, N\}^d} \int\limits_{\cU_\bj} \exp\left\{ -\frac{\|y - m_t g^\circ_\bj(u)\|^2}{2\tilde{\sigma}_t^2} \right\} \dd u \right)}.
\]
Let us fix an arbitrary $\bj \in \{1, \dots, N\}^d$ and consider
\begin{equation}
    \label{eq:jth_summands}
    \int\limits_{\cU_\bj} \exp\left\{ -\frac{\|y - m_t g^\circ_\bj(u)\|^2}{2\tilde{\sigma}_t^2} \right\} \dd u
    \quad \text{and} \quad
    \int\limits_{\cU_\bj} g^\circ_\bj(u) \exp\left\{ -\frac{\|y - m_t g^\circ_\bj(u)\|^2}{2\tilde{\sigma}_t^2} \right\} \dd u.
\end{equation}
Due to the definition of $g_\bj^\circ$, it holds that 
\begin{align*}
    \frac{\|y - m_t g^\circ_\bj(u)\|^2}{2 \tilde\sigma_t^2}
    &
    = \frac{\|y - m_t g^\circ_\bj(u_\bj)\|^2}{2 \tilde\sigma_t^2} + \frac{m_t^2 \|g^\circ_\bj(u) - g^\circ_\bj(u_\bj)\|^2}{2 \tilde\sigma_t^2}
    \\&\quad
    - \frac{m_t (y - m_t g^\circ_\bj(u_\bj))^\top (g^\circ_\bj(u) - g^\circ_\bj(u_\bj))}{\tilde\sigma_t^2}
    \\&
    = \frac{\|y - m_t g^*(u_\bj)\|^2}{2 \tilde\sigma_t^2} + \frac{m_t^2}{2 \tilde\sigma_t^2} \, \left\|g^\circ_\bj(u) - g^*(u_\bj)\right\|^2
    \\&\quad
    - \frac{m_t}{ \tilde\sigma_t^2} \sum\limits_{\substack{\bk \in \Z^d_+ \\ 1 \leq |\bk| \leq \lfloor\beta\rfloor}} \frac{(u - u_\bj)^{\bk}}{\bk!} \left. \partial_u^\bk \big((y - m_t g^*(u_\bj))^\top g^*(u) \big) \right|_{u = u_\bj}.
\end{align*}
Introducing
\begin{equation}
    \label{eq:v_j0}
    \cV(t) = \frac{m_t^2}{2 \tilde\sigma_t^2},
    \qquad
    V_{\bj, 0}(y, t) = \frac{\|y - m_t g^\circ_\bj(u_\bj)\|^2}{2 \tilde\sigma_t^2},
\end{equation}
and
\begin{equation}
    \label{eq:v_jk}
    V_{\bj, \bk}(y, t) = -\frac{m_t}{\tilde\sigma_t^2} \left. \partial_u^\bk \big((y - m_t g^*(u_\bj))^\top g^*(u) \big) \right|_{u = u_\bj},
    \quad \bk \in \Z^d_+, 1 \leq |\bk| \leq \lfloor\beta\rfloor,
\end{equation}
we observe that
\begin{equation}
    \label{eq:change_of_variables}
    \frac{\|y - m_t g^\circ_\bj(u)\|^2}{2 \tilde\sigma_t^2}
    = V_{\bj, 0}(y, t) + \cV(t) \left\|g^\circ_\bj(u) - g^*(u_\bj)\right\|^2 + \sum\limits_{\substack{\bk \in \Z^d_+ \\ 1 \leq |\bk| \leq \lfloor\beta\rfloor}} V_{\bj, \bk}(y, t) \, \frac{(u - u_\bj)^{\bk}}{\bk!}.
\end{equation}
For any $\bj \in \{1, \dots, N\}^d$, let $\V_\bj$ stand for a vector-valued function with components $V_{\bj, \bk}$, $\bk \in \Z^d_+$, $1 \leq |k| \leq \lfloor\beta\rfloor$, and $\cV$:
\begin{equation}
    \label{eq:vj}
    \V_\bj(y, t) = \left( \big(V_{\bj, \bk} : \bk \in \Z^d_+, 1 \leq |\bk| \leq \lfloor\beta\rfloor \big), \cV(t)\right)^\top \in \R^{\binom{d + \lfloor\beta\rfloor}{d}}.
\end{equation}
The identity \eqref{eq:change_of_variables} immediately implies that the integrals \eqref{eq:jth_summands} are compositions of $V_{\bj, 0}(y, t)$, $\V_\bj(y, t)$, and smooth functions of $\binom{d + \lfloor\beta\rfloor}{d} + 1$ variables. We would like to note that $\binom{d + \lfloor\beta\rfloor}{d} + 1$ may be much smaller than the ambient dimension $D$. This fact plays a crucial role in the proof of Theorem \ref{th:score_approximation}.

\medskip

\noindent\textbf{Step 4: approximation of $V_{\bj, 0}$ and $\V_\bj$.}
\quad
We proceed with approximation of the functions $V_{\bj, 0}(y, t)$, $V_{\bj, \bk}(y, t)$, $1 \leq |\bk| \leq \lfloor\beta\rfloor$, and $\cV(t)$ defined in \eqref{eq:v_j0} and \eqref{eq:v_jk}. Let us restrict our attention on the compact set $\cC^*_{[t_0, T]}$ and consider $V_{\bj, 0}(y, t)$ first. We represent $V_{\bj, 0}(y, t)$ in the following form:
\begin{align}\label{eq:v_j_0_decomp}
    V_{\bj, 0}(y, t) = \frac{\|y - m_t g^\circ_\bj(u_\bj)\|^2}{2\tilde{\sigma}_t^2} = \frac{\|y\|^2}{2\tilde{\sigma}_t^2} - \frac{m_t y^\top g^\circ_\bj(u_\bj)}{\tilde{\sigma}_t^2} + \frac{m_t^2 \|g^\circ_\bj(u_\bj)\|^2}{2\tilde{\sigma}_t^2}.
\end{align}
The terms in the right-hand side can be approximated by small neural networks. We provide the corresponding results in Appendix \ref{sec:approximation_auxiliary_results} (see Lemmata \ref{lem:m_t_over_sigma_t_approx}--\ref{lem:mt_y_dot_a_sigt}). Before we proceed, let us note that
\begin{align*}
    \sup_{(t, y) \in \cC^*_{[t_0, T]}}\|y\|_\infty
    &
    \leq \sup_{(t, y) \in \cC^*_{[t_0, T]}}\inf_{u \in [0, 1]^d} \big\{ \|y - m_t g^*(u)\| + m_t\|g^*(u)\| \big\}
    \\&
    \leq \sup_{t \in [t_0, T]} \big\{ R_t + m_t \big\}
    \\&
    \leq \tilde\sigma_t \sqrt{D} + 16 \tilde\sigma_t \left(\sqrt{D\log\left(\frac{\eps^{-2\beta}}{D}\right)} \vee \log\left(\frac{\eps^{-2\beta}}{D}\right) \right) + 1.
\end{align*}
Therefore, setting $\gamma = 2$ in Lemma \ref{lem:m_t_over_sigma_t_approx}, $M = \sup_{(t, y) \in \cC^*_{[t_0, T]}}\|y\|_\infty$ in Lemmata \ref{lem:y_sq_over_sigma_approx}, \ref{lem:mt_y_dot_a_sigt}, and
$\|a\|_\infty = \|g^\circ_\bj(u_\bj)\|_\infty \lesssim 1$ in Lemma \ref{lem:mt_y_dot_a_sigt},
we obtain that there exists a ReLU-network
\[
    \tilde{V}_{\bj, 0}(y, t) = \frac{\|g^\circ_\bj(u_\bj)\|^2}{2} + \rho_{\eps' / 3}(y, t) + \omega_{\eps' / 3}(y, t)
\]
such that
\begin{equation}
    \label{eq:v_j_0_approximation}
    \left\|\tilde{V}_{\bj, 0} - V_{\bj, 0}\right\|_{L^\infty(\cC^*_{[t_0, T]})} \leq \eps'.
\end{equation}
The functions $\chi_{2, \eps'/3}(t)$, $\rho_{\eps' / 3}(y, t)$, and $\omega_{\eps' / 3}(y, t)$ are defined in Lemmata  \ref{lem:m_t_over_sigma_t_approx}, \ref{lem:y_sq_over_sigma_approx}, and \ref{lem:mt_y_dot_a_sigt}, respectively.
Furthermore, $\tilde{V}_{\bj, 0}(y, t)$ belongs to the class $\NN(\tilde L, \tilde W, \tilde S, \tilde B)$ with
\begin{align}
    \label{eq:v_j_0_cfg}
    \tilde L \vee \log \tilde B
    &\notag
    \lesssim \log^2(1 / \eps') + \log^2(\tilde\sigma_{t_0}^{-2}) + \log^2 D,
    \\
    \|\tilde W\|_\infty
    &
    \lesssim D\left(\frac{1}{t_0 + \sdata^2}\vee 1\right)\left(\log^2(1/\eps') + \log^2(\tilde{\sigma}_{t_0}^{-2}) + \log^2 D \right),
    \\
    \tilde S
    &\notag
    \lesssim D\left(\frac{1}{t_0 + \sdata^2}\vee 1\right)\left(\log^3(1/\eps') + \log^3(\tilde{\sigma}_{t_0}^{-2}) + \log^3 D \right).
\end{align}

The functions $V_{\bj, \bk}(y, t)$, $\bk \in \Z^d_+, 1 \leq |\bk| \leq \floor{\beta}$, are approximated in a similar fashion. Let us recall that 
\begin{align*}
    V_{\bj, \bk}(y, t) = -\frac{m_t}{\tilde\sigma_t^2}(y - m_t g^*(u_\bj))^\top \partial^\bk g^*(u_\bj)
    =-\frac{m_t y^\top \partial^\bk g^*(u_\bj)}{\tilde\sigma_t^2} + \frac{m_t^2 g^*(u_\bj)^\top \partial^\bk g^*(u_\bj)}{\tilde\sigma_t^2}.
\end{align*}
Hence, setting $M = \sup_{(t, y) \in \cC^*_{[t_0, T]}}\|y\|_\infty$ in Lemma \ref{lem:mt_y_dot_a_sigt}
and the approximation accuracy of $\eps' / 2$ in Lemmata \ref{lem:m_t_over_sigma_t_approx} and \ref{lem:mt_y_dot_a_sigt} ,
we obtain that there exists a ReLU-network $\tilde{V}_{\bj, \bk}$ such that 
\begin{align*}
    \left\|\tilde{V}_{\bj, \bk} - V_{\bj, \bk}\right\|_{L^\infty(\cC^*_{[t_0, T]})} \leq \eps'.
\end{align*}
In addition, the configuration of $\tilde{V}_{\bj, \bk}$ is identical to \eqref{eq:v_j_0_cfg}.
Finally, we approximate $\cV(t)$ with the accuracy $\eps'$ on $[t_0, T]$ using Lemma \ref{lem:m_t_over_sigma_t_approx} directly.
Formally, there exists a ReLU neural network $\tilde\cV(t)$ with the configuration as specified in \eqref{eq:v_j_0_cfg} such that
\begin{align*}
    \left\| \tilde\cV - \cV \right\|_{L^\infty([t_0, T])} \leq \eps'.
\end{align*}

\medskip

\noindent\textbf{Step 5: approximation of the integrals \eqref{eq:jth_summands}.}
\quad
Before we move to approximation of the integrals \eqref{eq:jth_summands}, let us make a couple of preparatory steps. First, we fix an arbitrary $\bj \in \{1, \dots, N\}^d$ and substitute $u$ with $u_\bj - \eps w$, $w \in [0, 1]^d$. Then it is straightforward to observe that for any $u \in \cU_\bj$ we have
\begin{align*}
    \frac{\|y - m_t g^\circ_\bj(u)\|^2}{2\tilde{\sigma}_t^2}
    &
    = \frac{\|y - m_t g^\circ_\bj(u_\bj - \eps w)\|^2}{2\tilde{\sigma}_t^2}
    \\&
    = V_{\bj, 0}(y, t) + \cV(t) \left\|g^\circ_\bj(u_\bj - \eps w) - g^*(u_\bj)\right\|^2 + \sum\limits_{\substack{\bk \in \Z^d_+ \\ 1 \leq |\bk| \leq \lfloor\beta\rfloor}} V_{\bj, \bk}(y, t) \, \frac{(-\eps w)^{\bk}}{\bk!}.
\end{align*}
Second, we introduce a function $\sR_\bj : \R^D \times [t_0, T] \rightarrow \R^{\binom{d + \lfloor\beta\rfloor}{d}}$ with normalized components given by
\begin{equation}
    \label{eq:rj}
    \sR_\bj(y, t) = \left( \left( \frac{V_{\bj, \bk}(y, t)}{2 \|V_{\bj, \bk}\|_{L^\infty(\cC^*_{[t_0, T]})}} + \frac12 : \bk \in \Z^d_+, 1 \leq |\bk| \leq \lfloor\beta\rfloor \right), \frac{\cV(t)}{\|\cV\|_{L^\infty([t_0, T])}} \right)^\top
\end{equation}
and define auxiliary maps $a_\bj : \R^d \rightarrow \R^{\binom{d + \lfloor\beta\rfloor}{d}}$ and $b_\bj : \R^d \rightarrow \R$:
\begin{equation}
    \label{eq:a_j}
    \hspace{-0.4cm}
    a_\bj(w) =
    \Bigg( \left(a_{\bj, \bk}(w) : \bk \in \Z^d_+, 1 \leq |\bk| \leq \lfloor\beta\rfloor \right),
    \|\cV\|_{L^\infty([t_0, T])} \left\|g^\circ_\bj(u_\bj - \eps w) - g^*(u_\bj)\right\|^2 \Bigg),
\end{equation}
where
\begin{equation}
    \label{eq:a_j_k}
    a_{\bj, \bk}(w) = \frac{2 (-\eps w)^{\bk} \|V_{\bj, \bk}\|_{L^\infty(\cC_{[t_0, T]}^*)}}{\bk!}
    \quad \text{for all $\bk \in \Z^d_+, 1 \leq |\bk| \leq \lfloor\beta\rfloor$,}
\end{equation}
and
\begin{equation}
    \label{eq:b_j}
    b_\bj(w) = \sum\limits_{\substack{\bk \in \Z^d_+ \\ 1 \leq |\bk| \leq \lfloor\beta\rfloor}} \|V_{\bj, \bk}\|_{L^\infty(\cC_{[t_0, T]}^*)} \, \frac{(-\eps w)^{\bk}}{\bk!}.
\end{equation}
The functions $\sR_\bj$, $a_\bj(w)$, and $b_\bj(w)$ were chosen in such a way that $\sR_\bj$ takes its values in the unit cube $[0, 1]^{\binom{d + \lfloor\beta\rfloor}{d}}$ and 
\[
    \frac{\|y - m_t g^\circ_\bj(u_\bj - \eps w)\|^2}{2\tilde{\sigma}_t^2} = V_{\bj, 0}(y, t) + \sR_\bj(y, t)^\top a_\bj(w) + b_\bj(w).
\]
Hence, the integrals \eqref{eq:jth_summands} admit simple representations 
\[
    \eps^{-d} \int\limits_{\cU_\bj} \exp\left\{ -\frac{\|y - m_t g^\circ_\bj(u)\|^2}{2\tilde{\sigma}_t^2} \right\} \dd u
    = e^{-V_{\bj, 0}(y, t)} \int\limits_{[0, 1]^d} \exp\left\{ -\sR_\bj(y, t)^\top a(w) - b(w) \right\} \dd w
\]
and
\begin{align*}
    &
    \eps^{-d} \int\limits_{\cU_\bj} g_\bj^\circ(u) \exp\left\{ -\frac{\|y - m_t g^\circ_\bj(u)\|^2}{2\tilde{\sigma}_t^2} \right\} \dd u
    \\&
    = e^{-V_{\bj, 0}(y, t)} \int\limits_{[0, 1]^d} g_\bj^\circ(u_\bj - \eps w) \exp\left\{ -\sR_\bj(y, t)^\top a_\bj(w) - b_\bj(w) \right\} \dd w.
\end{align*}
From the previous step, we know that $V_{\bj, 0}(y, t)$, $V_{\bj, \bk}(y, t)$, and $\cV(t)$ can be approximated with neural networks. The definition \eqref{eq:rj} of $\sR_\bj(y, t)$ yields that it admits a ReLU neural network approximation as well. Hence, to approximate \eqref{eq:jth_summands}, we have to study expressions of the form 
\[
    e^{-V_{\bj, 0}(y, t)} \int\limits_{[0, 1]^d} \psi_\bj(w) \exp\left\{ -\sR_\bj(y, t)^\top a_{\bj}(w) - b_{\bj}(w) \right\} \dd w,
\]
where $\psi_\bj : [0, 1]^d \rightarrow \R$ is an arbitrary function with a bounded $L^\infty$-norm. For this purpose, we prove the following technical result in Appendix \ref{sec:lem_exponent_integral_product_approximation_proof}.

\begin{Lem}
    \label{lem:exponent_integral_product_approximation}
    Let $\eps, \eps' \in (0, 1)$ be as defined above and assume that
    \[
        \frac{D \eps \sqrt{\log(1 / \eps)}}{\tilde{\sigma}_{t_0}^2} \leq 1.
    \]
    Let us fix an arbitrary $\bj \in \{1, \dots, N\}^d$ and a function $\psi_\bj : [0, 1]^d \rightarrow \R$ such that $\|\psi_\bj\|_{L^\infty([0, 1]^d)} \leq 2$. Let $\sR_\bj$, $a_\bj$, and $b_\bj$ be as given by \eqref{eq:rj}--\eqref{eq:b_j} and consider the integral
    \[
        \Upsilon_\bj(y, t) = e^{-V_{\bj, 0}(y, t)} \int\limits_{[0, 1]^d} \psi_\bj(w) \exp\left\{ -\sR_\bj(y, t)^\top a(w) - b(w) \right\} \dd w.
    \]
    Then there exists a neural network $\tilde \Upsilon_\bj(y, t)$
    $\NN(L_\Upsilon, W_\Upsilon, S_\Upsilon, B_\Upsilon)$, which approximates $\Upsilon_\bj(y, t)$ within the accuracy $\cO(\eps' \eps)$ with respect to the $L^\infty$-norm on $\cC^*_{[t_0, T]}$:
    \[
        \left\|\Upsilon_\bj - \tilde\Upsilon_\bj\right\|_{L^\infty(\cC^*_{[t_0, T]})} \lesssim \eps' \eps.
    \]
    The network $\tilde \Upsilon(y, t)$ has a depth
    \[
        L_\Upsilon \lesssim \log^2(1 / \eps') + \log^2(\tilde\sigma_{t_0}^{-2}) + \log^2 D + \left(\log\frac1{\eps} + \log\frac1{\eps'} \right) \log^2\left(\log\frac1{\eps} + \log\frac1{\eps'} \right) 
    \]
    and a width
    \begin{align*}
        \|W_\Upsilon\|_\infty
        &
        \lesssim D\left(\frac{1}{t_0 + \sdata^2}\vee 1\right)\left(\log^3(1/\eps') + \log^3(\tilde{\sigma}_{t_0}^{-2}) + \log^3 D \right)
        \\&\quad
        \vee \left(\log\frac1{\eps} + \log\frac1{\eps'} \right)^{\binom{d + \lfloor\beta\rfloor}{d} + 1}.
    \end{align*}
    Furthermore, it has at most
    \[
        S_\Upsilon \lesssim D \left(\frac{1}{t_0 + \sdata^2}\vee 1\right)\left(\log^3(1/\eps') + \log^3(\tilde{\sigma}_{t_0}^{-2}) + \log^3 D \right) + \left(\log\frac1{\eps} + \log\frac1{\eps'} \right)^{2\binom{d + \lfloor\beta\rfloor}{d} + 5}
    \]
    non-zero weights of magnitude $B_\Upsilon$, where
    \[
        \log B_\Upsilon \lesssim \log^2(1 / \eps') + \log^2(\tilde\sigma_{t_0}^{-2}) + \log^2 D.
    \]
    In all the bounds, the hidden constants behind $\lesssim$ depend on $d$ and $\beta$ but not on $D$, $t_0$, and $\sdata$.
\end{Lem}

Since $\|g_\bj^\circ(u)\| \leq 2$ for all $u \in [0, 1]^d$ and $\bj \in \{1, \dots, N\}^d$, we can apply Lemma \ref{lem:exponent_integral_product_approximation} to the integrals \eqref{eq:jth_summands}. Let $g_{\bj, 1}^\circ, \dots, g_{\bj, D}^\circ$ be the components of the vector-valued function $g_{\bj}^\circ$. Then there exist $P_{\bj, 1}(y, t), \dots, P_{\bj, D}(y, t), Q_\bj(y, t) \in \NN(L_\Upsilon, W_\Upsilon, S_\Upsilon, B_\Upsilon)$ such that
\[
    \left|\eps^{-d} \int\limits_{\cU_\bj} \exp\left\{ -\frac{\|y - m_t g^\circ_\bj(u)\|^2}{2\tilde{\sigma}_t^2} \right\} \dd u - Q_\bj(y, t) \right|
    \lesssim \eps \eps'
\]
and
\[
    \max\limits_{1 \leq l \leq D} \left|\eps^{-d} \int\limits_{\cU_\bj} g_{\bj, l}^\circ(u) \exp\left\{ -\frac{\|y - m_t g^\circ_\bj(u)\|^2}{2\tilde{\sigma}_t^2} \right\} \dd u - P_{\bj, l}(y, t) \right|
    \lesssim \eps \eps'.
\]
The configuration parameters $L_\Upsilon$, $W_\Upsilon$, $S_\Upsilon$, and $B_\Upsilon$ are defined in Lemma \ref{lem:exponent_integral_product_approximation}.
Consider the neural networks
\[
    \cQ(y, t) = \sum\limits_{\bj \in \{1, \dots, N\}^d} \eps^d \, Q_\bj (y, t)
\]
and
\[
    \cP_l(y, t) = \sum\limits_{\bj \in \{1, \dots, N\}^d} \eps^d \, P_{\bj, l} (y, t),
    \quad 1 \leq l \leq D.
\]
Obviously, $\cQ$ and $\cP_l$, $1 \leq l \leq D$, have a depth
$\breve L = L_\Upsilon$, a width $\|\breve W\|_\infty = N^d \|W_\Upsilon\|_\infty$, at most $\breve S = N^d S_\Upsilon$ non-zero weights, and the weight magnitude $\breve B = B_\Upsilon$. Moreover, there are constants $C_\cP$ and $C_\cQ$ such that for all $(y, t) \in \cC^*_{[t_0, T]}$
\begin{equation}
    \label{eq:denom_approx}
    \left|\cQ(y, t) - \sum\limits_{\bj \in \{1, \dots, N\}^d} \; \int\limits_{\cU_\bj} \exp\left\{ -\frac{\|y - m_t g^\circ_\bj(u)\|^2}{2\tilde{\sigma}_t^2} \right\} \dd u \right|
    \lesssim N^d \eps^d \eps \eps' \leq C_\cQ \eps \eps'
\end{equation}
and, for any $1 \leq l \leq D$, $(y, t) \in \cC^*_{[t_0, T]}$
\begin{align}
    \label{eq:numerator_approx}
    &\notag
    \max\limits_{1 \leq l \leq D} \left|\cP_l(y, t) - \sum\limits_{\bj \in \{1, \dots, N\}^d} \; \int\limits_{\cU_\bj} g_{\bj, l}^\circ(u) \exp\left\{ -\frac{\|y - m_t g^\circ_\bj(u)\|^2}{2\tilde{\sigma}_t^2} \right\} \dd u \right|
    \\&
    \lesssim N^d \eps^d \eps \eps' \leq C_\cP \eps \eps'.
\end{align}

\medskip

\noindent
\textbf{Step 6: division approximation.}
\quad
It remains to approximate the ratios
\[
    \scriptstyle
    \left(\sum\limits_{\bj \in \{1, \dots, N\}^d} \; \int\limits_{\cU_\bj} g^\circ_{\bj, l}(u) \exp\left\{ -\frac{\|y - m_t g^\circ_\bj(u)\|^2}{2\tilde{\sigma}_t^2} \right\} \dd u \right) \Big/ \left( \sum\limits_{\bj \in \{1, \dots, N\}^d} \; \int\limits_{\cU_\bj} \exp\left\{ -\frac{\|y - m_t g^\circ_\bj(u)\|^2}{2\tilde{\sigma}_t^2} \right\} \dd u \right),
\]
where $l \in \{1, \dots, D\}$,
with the accuracy $\cO(\eps)$ to finish the proof. For this purpose, we show that $\cQ(y, t)$ approximates the denominator
\[
    \sum\limits_{\bj \in \{1, \dots, N\}^d} \int\limits_{\cU_\bj} \exp\left\{ -\frac{\|y - m_t g^\circ_\bj(u)\|^2}{2\tilde{\sigma}_t^2} \right\}
\]
with \emph{relative} accuracy $\cO(\eps)$ on $\cC^*_{[t_0, T]}$. Indeed, let us fix an arbitrary $(y, t) \in \cC^*_{[t_0, T]}$. According to the definition of $\cC^*_{[t_0, T]}$, for any $t \in [t_0, T]$ and any $y \in \cK_t$, there exist $\bj^* \in \{1, \dots, N\}^d$ and $u_\bj^* \in \cU_\bj^*$ such that
\[
    \frac{\|y - m_t g^*(u_\bj^*)\|^2}{\tilde{\sigma}_t^2}
    \leq R_t.
\]
This and Lemma \ref{lem:g_approximation} yield that, for any $u \in \cU_{\bj^*}$, we have
\begin{align*}
    \frac{\|y - m_t g_{\bj^*}^\circ(u)\|^2}{2 \tilde{\sigma}_t^2}
    &
    \leq \frac{\|y - m_t g^*(u_\bj^*)\|^2}{\tilde{\sigma}_t^2} + \frac{2 m_t^2 \|g^*(u) - g^*(u_\bj^*)\|^2}{\tilde{\sigma}_t^2} + \frac{2 m_t^2}{\tilde{\sigma}_t^2} \|g_{\bj^*}^\circ - g^*\|_{L^\infty(\cU_{\bj^*})}^2
    \\&
    \leq \frac{R_t^2}{\tilde\sigma_t^2} + \frac{2 H^2 D \eps^2}{\tilde\sigma_t^2} + \frac{2}{\tilde\sigma_t^2} \left( \frac{H d^{\lfloor\beta\rfloor} \sqrt{D}}{\lfloor\beta\rfloor!} \right)^2 \eps^{2\beta}
    \leq \frac{R_t^2}{\tilde\sigma_t^2} + 4.
\end{align*}
The last inequality follows from the conditions of the theorem.
Hence, we obtain that
\begin{equation}\label{eq:denum_low_bound}
    \begin{split}
    &
    \sum_{\bj \in \{1, \dots, N\}^d } \int\limits_{\cU_\bj} \exp\left\{-\frac{\|y - m_t g^\circ_\bj(u)\|^2}{2\tilde{\sigma}_t^2}\right\} \dd u
    \\&
    \geq \int\limits_{\cU_{\bj^*}} \exp\left\{-\frac{\|y - m_t g^\circ_{\bj^*}(u)\|^2}{2\tilde{\sigma}_t^2}\right\} \dd u
    \\&
    \geq \int\limits_{\cU_{\bj^*}} \exp\left\{-\frac{\|y - m_t g^\circ_{\bj^*}(u_\bj^*)\|^2}{\tilde{\sigma}_t^2} - \frac{m_t^2 \|g^\circ_{\bj^*}(u_\bj^*) - g^\circ_{\bj^*}(u)\|^2}{\tilde{\sigma}_t^2} \right\} \dd u
    \\&
    \geq \eps^d e^{-4 - R_t^2 / \tilde\sigma_t^2}.
    \end{split}
\end{equation}
This allows us to leverage the result on division operation approximation formulated below.
\begin{Lem}\label{lem:div_enhanced}
    Given a positive integer $K \geq 4$.
    Then, for any $\eps \in (0, 1]$, there exists a ReLU-network $\cR \in \NN(L, W, S, B)$ such that
    \begin{align}\label{eq:approx_enh_div}
        \left|\cR(x', y') - \frac{x}{y}\right| \leq 2049(4K^2\log^22 + \log^2(1/\eps))\eps,
    \end{align}
    for all $y \in [2^{-K}, 1]$, $|x| \leq y$ and $x', y' \in \R$ satisfying $|x - x'| \vee |y - y'| \leq 2^{-2K}\eps$.
    The network has $L \lesssim K^2 + \log^2(1/\eps)$ layers, a width $\|W\|_\infty \lesssim K^3 + K\log^2(1/\eps)$,
    $S \lesssim K^4 + K\log^3(1/\eps)$ non-zero weights, and the weight magnitude $B \lesssim 2^{4K}\log^2(1/\eps)$.
\end{Lem}
The formal proof is deferred to Appendix \ref{sec:lem_div_enhanced_proof}. 
Lemma \ref{lem:div_enhanced} also yields that there exists a network $\tilde\cR(y, t) \in \NN(L_\cR, W_\cR, S_\cR, B_\cR)$ such that
\[
    \left|\tilde\cR(x', y') - \frac{x}{y}\right| \leq \eps^{\beta}
\]
for all $y \in [2^{-K}, 1]$, $|x| \leq y$ and $x', y' \in \R$ satisfying $|x - x'| \vee |y - y'| \leq 4^{-K}\eps$. The configuration parameters $L_\cR$, $W_\cR$, $S_\cR$, $B_\cR$ fulfil the inequalities
\begin{align}
    \label{eq:r_cfg}
    &\notag
    L_\cR
    \lesssim K^2 + \log^2\left(\frac{K^2 + \log^2(1/\eps^\beta)}{\eps^\beta} \right)
    \lesssim K^2 + \log^2(1/\eps),
    \\&\notag
    \|W_\cR\|_{\infty}
    \lesssim K^3 + K\log^2\left(\frac{K^2 + \log^2(1/\eps^\beta)}{\eps^\beta} \right) \lesssim K^3 + K \log^2(1/\eps),
    \\&
    S_\cR \lesssim K^4 + K\log^3\left(\frac{K^2 + \log^2(1/\eps^\beta)}{\eps^\beta} \right)
    \lesssim K^4 + K\log^3(1/\eps),
    \\&\notag
    B_\cR
    \lesssim 16^K \log^2\left(\frac{K^2 + \log^2(1/\eps^\beta)}{\eps^\beta} \right)
    \lesssim 16^K \left( \log^2 K + \log^2(1/\eps) \right).
\end{align}
Based on \eqref{eq:denom_approx}--\eqref{eq:denum_low_bound}, we take $\eps' = 4^{-K} \eps / (C_\cP \vee C_\cQ)$ and 
\begin{align*}
    K
    &
    = \frac1{\log 2} \left( 4 + d \log(1/\eps) + \frac{R_t^2}{\tilde \sigma_t^2} \right)
    \\&
    = \frac1{\log 2} \left( 4 + d \log(1/\eps) + \left[ \sqrt{D} + 16 \left(\sqrt{D\log\left(\frac{\eps^{-2\beta}}{D}\right)} \vee \log\left(\frac{\eps^{-2\beta}}{D}\right) \right) \right]^2 \right)
    \\&
    \lesssim D + \log^2(1/\eps).
\end{align*}
Thus, we obtain that the neural network $\tilde\cR(\cP_l(y, t), \cQ(y, t))$ approximates the $l$-th component of
\[
    \scriptstyle
    \left(\sum\limits_{\bj \in \{1, \dots, N\}^d} \int\limits_{\cU_\bj} g^\circ_{\bj}(u) \exp\left\{ -\frac{\|y - m_t g^\circ_\bj(u)\|^2}{2\tilde{\sigma}_t^2} \right\} \dd u \right) \Big/ \left( \sum\limits_{\bj \in \{1, \dots, N\}^d} \int\limits_{\cU_\bj} \exp\left\{ -\frac{\|y - m_t g^\circ_\bj(u)\|^2}{2\tilde{\sigma}_t^2} \right\} \dd u \right)
\]
with the accuracy $\eps^\beta$ with respect to the $L^\infty$-norm on $\cC^*_{[t_0, T]}$. Hence, the neural network
\begin{equation}
    \label{eq:final_approx}
    \tilde f(y, t) = \left( \tilde\cR\big(\cP_1(y, t), \cQ(y, t) \big), \dots, \tilde\cR\big(\cP_D(y, t), \cQ(y, t) \big) \right)
\end{equation}
approximates the ratio
\[
    f^\circ(y, t)
    =
    {\scriptstyle
    \left(\sum\limits_{\bj \in \{1, \dots, N\}^d} \; \int\limits_{\cU_\bj} g^\circ_{\bj}(u) \exp\left\{ -\frac{\|y - m_t g^\circ_\bj(u)\|^2}{2\tilde{\sigma}_t^2} \right\} \dd u \right) \Big/ \left( \sum\limits_{\bj \in \{1, \dots, N\}^d} \; \int\limits_{\cU_\bj} \exp\left\{ -\frac{\|y - m_t g^\circ_\bj(u)\|^2}{2\tilde{\sigma}_t^2} \right\} \dd u \right)
    }
\]
with the accuracy $\eps^\beta \sqrt{D}$ with respect to the $L^\infty$-norm on $\cC^*_{[t_0, T]}$. For each $l \in \{1, \dots, D\}$, $\tilde\cR(\cP_l(y, t), \cQ(y, t))$ is a concatenation of $\tilde\cR$ with the parallel stack consisting of $\cP_l(y, t)$ and $\cQ(y, t)$. Let us recall that the configuration of $\cP_1(y, t), \dots, \cP_D(y, t), \cQ(y, t)$ satisfies the inequalities
\begin{align*}
    \breve L
    &
    = L_\Upsilon \lesssim \log^2(1 / \eps') + \log^2(\tilde\sigma_{t_0}^{-2}) + \log^2 D + \left(\log\frac1{\eps} + \log\frac1{\eps'} \right) \log^2\left(\log\frac1{\eps} + \log\frac1{\eps'} \right),
    \\
    \|\breve W\|_\infty
    &
    = N^d \|W_\Upsilon\|_\infty
    \lesssim D \eps^{-d} \left(\frac{1}{t_0 + \sdata^2}\vee 1\right)\left(\log^3(1/\eps') + \log^3(\tilde{\sigma}_{t_0}^{-2}) + \log^3 D \right)
    \\&\quad
    \vee \eps^{-d}\left(\log\frac1{\eps} + \log\frac1{\eps'} \right)^{\binom{d + \lfloor\beta\rfloor}{d} + 1}
    \\
    \breve S
    &
    = N^d S_\Upsilon
    \lesssim D \eps^{-d} \left(\frac{1}{t_0 + \sdata^2}\vee 1\right)\left(\log^3(1/\eps') + \log^3(\tilde{\sigma}_{t_0}^{-2}) + \log^3 D \right) \\&\quad
    + \eps^{-d} \left(\log\frac1{\eps} + \log\frac1{\eps'} \right)^{2\binom{d + \lfloor\beta\rfloor}{d} + 5}
    \\
    \log \breve B
    &
    = \log B_\Upsilon
    \lesssim \log^2(1 / \eps') + \log^2(\tilde\sigma_{t_0}^{-2}) + \log^2 D.
\end{align*}
Taking into account the configuration of $\tilde\cR$ given by \eqref{eq:r_cfg} and recalling that $\log(1/\eps') = K \log 4 + \log(1/\eps) \lesssim D + \log^2(1/\eps)$, we conclude that the neural network \eqref{eq:final_approx} belongs to the class $\NN(L, W, S, B)$ with
\begin{align*}
    &\notag
    L
    \lesssim D^2 + \log^4(1/\eps),
    \\&\notag
    \|W\|_{\infty}
    \lesssim D^2 \eps^{-d} \left(\frac{1}{t_0 + \sdata^2}\vee 1\right)\left(D + \log^2(1/\eps)\right)^3,
    \\&
    S \lesssim D^2 \eps^{-d} \left(\frac{1}{t_0 + \sdata^2}\vee 1\right)\left(D + \log^2(1/\eps)\right)^3 + D \eps^{-d} \left(D + \log^2\frac1{\eps} \right)^{2\binom{d + \lfloor\beta\rfloor}{d} + 5},
    \\&\notag
    \log B
    \lesssim K + \log \log \left(\frac{K^2 + \log^2(1/\eps^\beta)}{\eps^\beta} \right)
    \lesssim D + \log^2(1/\eps).
\end{align*}
To sum up, the function $\tilde f(y, t)$ defined in \eqref{eq:final_approx}, satisfies the bound
\[
    \sup\limits_{(y, t) \in \cC^*_{[t_0, T]}} \left\| \tilde f(y, t) - f^\circ(y, t) \right\|
    \leq \sqrt{D} \eps^\beta.
\]
This yields that the corresponding score function
\[
    \tilde s(y, t) = -\frac{y}{\tilde\sigma_t^2} + \frac{m_t}{\tilde\sigma_t^2} \clip{\tilde f(y, t)}{2}
\]
fulfils (see \eqref{eq:s_to_f_approx})
\begin{align*}
    &
    \int\limits_{t_0}^T \int\limits_{\R^D} \|s^\circ(y, t) - \tilde s(y, t)\|^2 \, \sfp_t^*(y) \, \dd y \, \dd t
    \\&
    \leq D \eps^{2\beta} \int\limits_{t_0}^T \frac{m_t^2}{\tilde\sigma_t^4} \dd t
    + \int\limits_{t_0}^T \int\limits_{\cK_t} \frac{m_t^2}{\tilde\sigma_t^4} \|f^\circ(y, t) - \tilde f(y, t)\|^2 \sfp_t^*(y) \, \dd y \, \dd t
    \\&
    \leq D \eps^{2\beta} \int\limits_{t_0}^T \frac{m_t^2}{\tilde\sigma_t^4} \dd t
    \lesssim \frac{D \eps^{2\beta}}{\sdata^2 + e^{2 t_0} - 1}
    \lesssim \frac{D \eps^{2\beta}}{\sdata^2 + t_0}.
\end{align*}
Then, due to \eqref{eq:s_circ_s_star_difference}, we finally obtain that
\[
    \int\limits_{t_0}^T \int\limits_{\R^D} \|\tilde s(y, t) - s^*(y, t)\|^2 \, \sfp_t^*(y) \, \dd y \, \dd t
    \lesssim \frac{D \eps^{2\beta}}{\sdata^2 + t_0}.
\]
The proof is complete.

\subsection{Proof of Lemma \ref{lem:g_approximation}}
\label{sec:lem_g_approximation_proof}

Due to the Taylor expansion with an integral remainder term, for any $m \in \{1, \dots, D\}$, $\bj \in \{1, \dots, N\}^d$, and $u \in \cU_\bj$ it holds that
\begin{align*}
    \left| g_m^*(u) - g_m^\circ(u) \right|
    &
    = \left| g_m^*(u) - \sum\limits_{\substack{\bk \in \Z^d_+ \\ |\bk| \leq \lfloor\beta\rfloor}} \frac{\partial^\bk g_m^*(u_\bj)}{\bk!} (u - u_\bj)^{\bk} \right|
    \\&
    = \left| \sum\limits_{\substack{\bk \in \Z^d_+ \\ |\bk| = \lfloor\beta\rfloor}} \int\limits_0^1 \frac{\partial^\bk g_m^*(v u + (1 - v) u_\bj) - \partial^\bk g_m^*(u_\bj)}{\bk!} (u - u_\bj)^{\bk} \dd v\right|.
\end{align*}
Applying the triangle inequality and taking into account that $\partial^\bk g_m^*$, $|\bk| = \lfloor\beta\rfloor$, is a $(\beta - \lfloor\beta\rfloor)$-H\"older function, we obtain that 
\begin{align}
    \label{eq:g_local_poly_approximation_small_cube}
    \left| g_m^*(u) - g_m^\circ(u) \right|
    &\notag
    \leq \sum\limits_{\substack{\bk \in \Z^d_+ \\ |\bk| = \lfloor\beta\rfloor}} \int\limits_0^1 \frac{\left|\partial^\bk  g_m^*(v u + (1 - v) u_\bj) - \partial^\bk g_m^*(u_\bj)\right|}{\bk!} (u - u_\bj)^{\bk} \dd v
    \\&
    \leq \sum\limits_{\substack{\bk \in \Z^D_+ \\ |\bk| = \lfloor\beta\rfloor}} \frac{H \|u - u_\bj\|_\infty^\beta}{\bk!}
    \leq \sum\limits_{\substack{\bk \in \Z^D_+ \\ |\bk| = \lfloor\beta\rfloor}} \frac{H \eps^\beta}{\bk!}
    = \frac{H d^{\lfloor\beta\rfloor} \eps^\beta}{\lfloor\beta\rfloor!}.
\end{align}
In the last line, we used the multinomial theorem which yields that
\[
    \sum\limits_{\substack{\bk \in \Z^D_+ \\ |\bk| = \lfloor\beta\rfloor}} \frac{\lfloor\beta\rfloor!}{\bk!}
    = \big( \underbrace{1 + 1 + \ldots + 1}_{\text{$d$ times}} \big)^{\lfloor\beta\rfloor}
    = d^{\lfloor\beta\rfloor}.
\]
Since the inequality \eqref{eq:g_local_poly_approximation_small_cube} holds for arbitrary $m \in \{1, \dots, D\}$, $\bj \in \{1, \dots, N\}^d$, and $u \in \cU_\bj$, we conclude that
\[
    \left\| g^* - g^\circ \right\|_{L^\infty([0, 1]^d)}
    = \max\limits_{u \in [0, 1]^d} \left\| g^*(u) - g^\circ(u) \right\|
    \leq \frac{H d^{\lfloor\beta\rfloor} \eps^\beta \sqrt{D}}{\lfloor\beta\rfloor!}.
\]
\endproof

\subsection{Proof of Lemma \ref{lem:pt_tails}}
\label{sec:lem_pt_tails_proof}

Due to the definition of $\sfp_t^*(y)$ (see \eqref{eq:pt_density}), it holds that
\[
    \int\limits_{\R^D \backslash \cK_t} \sfp_t^*(y) \, \dd y
    = \int\limits_{[0, 1]^d} \int\limits_{\R^D \backslash \cK_t} (\sqrt{2 \pi} \tilde\sigma_t)^{-D} \exp\left\{- \frac{\|y - m_t g(u)\|^2}{2 \tilde\sigma_t^2} \right\} \, \dd y \, \dd u.
\]
Let us fix an arbitrary $u \in [0, 1]^d$ and consider the integral
\[
    (\sqrt{2 \pi} \tilde\sigma_t)^{-D} \int\limits_{\R^D \backslash \cK_t} \exp\left\{- \frac{\|y - m_t g(u)\|^2}{2 \tilde\sigma_t^2} \right\} \, \dd y.
\]
Introducing a random vector $Y \sim \cN(m_t g(u), \tilde\sigma_t^2 I_D)$, we note that
\[
    (\sqrt{2 \pi} \tilde\sigma_t)^{-D} \int\limits_{\R^D \backslash \cK_t} \exp\left\{- \frac{\|y - m_t g(u)\|^2}{2 \tilde\sigma_t^2} \right\} \, \dd y
    = \p\left( Y \notin \cK_t \right)
    \leq \p\left( \|Y - m_t g(u)\| \geq R_t \right).
\]
The probability in the right-hand side is equal to
\[
    \p\left( \frac{\|Y - m_t g(u)\|^2}{\tilde\sigma_t^2} \geq \frac{R_t^2}{\tilde\sigma_t^2} \right),
    \quad \text{where} \quad
    \frac{\|Y - m_t g(u)\|^2}{\tilde\sigma_t^2} \sim \chi^2(D).
\]
Applying the standard concentration bounds for chi-squared random variables (see Proposition \ref{prop:sub-exp_concentration}), we obtain that
\begin{align*}
    \p\left( \|Y - m_t g(u)\| \geq R_t \right)
    &
    = \p\left( \frac{\|Y - m_t g(u)\|^2}{\tilde\sigma_t^2} \geq \frac{R_t^2}{\tilde\sigma_t^2} \right)
    \\&
    = \p\left( \frac{\|Y - m_t g(u)\|^2}{\tilde\sigma_t^2} \geq D + \frac{R_t^2 - D \tilde\sigma_t^2}{\tilde\sigma_t^2} \right)
    \\&
    \leq \exp\left\{-\frac1{16} \left( \frac{R_t^2}{D \tilde\sigma_t^2} \wedge \frac{R_t}{\tilde\sigma_t} \right) \right\}.
\end{align*}
Hence, it holds that
\begin{align*}
    \int\limits_{\R^D \backslash \cK_t} \sfp_t^*(y) \, \dd y
    &
    \leq \exp\left\{-\frac1{16} \left( \frac{R_t^2 - D \tilde\sigma_t^2}{D \tilde\sigma_t^2} \wedge \frac{\sqrt{R_t^2 - D \tilde\sigma_t^2}}{\tilde\sigma_t} \right) \right\} \int\limits_{[0, 1]^d} \dd u
    \\&
    = \exp\left\{-\frac1{16} \left( \frac{R_t^2 - D \tilde\sigma_t^2}{D \tilde\sigma_t^2} \wedge \frac{\sqrt{R_t^2 - D \tilde\sigma_t^2}}{\tilde\sigma_t} \right) \right\}.
\end{align*}

\subsection{Proof of Lemma \ref{lem:exponent_integral_product_approximation}}
\label{sec:lem_exponent_integral_product_approximation_proof}

The proof of Lemma \ref{lem:exponent_integral_product_approximation} is quite cumbersome, so we split it into several steps for convenience. Let us recall that, on the fourth step of the proof of Theorem \ref{th:score_approximation}, we showed the existence of ReLU neural networks $\tilde\cV$, $\tilde V_{\bj, 0}$, and $\tilde V_{\bj, \bk}$, $1 \leq |\bk| \leq \lfloor\beta\rfloor$, with configuration \eqref{eq:v_j_0_cfg} such that
\[
    \left\| \tilde\cV(t) - \cV(t) \right\|_{L^\infty([t_0, T])} \leq \eps',
    \quad
    \left\|\tilde{V}_{\bj, 0} - V_{\bj, 0}\right\|_{L^\infty(\cC^*_{[t_0, T]})} \leq \eps'\eps^d,
\]
and
\[
    \left\|\tilde{V}_{\bj, \bk} - V_{\bj, \bk}\right\|_{L^\infty(\cC^*_{[t_0, T]})} \leq \eps',
    \quad \text{for all $\bk \in \Z^D_+$, $1 \leq |\bk| \leq \lfloor\beta\rfloor$.}
\]
On the first step, we use this result to approximate $\sR_{\bj}$. Then we focus our attention on the integral
\begin{equation}
    \label{eq:psi_integral}
    \Psi_\bj(y, t) = \int\limits_{[0, 1]^d} \psi_\bj(w) \exp\left\{ -\sR_\bj(y, t)^\top a(w) - b(w) \right\} \dd w.
\end{equation}
After that, we use the result of \cite{oko2023diffusion} (see Lemma \ref{lem:exp_minus_oko} below) to construct a neural network approximating $e^{-V_{\bj, 0}}$. Finally, the last step is devoted to approximation of the product of $e^{-V_{\bj, 0}} \Psi_\bj$.

\medskip

\noindent
\textbf{Step 1. Approximation of $\sR_\bj$.}
\quad
We start with a simple auxiliary result proved in Appendix \ref{sec:lem_normalized_function_approximation}.

\begin{Lem}
    \label{lem:normalized_function_approximation}
    Let $\varphi : \Omega \rightarrow \R$ and $\tilde \varphi : \Omega \rightarrow \R$ be arbitrary functions defined on a set $\Omega$. Assume that
    \[
        \left\| \varphi - \tilde\varphi \right\|_{L^\infty(\Omega)} \leq \eps_0
        \quad \text{for some} \quad
        \eps_0 > 0.
    \]
    Then it holds that
    \[
        \left\| \frac{\varphi}{\|\varphi\|_{L^\infty(\Omega)}} - \frac{\tilde\varphi}{\|\tilde\varphi\|_{L^\infty(\Omega)}} \right\|_{L^\infty(\Omega)}
        \leq \frac{2 \eps_0}{\|\varphi\|_{L^\infty(\Omega)}}.
    \]
\end{Lem}
With Lemma \ref{lem:normalized_function_approximation} at hand, the approximation of $\sR_\bj$ is straightforward. According to the definition of $\sR_\bj$ (see \eqref{eq:rj}), we have
\[
    \sR_\bj(y, t) = \left( \left( \frac{V_{\bj, \bk}(y, t)}{2 \|V_{\bj, \bk}\|_{L^\infty(\cC^*_{[t_0, T]})}} + \frac12 : \bk \in \Z^d_+, 1 \leq |\bk| \leq \lfloor\beta\rfloor \right), \frac{\cV(t)}{\|\cV\|_{L^\infty([t_0, T])}} \right)^\top.
\]
As we mentioned, there are neural networks $\tilde\cV(t)$ and $\tilde V_{\bj, \bk}(y, t)$, $1 \leq |\bk| \leq \lfloor\beta\rfloor$, with configuration \eqref{eq:v_j_0_cfg} that approximate $\cV(t)$ and $V_{\bj, \bk}(y, t)$, $1 \leq |\bk| \leq \lfloor\beta\rfloor$, respectively, with accuracy $\eps'$. Then Lemma \ref{lem:normalized_function_approximation} yields that the neural network
\[
    \tilde\sR_\bj(y, t) = \left( \left( \frac{\tilde V_{\bj, \bk}(y, t)}{2 \|\tilde V_{\bj, \bk}\|_{L^\infty(\cC^*_{[t_0, T]})}} + \frac12 : \bk \in \Z^d_+, 1 \leq |\bk| \leq \lfloor\beta\rfloor \right), \frac{\tilde\cV(t)}{\|\tilde{\cV}\|_{L^\infty([t_0, T])}} \right)^\top
\]
approximates $\sR_\bj(y, t)$ with accuracy $\cO(\eps')$. To be more precise, it holds that
\begin{equation}
    \label{eq:cv_approximation}
    \left\| \frac{\cV(t)}{\|\cV\|_{L^\infty([t_0, T])}} - \frac{\tilde\cV(t)}{\|\cV\|_{L^\infty([t_0, T])}} \right\|_{L^\infty([t_0, T])} \leq \frac{2 \eps'}{\|\cV\|_{L^\infty([t_0, T])}}
\end{equation}
and, for any $\bk \in \Z^d_+, 1 \leq |\bk| \leq \lfloor\beta\rfloor$,
\begin{equation}
    \label{eq:v_j_k_approximation}
    \left\| \frac{\tilde V_{\bj, \bk}(y, t)}{2 \|V_{\bj, \bk}\|_{L^\infty(\cC^*_{[t_0, T]})}} - \frac{V_{\bj, \bk}(y, t)}{2 \|V_{\bj, \bk}\|_{L^\infty(\cC^*_{[t_0, T]})}} \right\|_{L^\infty(\cC^*_{[t_0, T]})}
    \leq \frac{\eps'}{\|V_{\bj, \bk}\|_{L^\infty(\cC^*_{[t_0, T]})}}
\end{equation}

\medskip

\noindent
\textbf{Step 2. Towards approximation of $\Psi_\bj$.}
\quad 
Let us introduce
\begin{equation}
    \label{eq:phi_j}
    \Phi_\bj(v) = \int\limits_{[0, 1]^d} \psi_\bj(w) \, e^{-v^\top a_\bj(w) - b_\bj(w)} \, \dd w,
    \quad
    v \in [0, 1]^{\binom{d + \lfloor\beta\rfloor}{d}}.
\end{equation}
It is easy to observe that $\Psi_\bj(y, t) = \Phi_\bj(\sR_\bj(y, t))$. We have already approximated $\sR_\bj(y, t)$ with a ReLU neural network. On this step, we construct a neural network $\tilde\Phi_\bj$ with ReLU activations that approximates $\Phi_\bj$ within the accuracy $\eps'$. Then the composition $\tilde\Psi_\bj(y, t) = \tilde\Phi_\bj(\tilde\sR_\bj(y, t))$ will be a natural candidate to approximate $\Psi_\bj(y, t)$. However, we postpone a rigorous proof of the last statement for further steps and focus on approximation of $\Phi_\bj$. Our proof relies on the following argument.

\begin{Lem}
    \label{lem:integral_approximation}
    Given arbitrary $r \in \N$ and some functions $\varphi : [0, 1]^d \rightarrow \R$, $a : [0, 1]^d \rightarrow \R^r$, and $b : [0, 1]^d \rightarrow \R$ defined on the unit cube in $\R^d$, consider
    \[
        \Phi(v) = \int\limits_{[0, 1]^d} \varphi(w) \, e^{-v^\top a(w) - b(w)} \, \dd w,
        \quad
        v \in [0, 1]^r.
    \]
    Let $\varphi_{\max} \geq 1$, $a_{\max} \geq 1$, and $A \geq 0$ be such that the following inequalities hold for all $v \in [0, 1]^r$ and $w \in [0, 1]^d$:
    \[
        \|\varphi(w)\|_\infty \leq \varphi_{\max},
        \quad
        \|a(w)\|_\infty \leq a_{\max},
        \quad \text{and} \quad
        -v^\top a(w) - b(w) \leq A.
    \]
    Then, for any $\eps_0 \in (0, 1)$, there exists a neural network $\tilde\Phi$ belonging to the class $\NN(L, W, S, 1)$ with
    \begin{align*}
        &
        L
        \lesssim \left(r + \log\frac1{\eps_0}\right) \log\left( r + \log\frac1{\eps_0}\right) \log \log\frac1{\eps_0},
        \\&
        \|W\|_\infty
        \lesssim \left( r + \log\frac1{\eps_0} \right)^{r + 1},
        \quad \text{and} \quad
        S
        \lesssim \left(r + \log\frac1{\eps_0}\right)^{2r + 5}
    \end{align*}
    such that $\|\widetilde \Phi - \Phi\|_{L^\infty([0, 1]^r)} \leq \eps_0$. The hidden constants behind $\lesssim$ depend on $\varphi_{\max}$, $a_{\max}$, and $A$ only.
\end{Lem}
We postpone the proof of Lemma \ref{lem:integral_approximation} to Appendix \ref{sec:lem_integral_approximation_proof} and elaborate on how it applies to our setup.
We are going to take $\eps_0 = \eps \eps'$, $\varphi(w) = \psi_\bj(w)$, $a(w) = a_\bj(w)$, and $b(w) = b_\bj(w)$, where the functions $a_\bj : \R^d \rightarrow \R^{\binom{d + \lfloor\beta\rfloor}{d}}$ and $b_\bj : \R^d \rightarrow \R$ are given by \eqref{eq:a_j}--\eqref{eq:b_j}, and $\psi_\bj(w)$ is from the statement of the lemma.
It only remains to specify the constants $a_{\max}$ and $A$ from the statement of Lemma \ref{lem:integral_approximation}. For this purpose, we prove the following result in Appendix \ref{sec:lem_infty_norms_bounds_proof}.

\begin{Lem}
    \label{lem:infty_norms_bounds}
    With the notation introduced above, it holds that
    \[
        \|V_{\bj, 0}\|_{L^\infty(\cC_{[t_0, T]}^*)} \lesssim D \log^2\left(\frac{\eps^{-2\beta}}{D}\right) + \frac{m_{t_0}^2}{\tilde{\sigma}_{t_0}^2},
        \qquad
        \|\cV\|_{L^\infty([t_0, T])} \lesssim \frac{m_{t_0}^2}{2\tilde\sigma_{t_0}^2}
    \]
    and
    \[
        \|V_{\bj, \bk}\|_{L^\infty(\cC_{[t_0, T]}^*)} \lesssim D\left(\frac{m_{t_0}^2}{\tilde\sigma_{t_0}^2} + \frac{m_{t_0}}{\tilde\sigma_{t_0}}\log\left(\frac{\eps^{-2\beta}}{D}\right) \right)
        \quad \text{for any $\bk \in \Z^d_+, 1 \leq |\bk| \leq \lfloor\beta\rfloor$.}
    \]
\end{Lem}
Lemma \ref{lem:infty_norms_bounds} yields that
\begin{align*}
    a_{\max}
    &
    \leq \max\left\{ \|\cV\|_{L^\infty([t_0, T])} \left\|g^\circ_\bj(u_\bj - \eps w) - g^*(u_\bj)\right\|^2, \max\limits_{\substack{\bk \in \Z^d_+ \\ 1 \leq |\bk| \leq \lfloor\beta\rfloor}} \frac{2 \eps^{|\bk|} \|V_{\bj, \bk}\|_{L^\infty(\cC_{[t_0, T]}^*)}}{\bk!} \right\}
    \\&
    \lesssim D \eps \left(\frac{m_{t_0}^2}{\tilde\sigma_{t_0}^2} + \frac{m_{t_0}}{\tilde\sigma_{t_0}}\log\left(\frac{\eps^{-2\beta}}{D}\right) \right)
    \lesssim 1.
\end{align*}
Moreover, for any $w \in [0, 1]^d$ and any $v = \big((v_\bk : k \in \Z^d_+, 1 \leq |k| \leq \lfloor\beta\rfloor), \nu\big) \in [0, 1]^{\binom{d + \lfloor\beta\rfloor}{d}}$, it holds that
\begin{align*}
    -v^\top a_\bj(w) - b_\bj(w)
    &
    = \sum\limits_{\substack{\bk \in \Z^d_+ \\ 1 \leq |\bk| \leq \lfloor\beta\rfloor}} v_\bk a_\bk(w)
    - \nu \|\cV\|_{L^\infty([t_0, T])} \left\|g^\circ_\bj(u_\bj - \eps w) - g^*(u_\bj)\right\|^2 - b(w)
    \\&
    \leq \sum\limits_{\substack{\bk \in \Z^d_+ \\ 1 \leq |\bk| \leq \lfloor\beta\rfloor}} v_\bk a_\bk(w) - b(w)
    \\&
    \leq \sum\limits_{\substack{\bk \in \Z^d_+ \\ 1 \leq |\bk| \leq \lfloor\beta\rfloor}} \frac{2 |(-\eps w)^{\bk}| \|V_{\bj, \bk}\|_{L^\infty(\cC_{[t_0, T]}^*)}}{\bk!}
    + \sum\limits_{\substack{\bk \in \Z^d_+ \\ 1 \leq |\bk| \leq \lfloor\beta\rfloor}} \frac{|(-\eps w)^{\bk}| \|V_{\bj, \bk}\|_{L^\infty(\cC_{[t_0, T]}^*)}}{\bk!}
    \\&
    = 3 \max\limits_{\substack{\bk \in \Z^d_+ \\ 1 \leq |\bk| \leq \lfloor\beta\rfloor}} \|V_{\bj, \bk}\|_{L^\infty(\cC_{[t_0, T]}^*)} \sum\limits_{m = 1}^{\lfloor\beta\rfloor} \sum\limits_{\substack{\bk \in \Z^d_+ \\ |\bk| = m}} \frac{|(-\eps w)^{\bk}|}{\bk!}.
\end{align*}
Due to the multinomial theorem, the expression in the right-hand side equals to
\begin{align*}
    &
    3 \max\limits_{\substack{\bk \in \Z^d_+ \\ 1 \leq |\bk| \leq \lfloor\beta\rfloor}} \|V_{\bj, \bk}\|_{L^\infty(\cC_{[t_0, T]}^*)} \sum\limits_{m = 1}^{\lfloor\beta\rfloor} \frac1{m!} \sum\limits_{\substack{\bk \in \Z^d_+ \\ |\bk| = m}} \frac{m! |(-\eps w)^{\bk}|}{\bk!}
    \\&
    = 3 \max\limits_{\substack{\bk \in \Z^d_+ \\ 1 \leq |\bk| \leq \lfloor\beta\rfloor}} \|V_{\bj, \bk}\|_{L^\infty(\cC_{[t_0, T]}^*)} \sum\limits_{m = 1}^{\lfloor\beta\rfloor} \frac{\eps^m \|w\|_1^m}{m!}.
\end{align*}
Taking into account that $w \in [0, 1]^d$ and $\|w\|_1 \leq d$, we obtain that
\begin{align}
    \label{eq:exp_argument_bound}
    -v^\top a_\bj(w) - b_\bj(w)
    &\notag
    \leq 3 \max\limits_{\substack{\bk \in \Z^d_+ \\ 1 \leq |\bk| \leq \lfloor\beta\rfloor}} \|V_{\bj, \bk}\|_{L^\infty(\cC_{[t_0, T]}^*)} \sum\limits_{m = 1}^{\lfloor\beta\rfloor} \frac{\eps^m d^{m}}{m!}
    \\&
    \leq 3 \max\limits_{\substack{\bk \in \Z^d_+ \\ 1 \leq |\bk| \leq \lfloor\beta\rfloor}} \|V_{\bj, \bk}\|_{L^\infty(\cC_{[t_0, T]}^*)} \left(e^{\eps d} - 1 \right)
    \\&\notag
    \lesssim D d \eps \left(\frac{m_{t_0}^2}{\tilde\sigma_{t_0}^2} + \frac{m_{t_0}}{\tilde\sigma_{t_0}}\log\left(\frac{\eps^{-2\beta}}{D}\right) \right)
    \lesssim 1.
\end{align}
Here the last inequality follows from Lemma \ref{lem:infty_norms_bounds}.  Hence, we can apply Lemma \ref{lem:integral_approximation} with $a_{\max} \lesssim 1$ and $A \lesssim 1$. It yields that there exists a neural network $\tilde\Phi_\bj \in \NN(L_\Phi, W_\Phi, S_\Phi, 1)$ with configuration
\begin{align}
    \label{eq:phi_nn_configuration}
    &\notag
    L_\Phi
    \lesssim \left(\log\frac1{\eps} + \log\frac1{\eps'} \right) \log^2\left(\log\frac1{\eps} + \log\frac1{\eps'} \right),
    \\&
    \|W_\Phi\|_\infty
    \lesssim \left(\log\frac1{\eps} + \log\frac1{\eps'} \right)^{\binom{d + \lfloor\beta\rfloor}{d} + 1},
    \quad \text{and} \quad
    S_\Phi
    \lesssim \left(\log\frac1{\eps} + \log\frac1{\eps'} \right)^{2\binom{d + \lfloor\beta\rfloor}{d} + 5}
\end{align}
such that
\[
    \left\|\widetilde \Phi_\bj - \Phi_\bj \right\|_{L^\infty\left([0, 1]^{\binom{d + \lfloor\beta\rfloor}{d}} \right)} \leq \eps \eps'.
\]
We proceed with approximation of $\Psi_\bj$.

\medskip

\noindent
\textbf{Step 3: approximation of $\Psi_\bj$.}
\quad As we announced on the previous step, we are going to show that $\tilde\Psi_\bj(y, t) = \tilde\Phi_\bj(\tilde\sR_\bj(y, t))$ is a reasonable approximation for $\Psi_\bj(y, t)$. Let us fix an arbitrary $(y, t) \in \cC^*_{[t_0, T]}$ and consider the difference $\tilde\Psi_\bj(y, t) - \Psi_\bj(y, t)$. Due to the triangle inequality, it holds that
\begin{align}
    \label{eq:psi_approximation_triangle_inequality}
    \left| \tilde\Psi_\bj(y, t) - \Psi_\bj(y, t) \right|
    &\notag
    \leq \left| \tilde\Phi_\bj \big(\tilde\sR_\bj(y, t) \big) - \Phi_\bj \big(\tilde \sR_\bj(y, t) \big) \right|
    + \left| \Phi_\bj \big(\tilde\sR_\bj(y, t) \big) - \Phi_\bj \big(\sR_\bj(y, t) \big) \right|
    \\&
    \leq \eps \eps' + \left| \Phi_\bj \big(\tilde\sR_\bj(y, t) \big) - \Phi_\bj \big(\sR_\bj(y, t) \big) \right|.
\end{align}
Let us take a closer look at $\Phi_\bj$. Its definition \eqref{eq:phi_j} yields that
\[
    \nabla \Phi_\bj(v) = - \int\limits_{[0, 1]^d} \psi_\bj(w) a_\bj(w) \, e^{-v^\top a_\bj(w) - b_\bj(w)} \, \dd w,
    \quad
    v \in [0, 1]^{\binom{d + \lfloor\beta\rfloor}{d}}.
\]
Using the Newton-Leibniz formula, we obtain that
\begin{align*}
    &
    \left| \Phi_\bj \big(\tilde\sR_\bj(y, t) \big) - \Phi_\bj \big(\sR_\bj(y, t) \big) \right|
    \\&
    = \left| \int\limits_0^1 \nabla \Phi_\bj\big(r \tilde\sR_\bj(y, t) + (1 - r) \sR_\bj(y, t) \big)^\top \big(\tilde\sR_\bj(y, t) - \sR_\bj(y, t) \big) \, \dd r\right|
    \\&
    \leq 2 \max\limits_{v \in [0, 1]^{\binom{d + \lfloor\beta\rfloor}{d}}} e^{-v^\top a_\bj(w) - b_\bj(w)} \max\limits_{w \in [0, 1]^d} \left| a_\bj(w)^\top \big(\tilde\sR_\bj(y, t) - \sR_\bj(y, t) \big) \right|.
\end{align*}
Note that the second factor is of order $\cO(1)$ according to \eqref{eq:exp_argument_bound}. At the same time, the definition of $a_\bj(w)$ (see \eqref{eq:a_j} and \eqref{eq:a_j_k}) and the bounds \eqref{eq:cv_approximation}, \eqref{eq:v_j_k_approximation} derived on the first step imply that
\begin{align*}
    &
    \max\limits_{w \in [0, 1]^d} \left| a_\bj(w)^\top \big(\tilde\sR_\bj(y, t) - \sR_\bj(y, t) \big) \right|
    \\&
    \leq \|\cV\|_{L^\infty([t_0, T])} \max\limits_{w \in [0, 1]^d} \left\|g^\circ_\bj(u_\bj - \eps w) - g^*(u_\bj)\right\|^2 \cdot \frac{2\eps'}{\|\cV\|_{L^\infty([t_0, T])}}
    \\&\quad
    + \sum\limits_{\substack{\bk \in \Z^d_+\\ 1 \leq |\bk| \leq \lfloor\beta\rfloor}} \frac{2 \eps^{|\bk|} \|V_{\bj, \bk}\|_{L^\infty(\cC_{[t_0, T]}^*)}}{\bk!} \cdot \frac{\eps'}{\|V_{\bj, \bk}\|_{L^\infty(\cC^*_{[t_0, T]})}}
    \\&
    = 2 \eps' \max\limits_{w \in [0, 1]^d} \left\|g^\circ_\bj(u_\bj - \eps w) - g^*(u_\bj)\right\|^2
    + 2 \eps' \sum\limits_{m = 1}^{\lfloor\beta\rfloor} \frac1{m!} \sum\limits_{\substack{\bk \in \Z^d_+\\ |\bk| = m}} \frac{m! \, \eps^{|\bk|}}{\bk!}.
\end{align*}
Applying the multinomial theorem, we note that
\[
    \sum\limits_{m = 1}^{\lfloor\beta\rfloor} \frac1{m!} \sum\limits_{\substack{\bk \in \Z^d_+\\ |\bk| = m}} \frac{m! \, \eps^{|\bk|}}{\bk!}
    = \sum\limits_{m = 1}^{\lfloor\beta\rfloor} \frac{(\eps d)^m}{m!}
    \leq e^{\eps d} - 1 \lesssim \eps d.
\]
Taking into account the inequality
\[
    \max\limits_{w \in [0, 1]^d} \left\|g^\circ_\bj(u_\bj - \eps w) - g^*(u_\bj)\right\|^2
    \lesssim D \eps^2
\]
following from the properties of local polynomial approximation, we deduce that
\[
    \left| \Phi_\bj \big(\tilde\sR_\bj(y, t) \big) - \Phi_\bj \big(\sR_\bj(y, t) \big) \right|
    \lesssim \max\limits_{w \in [0, 1]^d} \left| a_\bj(w)^\top \big(\tilde\sR_\bj(y, t) - \sR_\bj(y, t) \big) \right|
    \lesssim D \eps' \eps^2 + d \eps' \eps.
\]
This and \eqref{eq:psi_approximation_triangle_inequality} yield that
\[
    \left| \tilde\Psi_\bj(y, t) - \Psi_\bj(y, t) \right|
    \lesssim \eps' \eps + D \eps' \eps^2 + d \eps' \eps
    \lesssim \eps' \eps.
\]
In conclusion, we would like to note that $\tilde\Psi_j$ was obtained by concatenation of neural networks with configurations \eqref{eq:v_j_0_cfg} and \eqref{eq:phi_nn_configuration}. This means that $\tilde\Psi_j$ belongs to a class $\NN(L_\Psi, W_\Psi, S_\Psi, B_\Psi)$ of neural networks of depth
\[
    L_\Psi \lesssim \log^2(1 / \eps') + \log^2(\tilde\sigma_{t_0}^{-2}) + \log^2 D + \left(\log\frac1{\eps} + \log\frac1{\eps'} \right) \log^2\left(\log\frac1{\eps} + \log\frac1{\eps'} \right) 
\]
and width
\begin{align*}
    \|W_\Psi\|_\infty
    &
    \lesssim D\left(\frac{1}{t_0 + \sdata^2}\vee 1\right)\left(\log^3(1/\eps') + \log^3(\tilde{\sigma}_{t_0}^{-2}) + \log^3 D \right)
    \\&\quad
    \vee \left(\log\frac1{\eps} + \log\frac1{\eps'} \right)^{\binom{d + \lfloor\beta\rfloor}{d} + 1}.
\end{align*}
Furthermore, it has at most
\[
    S_\Psi \lesssim D \left(\frac{1}{t_0 + \sdata^2}\vee 1\right)\left(\log^3(1/\eps') + \log^3(\tilde{\sigma}_{t_0}^{-2}) + \log^3 D \right) + \left(\log\frac1{\eps} + \log\frac1{\eps'} \right)^{2\binom{d + \lfloor\beta\rfloor}{d} + 5}
\]
non-zero weights of magnitude $B_\Psi$, where
\[
    \log B_\Psi \lesssim \log^2(1 / \eps') + \log^2(\tilde\sigma_{t_0}^{-2}) + \log^2 D.
\]

\medskip

\noindent
\textbf{Step 4: final step.}
\quad
Let us recall that our goal is to approximate the product
\[
    \Upsilon_\bj(y, t) = e^{-V_{\bj, 0}(y, t)} \Psi_\bj(y, t).
\]
We already have approximated $V_{\bj, 0}(y, t)$ and $\Psi_\bj(y, t)$ by the neural networks $\tilde V_{\bj, 0}(y, t)$ and $\tilde \Psi_\bj(y, t)$, respectively (see the previous step and \eqref{eq:v_j_0_approximation}). The claim of the lemma easily follows from the results established by \cite{oko2023diffusion}. First, according to Lemma \ref{lem:exp_minus_oko} and Corollary \ref{co:exp}, there exists a neural network $\phi_{\exp}$ of depth $L_{\exp} \lesssim \log^2(1/\eps)$ and width $\|W_{\exp}\|_\infty \lesssim \log(1/\eps)$ with at most $S_{\exp} \lesssim \log^2(1/\eps)$ non-zero weights of magnitude $B_{\exp}$, $\log B_{\exp} \lesssim \log^2(1/\eps)$, such that
\[
    \left| e^{-V_{\bj, 0}(y, t)} - e^{\eps'\eps^d} \phi_{\exp}\left( \tilde V_{\bj, 0}(y, t) + \eps'\eps^d \right) \right|
    \leq e^{\eps'\eps^d} \left( \eps' \eps^d + |\tilde{V}_{\bj, 0}(y, t) - V_{\bj, 0}(y, t)| \right)
    \lesssim \eps'\eps^d
\]
for all $(y, t) \in \cC^*_{[t_0, T]}$. Finally, due to Lemma \ref{lem:multi_oko}, there is a neural network
\[
    \phi_{\mathrm{prod}} \in \NN(L_{\mathrm{prod}}, W_{\mathrm{prod}}, S_{\mathrm{prod}}, B_{\mathrm{prod}})
\]
with configuration
\[
     L_{\mathrm{prod}} \lesssim \log\left( \log\frac1\eps + \log\frac1{\eps'} \right),
     \quad
     \|W_{\mathrm{prod}}\|_\infty = 96,
     \quad S_{\mathrm{prod}} \lesssim \log\left( \log\frac1\eps + \log\frac1{\eps'} \right),
     \quad B_{\mathrm{prod}} \lesssim 1
\]
such that for any $(y, t) \in \cC^*_{[t_0, T]}$
\begin{align*}
    &
    \left| e^{-V_{\bj, 0}(y, t)} \Psi(y, t) - \phi_{\mathrm{prod}}\left(e^{\eps'\eps^d} \phi_{\exp}\left( \tilde V_{\bj, 0}(y, t) + \eps'\eps^d \right), \tilde\Psi_\bj(y, t) \right) \right|
    \\&
    \lesssim \eps \eps' + \left| e^{-V_{\bj, 0}(y, t)} - e^{\eps'\eps^d} \phi_{\exp}\left( \tilde V_{\bj, 0}(y, t) + \eps'\eps^d \right) \right| \vee \left| \Psi_\bj(y, t) - \tilde\Psi_\bj(y, t) \right|
    \lesssim \eps \eps'.
\end{align*}
It only remains to note that, by the construction, the approximating neural network
\[
    \tilde\Upsilon_\bj(y, t) = \phi_{\mathrm{prod}}\left(e^{\eps'\eps^d} \phi_{\exp}\left( \tilde V_{\bj, 0}(y, t) + \eps'\eps^d \right), \tilde\Psi_\bj(y, t) \right)
\]
belongs to a class $\NN(L_\Upsilon, W_\Upsilon, S_\Upsilon, B_\Upsilon)$ of feed-forward ReLU nets of depth
\[
    L_\Upsilon \lesssim \log^2(1 / \eps') + \log^2(\tilde\sigma_{t_0}^{-2}) + \log^2 D + \left(\log\frac1{\eps} + \log\frac1{\eps'} \right) \log^2\left(\log\frac1{\eps} + \log\frac1{\eps'} \right) 
\]
and width
\begin{align*}
    \|W_\Upsilon\|_\infty
    &
    \lesssim D\left(\frac{1}{t_0 + \sdata^2}\vee 1\right)\left(\log^3(1/\eps') + \log^3(\tilde{\sigma}_{t_0}^{-2}) + \log^3 D \right)
    \\&\quad
    \vee \left(\log\frac1{\eps} + \log\frac1{\eps'} \right)^{\binom{d + \lfloor\beta\rfloor}{d} + 1}.
\end{align*}
Furthermore, it has at most
\[
    S_\Upsilon \lesssim D \left(\frac{1}{t_0 + \sdata^2}\vee 1\right)\left(\log^3(1/\eps') + \log^3(\tilde{\sigma}_{t_0}^{-2}) + \log^3 D \right) + \left(\log\frac1{\eps} + \log\frac1{\eps'} \right)^{2\binom{d + \lfloor\beta\rfloor}{d} + 5}
\]
non-zero weights of magnitude $B_\Upsilon$, where
\[
    \log B_\Upsilon \lesssim \log^2(1 / \eps') + \log^2(\tilde\sigma_{t_0}^{-2}) + \log^2 D.
\]
The proof is finished.

\endproof

\subsection{Proof of Lemma \ref{lem:div_enhanced}}
\label{sec:lem_div_enhanced_proof}
We first define a sequence $t_k := 2^{-K + k}$, where $k \in \{0, \dots, K\}$.
Second, we build a sequence of approximators and ensemble them to obtain the enhanced one.
The subsequent lemma presents a basic approximation result within this framework.
Its formal proof is moved to Appendix \ref{sec:lem_div_on_segment_proof}.
\begin{Lem}
\label{lem:div_on_segment}
Let $0 < a \leq b \leq 1$ and let also $\eps' \in (0, b]$.
Then, for any $\eps \in (0, 1]$ there exists
a ReLU-network $q \in \NN(L, W, S, B)$ such that for all $y \in [a, b]$,
 $|x| \leq y$ and $x', y' \in \R$ with $|x - x'| \vee |y - y'| \leq \eps'$ it holds that
\begin{align}\label{eq:div_segment_acc}
    \left|q(x', y') - \frac{x}{y}\right| \leq \frac{32\log^2(1/\eps)}{a^2}(\eps + \eps').
\end{align}
Furthermore, the neural network is implemented with $L \lesssim (\log(b/a) + \log(\log(1/\eps)\vee e) )\log(1/\eps)$, weight magnitude $B \lesssim b^{-2}$, the number of non-zero parameters $S \lesssim (b^2 / a^2) \log^3(1/\eps)$
and the width $\|W\|_\infty \lesssim (b^2 / a^2) \log^2(1/\eps)$. Moreover, the function $q$ satisfies the inequality 
\begin{align*}
    |q(x', y')| \lesssim a^{-1}\log(1/\eps).
\end{align*} 
\end{Lem}

For each $k \in \{1, \dots, K - 1\}$, let $q_k(x, y)$ be a ReLU-network from Lemma \ref{lem:div_on_segment} corresponding to
the parameters $a = t_{(k-2)\vee 0}$, $b = t_{(k+2)\wedge K}$ and the accuracy parameter $2^{-2(K-k)}\eps$.
So, the sensitivity analysis suggests that for any $x', y' \in \R$ such that $|x - x'| \vee |y - y'| \leq 2^{-2K}\eps \leq t_{(k+2)\wedge K}$,
$y \in [t_{(k-2) \vee 0}, t_{(k+2)\wedge K}]$ and $|x| \leq y$
\begin{equation}\label{eq:q_k_sens}
\begin{split}
    \left|q_k(x', y') - \frac{x}{y}\right| &\leq \frac{32\log^2(2^{2(K-k)} / \eps)}{t_{(k-2)\vee 0}^2}(2^{-2(K-k)}\eps + 2^{-2K}\eps) \\
    &\leq \frac{64(4K^2\log^22 + \log^2(1/\eps))}{2^{-2K + 2k - 4}}(2^{-2(K-k)}\eps + 2^{-2K}\eps) \\
    &\leq 2048(4K^2\log^22 + \log^2(1/\eps))\eps.
\end{split}
\end{equation}
In what follows, we construct a partition of unity for $k \in \{1, \dots, K - 1\}$ to switch between the introduced approximations.
To be more precise, let
\begin{align*}
    g_k(y) =
    \begin{cases}
        h(t_1, t_2, y), \quad & k = 1, \\
        h(t_k, t_{k+1}, y) + h(-t_k, -t_{k - 1}, -y) - 1, \quad & k \in \{2, \dots, K - 2\}, \\
        h(-t_{K - 1}, -t_{K - 2}, -y), \quad & k = K - 1.
    \end{cases}
\end{align*}
where $h(a, b, y) = \relu\left( (b - y) / (b - a) \right) - \relu\left( (a - y) / (b - a) \right)$.
Then, it is evident that the collection $\{g_k : 1 \leq k \leq K\}$ forms a partition of unity, that is, for all $y \in \R$, it holds that
\begin{align}\label{eq:pou_division}
    \sum_{k=1}^{K-1}g_k(y) = 1
\end{align}
and, also, for all $k \in \{1, \dots, K - 1\}$ we have
\begin{align}\label{eq:pou_zeros}
    g_k(y) = 0, \quad y \in [2^{-K}, 1] \setminus [t_{k - 1}, t_{k+1}]
\end{align}
Let a constant $C > 0$ to be specified a bit later and let $h_2(x, y)$ stand for a ReLU network from Lemma \ref{lem:multi_oko} approximating the product of two terms with an accuracy $\eps$ over $[-C, C]^2$. Let us show that
\begin{align}\label{eq:q_pou}
    \cR(x, y) = \sum_{k = 1}^{K-1} h_2(g_k(y), q_k(x, y))
\end{align}
is a reasonable approximation of $x/y$.
Indeed, for an arbitrary $k^* \in \{1, \dots, K- 1\}$,
any $y \in [t_{k^*}, t_{k^* + 1}]$, $|x| \leq y$, and
$x', y' \in \R$ such that $|x-x'|\vee |y - y'| \leq 2^{-2K}\eps$,
the bound \eqref{eq:pou_division} along with the triangle inequality yield that
\begin{align*}
    &
    \left|\cR(x', y') - \frac{x}{y}\right|
    \leq \left|\cR(x', y') - \sum_{k=1}^{K-1}g_k(y')q_k(x', y')\right| + \left|\sum_{k=1}^{K-1}g_k(y')q_k(x', y') - \frac{x}{y}\right| \\
    &\leq\sum_{k=1}^{K-1}\left|h_2(g_k(y'), q_k(x', y')) - g_k(y')q_k(x', y')\right| + \left|\sum_{k=1}^{K-1}g_k(y')\left(q_k(x', y') - \frac{x}{y}\right)\right|.
\end{align*}
Note that, according to \eqref{eq:pou_zeros}, for any $y'$ satisfying the inequality $|y - y'| \leq 2^{-2K}\eps$ we have $g_k(y') = 0$ for all $k \in \{1, \dots, K-1\}$ such that $|k - k^*| > 2$.
In addition, for any $k \in \{1, \dots, K - 1\}$ fulfilling $|k - k^*| \leq 2$ it holds that $y \in [t_{(k-2) \vee 0}, t_{(k+2)\wedge K}]$.
Hence, the sensitivity analysis from \eqref{eq:q_k_sens}, the property that $h_2(g_k(y'), q_k(x', y')) = 0$ for $|k - k^*| > 2$, and \eqref{eq:pou_division} suggest that
\begin{align*}
    \left|\cR(x', y') - \frac{x}{y}\right|
    &\leq \sum_{k=1}^{K - 1}\1(|k - k^*| \leq 2)\left(\eps + g_k(y') 2048(4K^2\log^22 + \log^2(1/\eps))\eps\right) \\
    &\leq 5\eps + 2048(4K^2\log^22 + \log^2(1/\eps))\eps \\
    &\leq 2049(4K^2\log^22 + \log^2(1/\eps))\eps.
\end{align*}
Therefore, \eqref{eq:approx_enh_div} holds true. Let us elaborate on the configuration of $q$ from \eqref{eq:q_pou}.
First, note that for the specified $x'$, $y'$ and $|k - k^*| \leq 2$ Lemma \ref{lem:div_on_segment} we deduce that
\begin{align*}
    |q_k(x', y')| \lesssim t_{(k-2)\vee 0}^{-1}\log(2^{2(K-k)}/\eps) \lesssim 2^K K\log(1/\eps).
\end{align*}
Therefore, we can take $C \asymp 2^K K \log(1 / \eps)$.
Then, $h_2$ has the following configuration:
\begin{equation}\label{eq:div_h2_cfg}
\begin{split}
    &
    L(h_2) \lesssim \log(1/\eps) + \log(2^K K\log(1/\eps)) \lesssim \log(1/\eps) + K, \\
    &
    \|W(h_2)\|_\infty \lesssim 1, \\
    &
    S(h_2) \lesssim \log(1/\eps) + \log(2^KK\log(1/\eps)) \lesssim \log(1/\eps) + K, \\
    &
    B(h_2) \lesssim (2^KK\log(1/\eps))^2 \lesssim 2^{4K}\log^2(1/\eps).
\end{split}
\end{equation}
Next, we report the configuration of $q_k$,
$k \in \{1, \dots, K - 1\}$, from Lemma \ref{lem:div_on_segment}:
\begin{equation}\label{eq:div_qk_cfg}
\begin{split}
    L(q_k) &\lesssim \left(\log\frac{t_{(k+2)\wedge K}}{t_{(k-2) \vee 0}} + \log\log(2^{2(K-k)}/\eps)\right)\log(2^{2(K-k)}/\eps) \lesssim \log^2(1/\eps) + K^2, \\
    \|W(q_k)\|_\infty &\lesssim \left(\frac{t_{(k+2)\wedge K}}{t_{(k-2) \vee 0}}\right)^2 \log^2(2^{2(K - k)}/\eps) \lesssim K^2 + \log^2(1/\eps), \\
    S(q_k) &\lesssim \left(\frac{t_{(k+2)\wedge K}}{t_{(k-2) \vee 0}}\right)^2 \log^3(2^{2(K - k)}/\eps) \lesssim K^3 + \log^3(1/\eps), \\
    B(q_k) &\lesssim t_{(k+2)\wedge K}^{-2} \lesssim 2^{2K}.
\end{split}
\end{equation}
Finally, the configuration of $g_k$ for $k \in \{1, \dots, K - 1\}$ is such that
\begin{equation}\label{eq:div_gk_cfg}
\begin{split}
    L(g_k) \vee \|W(g_k)\|_\infty \vee S(g_k) &\lesssim 1, \quad
    B(g_k) \lesssim 2^K.
\end{split}
\end{equation}
Summing up \eqref{eq:div_h2_cfg}, \eqref{eq:div_qk_cfg}, \eqref{eq:div_gk_cfg} and taking into account that the architecture of $\cR$ described in \eqref{eq:q_pou} incorporates $K-1$ occurrences of $h_2$ with the arguments
$g_k(y)$ and $q_k(x, y)$, we conclude that the configuration of $\cR$ satisfies
\begin{align*}
    L(\cR) 
    &\lesssim K^2 + \log^2(1/\eps), \\
    \|W(\cR)\|_\infty
    &\lesssim K^3 + K\log^2(1/\eps), \\
    S(\cR) 
    &\lesssim K^4 + K\log^3(1/\eps), \\
    B(\cR) 
    &\lesssim 2^{4K}\log^2(1/\eps).
\end{align*}
The proof is complete.

\endproof

\subsection{Proof of Lemma \ref{lem:normalized_function_approximation}}
\label{sec:lem_normalized_function_approximation}

Applying the triangle inequality, we obtain that
\begin{align*}
    \left\| \frac{\varphi}{\|\varphi\|_{L^\infty(\Omega)}} - \frac{\tilde\varphi}{\|\tilde\varphi\|_{L^\infty(\Omega)}} \right\|_{L^\infty(\Omega)}
    &
    \leq \left\| \frac{\varphi}{\|\varphi\|_{L^\infty(\Omega)}} - \frac{\tilde\varphi}{\|\varphi\|_{L^\infty(\Omega)}} \right\|_{L^\infty(\Omega)}
    \\&\quad
    + \left\| \frac{\tilde\varphi}{\|\varphi\|_{L^\infty(\Omega)}} - \frac{\tilde\varphi}{\|\tilde\varphi\|_{L^\infty(\Omega)}} \right\|_{L^\infty(\Omega)}.
\end{align*}
The expression in the right-hand side does not exceed
\begin{align*}
    &
    \frac{\|\tilde\varphi - \varphi\|_{L^\infty(\Omega)}}{\|\varphi\|_{L^\infty(\Omega)}} + \|\tilde\varphi\|_{L^\infty(\Omega)} \left| \frac1{\|\varphi\|_{L^\infty(\Omega)}} - \frac1{\|\tilde\varphi\|_{L^\infty(\Omega)}} \right|
    \\&
    = \frac{\|\tilde\varphi - \varphi\|_{L^\infty(\Omega)}}{\|\varphi\|_{L^\infty(\Omega)}} + \frac{1}{\|\varphi\|_{L^\infty(\Omega)}} \left|\|\varphi\|_{L^\infty(\Omega)}  - \|\tilde\varphi\|_{L^\infty(\Omega)}\right|
    \leq \frac{2\eps_0}{\|\varphi\|_{L^\infty(\Omega)}}.
\end{align*}
This completes the proof.

\endproof

\subsection{Proof of Lemma \ref{lem:integral_approximation}}
\label{sec:lem_integral_approximation_proof}

The proof relies on the well-known result of \cite{schmidt-hieber20} on approximation properties of feedforward ReLU neural networks. We just have to bound the norm of $\Phi(v)$ in the space $\cH^\alpha([0, 1]^r)$, $\alpha \in \N$. For this purpose, let us fix any $\bk \in \Z^r_+$ and note that
\[
    \partial^\bk \Phi(v)
    = (-1)^\bk \int\limits_{[0, 1]^d} \varphi(w) a(w)^\bk \, e^{-v^\top a(w) - b(w)} \, \dd w.
\]
Due to the conditions of the lemma, it holds that
\[
    \max\limits_{v \in [0, 1]^r} \left| \partial^\bk \Phi(v) \right|
    \leq \varphi_{\max} a_{\max}^{|\bk|} \int\limits_{[0, 1]^d} e^{-v^\top a(w) - b(w)} \, \dd w
    \leq \varphi_{\max} a_{\max}^{|\bk|} e^A.
\]
Thus, we have
\[
    \|\Phi\|_{\cH^\alpha([0, 1]^r)} \leq \varphi_{\max}  a_{\max}^{\alpha} e^A
    \quad \text{for all $\alpha \in \N$.}
\]
Let us apply Theorem \ref{th:relu_approximation} with the integers $\alpha \geq e (1 + \varphi_{\max}  a_{\max}^{\alpha} e^A)^{1/r} - 1$, $N = (\alpha + 1)^r$, and
\[
    m
    = \left\lceil(\alpha + r) \log_2(1 + \alpha) + r \log_2 6 + \log_2\left(1 + r^2 + \alpha^2 \right) \right\rceil.
\]
According to Remark \ref{rem:relu_approximation_corollary}, we obtain that there exists a neural network $\tilde\Phi \in \NN(L, W, S, 1)$ with depth
\[
        L = 8 + (m + 5) (1 + \lceil\log_2(r \vee \alpha)\rceil,
\]
width
\[
    \|W\|_\infty = 6 (r \vee \lceil\alpha\rceil) N,
\]
and with at most
\[
    S \leq 141 (r + \alpha + 1)^{3 + r} N (m + 6)
\]
non-zero weights such that
\[
    \|\widetilde \Phi - \Phi\|_{L^\infty([0, 1]^r)}
    \leq \varphi_{\max} a_{\max}^{\alpha} e^A \left( \frac{3}{\alpha + 1} \right)^{\alpha + 1}
    \leq \varphi_{\max} e^A \left( \frac{3 a_{\max}}{\alpha + 1} \right)^{\alpha + 1}.
\]
The choice $\alpha = \lceil (3e a_{\max}) \vee (A + \log \varphi_{\max} + \log(1 / \eps_0)) \rceil \lesssim \log(1 / \eps_0)$ ensures that 
\[
    \|\widetilde \Phi - \Phi\|_{L^\infty([0, 1]^r)}
    \leq e^A \left( \frac{3 a_{\max}}{\alpha + 1} \right)^{\alpha + 1}
    \leq e^{A - \alpha - 1}
    \leq \eps_0.
\]
Moreover, the integers $m$ and $N$ satisfy the bounds
\[
    m \lesssim \left(r + \log\frac1{\eps_0}\right) \log \log\frac1{\eps_0}
    \quad \text{and} \quad
    N = (\alpha + 1)^r \lesssim \left( \log\frac1{\eps_0} \right)^r.
\]
Hence, the architecture of the neural network $\tilde\Phi$ is such that its depth $L$, width $W$, and the number of non-zero weights $S$ fulfil the inequalities
\begin{align*}
    &
    L
    \lesssim \left(r + \log\frac1{\eps_0}\right) \log\left( r + \log\frac1{\eps_0}\right) \log \log\frac1{\eps_0},
    \\&
    \|W\|_\infty
    \lesssim \left(r \vee \log\frac1{\eps_0}\right) \left( \log\frac1{\eps_0} \right)^r
    \lesssim \left( r + \log\frac1{\eps_0} \right)^{r + 1},
\end{align*}
and
\[
    S
    \lesssim \left(r + \log\frac1{\eps_0}\right)^{r + 3} \cdot \left( \log\frac1{\eps_0} \right)^r \cdot \left(r + \log\frac1{\eps_0}\right) \log\log\frac1{\eps_0}
    \lesssim \left(r + \log\frac1{\eps_0}\right)^{2r + 5}.
\]
The hidden constants behind $\lesssim$ depend on $\varphi_{\max}$, $a_{\max}$, and $A$ only.

\endproof

\subsection{Proof of Lemma \ref{lem:infty_norms_bounds}}
\label{sec:lem_infty_norms_bounds_proof}
We first provide an upper bound for $\|V_{\bj, 0}\|_{L^\infty(\cC_{[t_0, T]}^*)}$
\begin{align*}
    \|V_{\bj, 0}\|_{L^\infty(\cC_{[t_0, T]}^*)} 
    &= \sup_{(t, y) \in \cC_{[t_0, T]}^*}\frac{\|y - m_tg^\circ_\bj(u_\bj)\|^2}{2\tilde\sigma_t^2} \\
    &\leq \sup_{(t, y) \in \cC_{[t_0, T]}^*}\inf_{u \in [0, 1]^d}\left(\frac{\|y - m_t g^*(u)\|^2}{\tilde\sigma_t^2} + \frac{m_t^2 \|g^*(u) - g^*(u_\bj)\|^2}{\tilde\sigma_t^2}\right).
\end{align*}
Next, the definitions of $\cK_t$ \eqref{eq:k_m_t_def} and $\cC_{[t_0, T]}^*$ \eqref{eq:c_star_def}
\begin{align*}
    \|V_{\bj, 0}\|_{L^\infty(\cC_{[t_0, T]}^*)}
    \leq \sup_{(t, y) \in \cC_{[t_0, T]}^*}\frac{R_t^2}{\tilde\sigma_t^2} + \sup_{u \in [0, 1]^d}\frac{m_t^2 \|g^*(u) - g^*(u_\bj)\|^2}{\tilde\sigma_t^2}.
\end{align*}
Finally, recall that $\|g^*\|_{L^\infty([0, 1]^d)} \leq 1$ and also
\begin{align*}
    R_t = 16 \tilde\sigma_t \left( \sqrt{D \log\left( \frac{\eps^{-2\beta}}D \right)} \vee \log\left( \frac{\eps^{-2\beta}}D \right) \right),
\end{align*}
as suggested by \eqref{eq:rt_def}.
Thus, it holds that
\begin{align*}
    \|V_{\bj, 0}\|_{L^\infty(\cC_{[t_0, T]}^*)}
    &\lesssim D \log^2\left(\frac{\eps^{-2\beta}}{D}\right) + \frac{m_{t_0}^2}{\tilde{\sigma}_{t_0}^2}.
\end{align*}
Next, provide an upper bound for $\|V_{\bj, \bk}\|_{L^\infty(\cC^*_{[t_0, T]})}$ with $\bk \in \Z^d_+, 1 \leq |\bk| \leq \floor{\beta}$
\begin{align*}
    \|V_{\bj, \bk}\|_{L^\infty(\cC^*_{[t_0, T]})}
    &= \sup_{(t, y) \in \cC^*_{[t_0, T]}}\left|\frac{m_t}{\tilde\sigma_t^2}(y - m_t g^*(u_\bj))^\top\partial^\bk g^*(u_\bj) \right| \\
    &\lesssim \sup_{(t, y) \in \cC^*_{[t_0, T]}}\frac{\sqrt{D} m_t\|y - m_tg^*(u_\bj)\|}{\tilde\sigma_t^2}.
\end{align*}
Using the identical argument as above, we deduce that
\begin{align*}
    \|V_{\bj, \bk}\|_{L^\infty(\cC^*_{[t_0, T]})}
    &\lesssim\sup_{(t, y) \in \cC^*_{[t_0, T]}}\frac{\sqrt{D}m_t}{\tilde\sigma_t^2}\inf_{u\in [0,1]^d}\left(\|y - m_t g^*(u)\| + m_t\|g^*(u) - g^*(u_\bj)\| \right) \\
    &\lesssim \sup_{t \in [t_0, T]}\frac{\sqrt{D}m_{t}}{\tilde\sigma_{t}^2}(R_t + m_t) \\
    &\lesssim D\left(\frac{m_{t_0}^2}{\tilde\sigma_{t_0}^2} + \frac{m_{t_0}}{\tilde\sigma_{t_0}}\log\left(\frac{\eps^{-2\beta}}{D}\right) \right).
\end{align*}
The bound for $\|\cV\|_{L^\infty([t_0, T])}$ is trivial
\begin{align*}
    \|\cV\|_{L^\infty([t_0, T])} = \sup_{t \in [t_0, T]} \frac{m_t^2}{2\tilde\sigma_t^2}
    \leq \frac{m_{t_0}^2}{2\tilde\sigma_{t_0}^2}.
\end{align*}
The proof is thus complete.

\endproof

\subsection{Proof of Lemma \ref{lem:div_on_segment}}
\label{sec:lem_div_on_segment_proof}
We use the observation of  \cite{telgarsky2017neural} who noted that
$\frac{x}{y} = \frac{x}{b}\sum_{i=0}^{\infty}(1 - \frac{y}{b})^i$ for any $x \in \R$ and $y \in [a, b]$.
Consider $x, x', y, y'$ as specified in the statement and bound the approximation accuracy error for some $r \in \N$ using the observation that $|x| \leq y$
and the fact that $1 + x \leq e^{x}$ for any $x \in \R$
\begin{align*}
    \left| \frac{x}{b}\sum_{i=0}^r\left(1 - \frac{y}{b}\right)^i - \frac{x}{y}\right|
    &\leq \frac{|x|}{b}\left(1 - \frac{y}{b}\right)^{r+1}\sum_{i=0}^\infty\left(1 - \frac{y}{b}\right)^i 
    \leq \frac{|x|}{y}\left(1 - \frac{y}{b}\right)^r
    \leq \exp\left\{-\frac{ra}{b}\right\}.
\end{align*}
The choice $r = \frac{b}{a}\ceil{\log(1/\eps)}$ ensures that
\begin{align}\label{eq:div_on_seg_taylor_acc}
    \left| \frac{x}{b}\sum_{i=0}^r\left(1 - \frac{y}{b}\right)^i - \frac{x}{y}\right| \leq \eps.
\end{align}
For any $i \in \{1, \dots, r - 1\}$, let $h_i(x_1, x_2)$ be a ReLU-network from Lemma \ref{lem:multi_oko} that approximates a monomial $x_1 x_2^i$ for $x_1, x_2 \in [-1, 1]$ and accuracy parameter $\eps$.
Note that our subsequent analysis encompasses the case $i = 0$, as the ReLU network in this case is exact.
Hence, the sensitivity analysis from Lemma \ref{lem:multi_oko} implies that for any $x_1, x_2 \in [-1, 1]$ and $|x_1' - x_1| \vee |x_2' - x_2| \leq \eps' \leq 1$ we have
\begin{align}\label{eq:div_segm_sensit_h}
    \left|h_i(x_1', x_2') - x_1 x_2^i\right| \leq \eps + (i + 1)\eps'.
\end{align}
In addition, the value of $h_i(x_1', x_2')$ is bounded by
\begin{equation}\label{eq:q_div_h_i_bound}
    |h_i(x_1', x_2')| \leq 1.
\end{equation}
Let us consider
\begin{align}\label{eq:q_div_approx_a_b}
    q(x, y) := \frac{1}{b}\sum_{i=0}^r h_i\left(1 - \frac{y}{b}, x\right).
\end{align}
The function $q$ is a good approximation for $x/y$ in view of the bounds \eqref{eq:div_on_seg_taylor_acc} and \eqref{eq:div_segm_sensit_h}. Indeed, it holds that
\begin{equation}\label{eq:q_a_b_acc_interm}
\begin{split}
    \left|q(x', y') - \frac{x}{y}\right| &\leq \left|q(x', y') - \frac{x}{b}\sum_{i=0}^r\left(1 - \frac{y}{b}\right)^i\right| + \left| \frac{x}{b}\sum_{i=0}^r\left(1 - \frac{y}{b}\right)^i - \frac{x}{y}\right| \\
    &\leq \frac{1}{b}\sum_{i=0}^r\left|h_i\left(1 - \frac{y'}{b}, x'\right) - \left(1 - \frac{y}{b}\right)^ix\right| + \eps \\
    &\leq \frac{1}{b}\sum_{i=0}^r\left(\eps + \frac{(i + 1)\eps'}{b}\right) + \eps \\
    &\leq \frac{(r + 1)^2}{b^2}(\eps + \eps') + \eps.
\end{split}
\end{equation}
By substituting the expression for $r$ into \eqref{eq:q_a_b_acc_interm}, we obtain the desired result \eqref{eq:div_segment_acc}.
Now equation \eqref{eq:q_div_approx_a_b} and the configuration of $h_i$, namely, $L(h_i) \lesssim \log(i + 1)\log(1 / \eps)$,
$\|W\|_\infty(h_i) \lesssim (i + 1)$, $S(h_i) \lesssim (i + 1)\log(1 / \eps)$, $B(h_i) \lesssim 1$, imply that the resulting network $q$ has
\begin{align*}
    L(q) &\lesssim \log(r + 1)\log(1/\eps) \lesssim (\log(b/a) + \log(\log(1/\eps) \vee e) )\log(1/\eps), \\
    \|W\|_\infty &\lesssim (r + 1)^2 \lesssim (b/a)^2\log^2(1/\eps), \\
    S &\lesssim (r + 1)^2\log(1/\eps) \lesssim (b/a)^2\log^3(1/\eps), \\
    B &\lesssim b^{-2},
\end{align*}
where the last inequality stems from the fact that the first weight matrix of each $h_i$ and the output layer parameters are multiplied
by $b^{-1}$ with the potential coincidence.
Finally, leveraging the bound for $h$ provided in \eqref{eq:q_div_h_i_bound} and the definition of $q$ outlined in \eqref{eq:q_div_approx_a_b}, we arrive at
\begin{align*}
    \left|q(x', y')\right| \leq \frac{r + 1}{b} \lesssim \frac{1}{a}\log(1/\eps).
\end{align*}
This concludes the proof.
\endproof

\section{Proof of Theorem \ref{thm:estim_enh_zero}}
\label{sec:thm_estim_enh_zero_proof}

The proof of Theorem \ref{thm:estim_enh_zero} is quite technical, so we split it into several steps for convenience.

\medskip
\noindent
\textbf{Step 1: Bernstein's condition.} \quad We start with the following technical lemma derived in Appendix \ref{sec:lem_losses_diff_bound_proof}, which allows us to exploit the loss curvature.
\begin{Lem}
    \label{lem:losses_diff_bound}
    Let $h : \R^D \times [t_0, T] \to \R^D$ and $h' : \R^D \times [t_0, T] \to \R^D$ be any Borel functions such that
    \begin{align*}
        \|h\|_{L^\infty(\R^D \times [t_0, T])} \vee \|h'\|_{L^\infty(\R^D \times [t_0, T])} \leq 2 .
    \end{align*}
    Consider the corresponding score functions
    \begin{equation}
        \label{eq:score_function_candidate}
        s(y, t) = -\frac{y}{m_t^2\sigma^2 + \sigma_t^2} + \frac{m_t}{m_t^2\sigma^2 + \sigma_t^2} h(y, t)
    \end{equation}
    and
    \begin{equation}
        \label{eq:score_function_candidate_prime}
        s'(y, t) = -\frac{y}{m_t^2(\sigma')^2 + \sigma_t^2} + \frac{m_t}{m_t^2(\sigma')^2 + \sigma_t^2} h'(y, t) ,
    \end{equation}
    where $\sigma$ and $\sigma'$ are some constants from $[0, 1)$.
    Then, for any $x \in \R^D$, it holds that
    \begin{align*}
        (\ell(s, x) - \ell(s', x))^2
        &
        \leq 48 \left( \frac{\|x\|^2 + 1}{\sigma_{t_0}^2} + D \log(1 / \sigma_{t_0}^2) + D (T - t_0) \right)
        \\&\quad
        \cdot \left(\integral{t_0}^T \E_{X_t | X_0 = x}\|s(X_t, t) - s'(X_t, t)\|^2 \dd t\right).
    \end{align*}
    In addition, we have
    \begin{equation}
        \label{eq:scores_diff_exp_berns}
        \integral{t_0}^T \E_{X_t | X_0 = x}\|s(X_t, t) - s'(X_t, t)\|^2 \dd t
        \leq \frac{2(\|x\|^2 + 8)}{\sigma_{t_0}^2} + 2D \left( \log(1 / \sigma_{t_0}^2) + 2(T - t_0) \right).
    \end{equation}
\end{Lem}
Lemma \ref{lem:losses_diff_bound} helps us to verify Bernstein's condition for the excess loss class
\begin{equation}
    \label{eq:loss_class}
    \cL = \big\{ \ell(s, \cdot) - \ell(s^*, \cdot) : s \in \cS \big\},
\end{equation}
which is a starting point on a way to the high-probability upper bound on the risk of $\widehat s$.
Let is fix some $\varkappa \geq 2$ to be determined a bit later (see \eqref{eq:kappa_bernstein_def} below). Applying Hölder's inequality with the parameters $p = \varkappa$ and $q = (1 - 1 / \varkappa)^{-1}$, we obtain that
\begin{align*}
    &\E_{X_0}(\ell(s, X_0) - \ell(s^*, X_0))^2 \\
    &
    \leq \left\{\E_{X_0} \left( \frac{48(\|X_0\|^2 + 1)D(T - t_0)\log(1 / \sigma_{t_0}^2)}{\sigma_{t_0}^2} \right)^\varkappa \integral{t_0}^T\E_{X_t | X_0 = x}\|s(X_t, t) - s^*(X_t, t)\|^2 \dd t \right\}^{1 / \varkappa} \\
    &\quad
    \cdot \left\{ \integral{t_0}^T \E_{X_t} \|s(X_t, t) - s^*(X_t, t)\|^2 \dd t \right\}^{1 - 1 / \varkappa}.
\end{align*}
The bound \eqref{eq:scores_diff_exp_berns} from Lemma \ref{lem:losses_diff_bound} implies that
\begin{align*}
    \E_{X_0}(\ell(s, X_0) - \ell(s^*, X_0))^2
    &
    \leq \left\{ \E_{X_0} \left( \frac{48(\|X_0\|^2 + 1)D(T - t_0)\log(1 / \sigma_{t_0}^2)}{\sigma_{t_0}^2} \right)^{\varkappa + 1} \right\}^{1 / \varkappa}
    \\&\quad
    \cdot \left\{ \integral{t_0}^T \E_{X_t} \|s(X_t, t) - s^*(X_t, t)\|^2 \dd t \right\}^{1 - 1 / \varkappa}.
\end{align*}
Then, taking into account \eqref{eq:vincent_int_t0_T}, we conclude that 
\begin{align*}
    \E_{X_0}(\ell(s, X_0) - \ell(s^*, X_0))^2
    &
    \leq \left\{ \E_{X_0} \left( \frac{48(\|X_0\|^2 + 1)D(T - t_0)\log(1 / \sigma_{t_0}^2)}{\sigma_{t_0}^2} \right)^{\varkappa + 1} \right\}^{1 / \varkappa}
    \\&\quad
    \cdot \Big\{\E_{X_0} \big[\ell(s, X_0) - \ell(s^*, X_0) \big] \Big\}^{1 - 1 / \varkappa}.
\end{align*}
Let us recall that, according to Assumption \ref{as:data_distr}, we have $X_0 = g^*(U) + \sdata Z$, where $\|g\|_{L^\infty([0, 1]^d)}$ does not exceed $1$ and the random vectors $U \sim \mathrm{Un}([0, 1]^d)$ and $Z \sim \cN(0, I_D)$ are independent.
In view of the triangle inequality, we obtain that
\begin{align*}
    \left( \E_{X_0} \left( \|X_0\|^2 + 1 \right)^{\varkappa + 1} \right)^{1 / (\varkappa + 1)}
    &
    \leq \left( \E \|X_0\|^{2\varkappa + 2} \right)^{1 / (\varkappa + 1)} + 1
    \\&
    = \left( \E \|g^*(U) + Z\|^{2\varkappa + 2} \right)^{1 / (\varkappa + 1)} + 1
    \\&
    = \left( \E \left(\|Z\|^2 + 1 \right)^{\varkappa + 1} \right)^{1 / (\varkappa + 1)} + 1
    \\&
    \leq \left( \E \|Z\|^{2\varkappa + 2} \right)^{1 / (\varkappa + 1)} + 2.
\end{align*}
Let us note that $\|Z\|^2 \sim \chi^2(D)$ is a sub-exponential random variable (see Remark \ref{rem:chi-squared_concentration}). Applying \citep[Proposition 2.7.1]{vershynin2018high}, we obtain that
\[
    \left( \E \|Z\|^{2\varkappa + 2} \right)^{1 / (\varkappa + 1)}
    \lesssim D (\varkappa + 1),
\]
and, as a consequence,
\begin{align*}
    \left(\E_{X_0}(\|X_0\|^2 + 1)^{\varkappa + 1} \right)^{1 / (\varkappa + 1)}
    \leq \left( \E \|Z\|^{2\varkappa + 2} \right)^{1 / (\varkappa + 1)} + 2
    \lesssim D (\varkappa + 1).
\end{align*}
Hence, the Bernstein condition is now verified, since
\begin{align}
    \label{eq:bernst_cond_positive_sdata}
    \notag
    &
    \E_{X_0}(\ell(s, X_0) - \ell(s^*, X_0))^2
    \\&
    \lesssim \left(\frac{D^2(T - t_0)(1 + \varkappa)\log(1 / \sigma_{t_0}^2)}{\sigma_{t_0}^2}\right)^{1 + 1 /\varkappa} \left\{\E_{X_0}[\ell(s, X_0) - \ell(s^*, X_0)]\right\}^{1 - 1 / \varkappa}
    \\& \notag
    \lesssim \left(\frac{D^2 T (1 + \varkappa)\log(1 / \sigma_{t_0}^2)}{\sigma_{t_0}^2}\right)^{1 + 1 /\varkappa} \left\{\E_{X_0}[\ell(s, X_0) - \ell(s^*, X_0)]\right\}^{1 - 1 / \varkappa}.
\end{align}
The rest of the proof relies on Bernstein's inequality and the $\eps$-net argument.
For simplicity, we split it into several steps.

\medskip
\noindent
\textbf{Step 2: Bernstein's large deviation bound.} \quad
For any $s \in \cS_0$, let us denote
\[
    \widehat \E_{X_0} \big( \ell(s, X_0) - \ell(s^*, X_0) \big)
    = \frac1n \sum\limits_{i = 1}^n \big( \ell(s, Y_i) - \ell(s^*, Y_i) \big)
\]
where the samples $Y_1, \dots, Y_n$ are drawn independently from the same distribution as $X_0$. The goal of this step is to provide a high-probability upper bound on
\[
    \E_{X_0} \big( \ell(s, X_0) - \ell(s^*, X_0) \big) - \widehat \E_{X_0} \big( \ell(s, X_0) - \ell(s^*, X_0) \big)
\]
for a fixed $s \in \cS_0$.
For this purpose, we use Bernstein's inequality for unbounded random variables \citep[Proposition 5.2]{lecue2012oracle}.
The following Lemma ensures that for all $s \in \cS_0$ the random variable $\ell(s, X_0) - \ell(s^*, X_0)$ has a bounded $\|\cdot\|_{\psi_1}$ norm.

\begin{Lem}
    \label{lem:loss_orlicz_norm_bound}
    Under Assumption \ref{as:data_distr}, we let $h : \R^D \times [t_0, T] \to \R^D$ be a Borel function satisfying $\|h\|_{L^\infty(\R^D \times [t_0, T])} \leq 2$.
    Consider the corresponding score function surrogate
    \[
        s(y, t) = -\frac{y}{m_t^2\sigma^2 + \sigma_t^2} + \frac{m_t \, h(y, t)}{m_t^2\sigma^2 + \sigma_t^2}, \quad \sigma \in [0, 1).
    \]
    Then it holds that
    \begin{align*}
        \|\ell(s, X_0)\|_{\psi_1} \lesssim D \log(\sigma_{t_0}^{-2}) + D(T - t_0) + \frac{D\sdata^2 + 1}{\sigma_{t_0}^2},
    \end{align*}
    where the hidden constant does not depend on $s$ and the parameters $D$, $T$, $t_0$, and $\sdata$.
    
\end{Lem}
The proof of Lemma \ref{lem:loss_orlicz_norm_bound} is deferred to Appendix \ref{sec:lem_loss_orlicz_norm_bound_proof}.
According this lemma, for any $s \in \cS_0$, it holds that
\begin{align}
    \label{eq:psi_1_loss_diff_simpl}
    \|\ell(s, X_0) - \ell(s^*, X_0)\|_{\psi_1} 
    &\notag
    \leq \|\ell(s, X_0)\|_{\psi_1} + \|\ell(s^*, X_0)\|_{\psi_1}
    \\&
    \lesssim \frac{D  + \log(\sigma_{t_0}^{-2})}{\sigma_{t_0}^2} + D(T - t_0).
\end{align}
Therefore, applying the Bernstein inequality for unbounded random variables \citep[Proposition 5.2]{lecue2012oracle}, we obtain that, for a fixed $s \in \cS_0$ and any $\delta \in (0, 1)$, with probability at least $(1 - \delta / 2)$ it holds that
\begin{align*}
    &\left| \E_{X_0} \left[ \ell(s, X_0) - \ell(s^*, X_0) \right] - \widehat \E_{X_0} \left[ \ell(s, X_0) - \ell(s^*, X_0) \right] \right| \\
    &
    \lesssim \sqrt{\frac{\Var_{X_0}\left[ \ell(s, X_0) - \ell(s^*, X_0) \right] \log(4 / \delta)}n} + \frac{ \|\ell(s, X_0) - \ell(s^*, X_0)\|_{\psi_1} \log n \log(4 / \delta)}{n},
\end{align*}
Using \eqref{eq:bernst_cond_positive_sdata} with
\begin{align}
    \label{eq:kappa_bernstein_def}
    \varkappa = 2 \vee \big( \log n + \log(\sigma_{t_0}^{-2}) \big)
\end{align}
and \eqref{eq:psi_1_loss_diff_simpl}, we deduce that, on the same event, the following holds true:
\begin{align}
    \label{eq:bernstein_bound}
    \notag
    &\left| \E_{X_0} [\ell(s, X_0) - \ell(s^*, X_0)] - \widehat \E_{X_0}[\ell(s, X_0) - \ell(s^*, X_0)] \right| \\
    &
    \lesssim \sqrt{\frac{C_{\mathrm{b}}\{\E_{X_0}[\ell(s, X_0) - \ell(s^*, X_0)]\}^{1 - 1 / \varkappa}\log(4 / \delta)}{n}} + \frac{C_{\mathrm{b}}\log(4 / \delta)}{n},
\end{align}
where we denote the constant associated with the above Bernstein-type inequality as
\begin{align}
    \label{eq:const_bernst_large_dev}
    C_{\mathrm{b}} =  \frac{D^3 T^2\log^2(\sigma_{t_0}^{-2}) \log n}{\sigma_{t_0}^{2 + 2 / \varkappa}} .
\end{align}

\medskip
\noindent
\textbf{Step 3: $\eps$-net argument and a uniform bound.}\quad
Our next goal is to derive a uniform large deviation bound based on \eqref{eq:bernstein_bound}. We rely on the standard $\eps$-net argument.
The next result offers an estimation of the covering for the class of denoising score matching loss functions.
\begin{Lem}
\label{lem:cov_number_unbounded_eval}
    For any $\delta \in (0, 1)$ and $\tau \in (0, 1)$ there exists a subclass of score estimators $\cS_\tau \subseteq \cS(L, W, S, B)$ (see Definition \ref{def:score_class}) satisfying
    \begin{align}
        \label{eq:s_s_tau_sum_exp_prox}
        \sup_{s \in \cS(L, W, S, B)}\inf_{s_\tau \in \cS_\tau} \left\{ |\E_{X_0}[\ell(s, X_0) - \ell(s_\tau, X_0)]| + |\widehat{\E}_{X_0}[\ell(s, X_0) - \ell(s_\tau, X_0)]| \right\} \leq \tau,
    \end{align}
    with probability at least $1 - \delta$.
    Furthermore, it holds that
    \begin{align*}
        \log |\cS_\tau| \lesssim SL \log(\tau^{-1} L(\|W\|_\infty + 1)(B \vee 1) D T \sigma_{t_0}^{-2}\log(n / \delta)) .
    \end{align*}
    
\end{Lem}
We move the proof of Lemma \ref{lem:cov_number_unbounded_eval} to Appendix \ref{sec:lem_cov_number_unbounded_eval_proof}.
Now using the union bound and applying Lemma \ref{lem:cov_number_unbounded_eval} for the confidence parameter $\delta / 2$ and the precision parameter $\tau \in (0, 1)$, which will be determined later in the proof, it follows from \eqref{eq:bernstein_bound} that there exists an event with probability at least $(1 - \delta / 2)$, such that
\begin{align}
    \label{eq:bernstein_bound_eps_net}
    &\notag
    \left| \E_{X_0} [\ell(s_\tau, X_0) - \ell(s^*, X_0)] - \widehat \E_{X_0}[\ell(s_\tau, X_0) - \ell(s^*, X_0)] \right| \\
    &
    \lesssim \sqrt{\frac{C_{\mathrm{b}}\{\E_{X_0}[\ell(s_\tau, X_0) - \ell(s^*, X_0)]\}^{1 - 1 / \varkappa}\log(4 |\cS_\tau| / \delta)}{n}} + \frac{C_{\mathrm{b}}\log(4 |\cS_\tau| / \delta)}{n} 
\end{align}
simultaneously for all $s_\tau \in \cS_\tau$.
Let us restrict our attention on the event $\cE$ with $\p(\cE) \geq 1 - \delta$, where both \eqref{eq:bernstein_bound_eps_net} and the statement of Lemma \ref{lem:cov_number_unbounded_eval} hold.
Let $s_\tau \in \cS_\tau$ be the nearest element to an arbitrary $s \in \cS(L, W, S, B)$ such that \eqref{eq:s_s_tau_sum_exp_prox} holds.
Therefore, it follows from \eqref{eq:bernstein_bound_eps_net} that
\begin{align*}
    &
    \left| \E_{X_0} [\ell(s, X_0) - \ell(s^*, X_0)] - \widehat \E_{X_0}[\ell(s, X_0) - \ell(s^*, X_0)] \right| \\
    &
    \leq \tau + \left| \E_{X_0} [\ell(s_\tau, X_0) - \ell(s^*, X_0)] - \widehat \E_{X_0}[\ell(s_\tau, X_0) - \ell(s^*, X_0)] \right| \\
    &
    \lesssim \tau + \sqrt{\frac{C_{\mathrm{b}}\{\E_{X_0}[\ell(s_\tau, X_0) - \ell(s^*, X_0)]\}^{1 - 1 / \varkappa} \log(4 |\cS_\tau| / \delta)}{n}} + \frac{C_{\mathrm{b}}\log(4 |\cS_\tau| / \delta)}{n} \\
    &
    \lesssim \tau + \sqrt{\frac{C_{\mathrm{b}}\{\tau + \E_{X_0}[\ell(s, X_0) - \ell(s^*, X_0)]\}^{1 - 1 / \varkappa}\log(4 |\cS_\tau| / \delta)}{n}} + \frac{C_{\mathrm{b}}\log(4 |\cS_\tau| / \delta)}{n} .
\end{align*}
We next note that
\begin{align*}
    &
    \left( \tau + \E_{X_0}[\ell(s, X_0) - \ell(s^*, X_0)] \right)^{(1 - 1 / \varkappa) /  2}
    \\&
    \leq 2 \tau^{(1 - 1 / \varkappa)/2} + 2\{ \E_{X_0}[\ell(s, X_0) - \ell(s^*, X_0)] \}^{(1 - 1 / \varkappa) / 2},
\end{align*}
which together with the Young inequality leads to
\begin{align*}
    &\sqrt{\frac{C_{\mathrm{b}}\{\tau + \E_{X_0}[\ell(s, X_0) - \ell(s^*, X_0)]\}^{1 - 1 / \varkappa}\log(4 |\cS_\tau| / \delta)}{n}} \\
    &
    \lesssim \frac{C_{\mathrm{b}}\log(4|\cS_\tau| / \delta)}{n} + \tau^{1 - 1 / \varkappa} + \sqrt{\frac{C_{\mathrm{b}} \left\{\E_{X_0}[\ell(s, X_0) - \ell(s^*, X_0)]\right\}^{1 - 1 / \varkappa} \log(4|\cS_\tau| / \delta)}{n}} .
\end{align*}
Hence, on the event $\cE$ of probability at least $(1 - \delta)$, it holds that
\begin{align}
    \label{eq:bernstein_unif_bound}
    &\notag
    \left| \E_{X_0}[\ell(s, X_0) - \ell(s^*, X_0)] - \widehat \E_{X_0}[\ell(s, X_0) - \ell(s^*, X_0)]  \right| \\
    &\quad \lesssim \tau^{1 - 1 / \varkappa} + \sqrt{\frac{C_{\mathrm{b}}\{\E_{X_0}[\ell(s, X_0) - \ell(s^*, X_0)]\}^{1 - 1 / \varkappa}\log(4 |\cS_\tau| / \delta)}{n}} + \frac{C_{\mathrm{b}}\log(4 |\cS_\tau| / \delta)}{n}
\end{align}
simultaneously for all $s \in \cS$.

\medskip
\noindent
\textbf{Step 4: final bound for $\widehat s$.}
\quad
The upper bound on the excess risk of the denoising score matching estimate $\widehat s$ easily follows from the 
uniform bound \eqref{eq:bernstein_unif_bound}. Let us recall that $\widehat s$ minimizes the empirical risk
\[
    \widehat \E_{X_0} \ell(s, X_0) = \frac1n \sum\limits_{i = 1}^n \ell(s, Y_i)
\]
over the class $\cS(L, W, S, B)$. Let us take $\eps \in (0, 1)$ satisfying the conditions of Theorem \ref{th:score_approximation} and let $\bar{s}$ be the score from Theorem \ref{th:score_approximation} such that
\[
    \E_{X_0}[\ell(\bar{s}, X_0) - \ell(s^*, X_0)]
    = \int\limits_{t_0}^T \E_{X_t} \left\|s^*(X_t, t) - \bar{s}(X_t, t) \right\|^2 \, \dd t
    \lesssim \frac{D \eps^{2\beta}}{\sigma_{t_0}^2}.
\]
This observation together with \eqref{eq:bernstein_unif_bound} implies that
\begin{align*}
    &\widehat{\E}_{X_0}[\ell(\bar{s}, X_0) - \ell(s^*, X_0)] \\
    &\quad \lesssim \frac{D\eps^{2\beta}}{\sigma_{t_0}^2} + \tau^{1 - 1 / \varkappa} + \sqrt{\frac{C_{\mathrm{b}}\{ D\eps^{2\beta} \}^{1 - 1 / \varkappa}\log(4 |\cS_\tau| / \delta)}{n \cdot \sigma_{t_0}^{2 - 2 / \varkappa} }} + \frac{C_{\mathrm{b}}\log(4 |\cS_\tau| / \delta)}{n} \\
    &\quad \lesssim \tau^{1 - 1 / \varkappa} + \frac{C_{\mathrm{b}}\log(4 |\cS_\tau| / \delta)}{n} + \left(\frac{D\eps^{2\beta}}{\sigma_{t_0}^2}\right)^{1 - 1/ \varkappa},
\end{align*}
where the last line uses Young's inequality.
From the above bound in conjunction with \eqref{eq:bernstein_unif_bound} and the fact that $\widehat{\E}_{X_0}[\ell(\widehat{s}, X_0)] \leq \widehat{\E}_{X_0}[\ell(\bar{s}, X_0)]$ it follows that
\begin{align*}
    &\E_{X_0}[\ell(\widehat{s}, X_0) - \ell(s^*, X_0)] \lesssim \tau^{1 - 1 / \varkappa} + \left(\frac{D\eps^{2\beta}}{\sigma_{t_0}^2}\right)^{1 - 1/ \varkappa} \\
    &\quad + \sqrt{\frac{C_{\mathrm{b}}\{\E_{X_0}[\ell(\widehat{s}, X_0) - \ell(s^*, X_0)]\}^{1 - 1 / \varkappa}\log(4 |\cS_\tau| / \delta)}{n}} + \frac{C_{\mathrm{b}}\log(4 |\cS_\tau| / \delta)}{n}  .
\end{align*}
The derived bound and \eqref{eq:vincent_int_t0_T} imply that with probability at least $(1 - \delta)$,
\begin{align}
\label{eq:gen_bound_raw_zero}
    \integral{t_0}^T\E_{X_t}\|\hat{s}(X_t, t) - s^*(X_t, t)\|^2 \, \dd t
    \lesssim \tau^{1 - 1 / \varkappa} + \left(\frac{D\eps^{2\beta}}{\sigma_{t_0}^2}\right)^{1 - 1/ \varkappa} + \frac{(C_{\mathrm{b}} \vee 1)\log(4 |\cS_\tau| / \delta)}{n^{1 / (1 + 1 / \varkappa)}} ,
\end{align}
since the inequality $x \leq a\sqrt{x^{1 - 1 / \varkappa}} + b$ implies $x \leq 2a^{2 / (1 + 1 / \varkappa)} + 2b$ for non-negative $a$, $b$ and $x$.
It remains to specify $\tau$ and $\eps$.
Let us remind the reader that, due to Lemma \ref{lem:cov_number_unbounded_eval}, we have
\[
    \log |\cS_\tau| \lesssim SL \log(\tau^{-1} L(\|W\|_\infty + 1)(B \vee 1) D T \sigma_{t_0}^{-2}\log(n / \delta)) ,
\]
where the configuration
\begin{equation}\label{eq:s_zero_cfg}
\begin{split}
    &L \lesssim D^2\log^4(1/\eps), \quad \|W\|_\infty \lesssim D^5\eps^{-d}\left(\frac{1}{t_0} \vee 1\right)\log^6(1 / \eps) \\
    &S \lesssim D^{6 + 2\binom{d + \floor{\beta}}{d}}\eps^{-d}\left(\frac{1}{t_0} \vee 1\right)\log^{10 + 4\binom{d + \floor{\beta}}{d}}(1/\eps), \quad \log B \lesssim D\log^2(1/\eps)
\end{split}
\end{equation}
is suggested by Theorem \ref{th:score_approximation}.
Hence, for $\tau, t_0, \eps \in (0, 1)$ it holds that
\begin{align*}
    \log |\cS_\tau|
    \lesssim D^{9 + 2\binom{d + \floor{\beta}}{d}}\eps^{-d}  t_0^{-1} (\log(1 / \eps))^{17 + 4\binom{d + \floor{\beta}}{d}}\log(1 / \tau) \log T \log D \log(1 / t_0) \log(n / \delta).
\end{align*}
Setting setting $\tau = \eps^{2\beta} \in (0, 1)$ and combining this result with \eqref{eq:const_bernst_large_dev} and \eqref{eq:gen_bound_raw_zero}, we deduce that with probability at least $(1 - \delta)$,
\begin{align*}
    &\integral{t_0}^T\E_{X_t}\|\hat{s}(X_t, t) - s^*(X_t, t)\|^2 \, \dd t \\
    & \quad \lesssim \left( \left(\frac{D\eps^{2\beta}}{\sigma_{t_0}^2}\right)^{1 - 1 / \varkappa} + \frac{D^{12 + 2\binom{d + \floor{\beta}}{d}} T^2\eps^{-d}}{t_0 \cdot \sigma_{t_0}^{2 + 2 / \varkappa} n^{1 / (1 + 1 / \varkappa)}} \right)  L'(t_0, \eps) \log(4 / \delta),
\end{align*}
where the logarithmic factors are embedded within the expression
\begin{align*}
    L'(t_0, \eps) = (\log(1 / \eps))^{18 + 4\binom{d + \floor{\beta}}{d}} \log T \log D \log^3(1 / t_0) \log^2 n.
\end{align*}
Now setting $\eps = (\sigma_{t_0}^2 n)^{-\frac{1}{2\beta + d}}$, which ensures that the sample size satisfies \eqref{eq:sample_size_zero_lb}, and observing that $\sigma_{t_0}^2 \asymp t_0$ for $t_0 \leq 1$, we have that 
\begin{align*}
     \integral{t_0}^T\E_{X_t}\|\hat{s}(X_t, t) - s^*(X_t, t)\|^2 \, \dd t 
    \lesssim \frac{T^2 D^{12 + 2\binom{d + \floor{\beta}}{d}} }{\sigma_{t_0}^{2 + 2 / \varkappa}}(n\sigma_{t_0}^2)^{-\frac{2\beta}{2\beta + d}} n^{\frac{1}{1 + \varkappa}} L(t_0, n) \log(4 / \delta)
\end{align*}
holds with probability at least $(1 - \delta)$.
Here we introduced $L(t_0, n) = L'(t_0, n^{-1})$.
The choice of $\varkappa$ given in \eqref{eq:kappa_bernstein_def} ensures that
\begin{align*}
     \integral{t_0}^T\E_{X_t}\|\hat{s}(X_t, t) - s^*(X_t, t)\|^2 \, \dd t 
    \lesssim \frac{T^2 D^{12 + 2\binom{d + \floor{\beta}}{d}}}{\sigma_{t_0}^2}(n\sigma_{t_0}^2)^{-\frac{2\beta}{2\beta + d}} L(t_0, n) \log(4 / \delta).
\end{align*}
Finally, substituting the optimized $\eps$ into the configuration outlined in \eqref{eq:s_zero_cfg} completes the proof.

\endproof

\subsection{Proof of Lemma \ref{lem:losses_diff_bound}}
\label{sec:lem_losses_diff_bound_proof}

First, let us note that
\begin{align*}
    &\ell(s, x) - \ell(s', x) = \integral{t_0}^T\E_{X_t | X_0 = x}\|s(X_t, t) - s'(X_t, t)\|^2 \dd t \\
    &\quad + 2\integral{t_0}^T \E_{X_t | X_0 = x} (s(X_t, t) - s'(X_t, t))^\top \left(s'(X_t, t) + \frac{X_t - m_t x}{\sigma_t^2}\right) \dd t.
\end{align*}
Using Young's inequality and the Cauchy-Schwarz bound, we deduce that
\begin{align}
    \label{eq:loss_sq_berns}
    &
    (\ell(s, x) - \ell(s', x))^2
    \leq 2 \left(\integral{t_0}^T \E_{X_t | X_0 = x}\|s(X_t, t) - s'(X_t, t)\|^2 \dd t \right)^2 \\
    \notag
    &
    + 8 \left(\integral{t_0}^T \E_{X_t | X_0 = x}\|s(X_t, t) - s'(X_t, t)\|^2 \dd t\right) \cdot \left( \integral{t_0}^T\E_{X_t | X_0 = x} \left\|s'(X_t, t) + \frac{X_t - m_t x}{\sigma_t^2}\right\|^2 \dd t \right).
\end{align}
Next, according to the definition of $s$ and $s'$ (see \eqref{eq:score_function_candidate} and \eqref{eq:score_function_candidate_prime}) it holds that
\begin{align*}
    \|s(X_t, t) - s'(X_t, t)\|^2
    &\leq \frac{2\|X_t\|^2}{\sigma_t^4} + \frac{4 m_t^2 \|h(X_t, t)\|^2}{(m_t^2\sigma^2 + \sigma_t^2)^2} + \frac{4 m_t^2 \|h'(X_t, t)\|^2}{(m_t^2(\sigma')^2 + \sigma_t^2)^2} \\
    &\leq \frac{4 m_t^2\|x\|^2}{\sigma_t^4} + \frac{4 \|X_t - m_t x\|^2}{\sigma_t^4} + \frac{32 m_t^2}{\sigma_t^4}.
\end{align*}
In the last inequality, we used the fact that both $\|h\|_{L^\infty(\R^D \times [t_0, T])}$ and $\|h'\|_{L^\infty(\R^D \times [t_0, T])}$ do not exceed $2$. Since the conditional distribution of $X_t$ given $X_0 = x$ is Gaussian $\cN(m_t x, \sigma_t^2 I_D)$, we obtain that
\[
    \E \|X_t - m_t x\|^2 = D \sigma_t^2,
\]
and then
\begin{align*}
    \integral{t_0}^T \E_{X_t | X_0 = x}\|s(X_t, t) - s'(X_t, t)\|^2 \dd t
    \leq \integral{t_0}^T \left(\frac{4(\|x\|^2 + 8)m_t^2}{\sigma_t^4} + \frac{4 D}{\sigma_t^2}\right) \dd t.
\end{align*}
One can simplify the expression in the right-hand side evaluating the integrals
\[
    \integral{t_0}^T \frac{2 m_t^2}{\sigma_t^4} \dd t
    \quad \text{and} \quad
    \integral{\sigma_{t_0}^2}^{\sigma_T^2} \frac{\dd u}{u(1 - u)}.
\]
Indeed, substituting $\sigma_t^2 = 1 - e^{-2t}$ with $u$, it is straightforward to check that
\begin{align}
    \label{eq:ints_aux_berns}
    \integral{t_0}^T \frac{2 m_t^2}{\sigma_t^4} \dd t = \integral{\sigma_{t_0}^2}^{\sigma_T^2} \frac{\dd u}{u^2} \leq \frac{1}{\sigma_{t_0}^2}
    \quad \text{and} \quad
    \integral{t_0}^T \frac{2}{\sigma_t^2} \dd t = \integral{\sigma_{t_0}^2}^{\sigma_T^2} \frac{\dd u}{u(1 - u)}
    \leq \log(1 / \sigma_{t_0}^2) + 2(T - t_0).
\end{align}
With these bounds at hand, we conclude that
\begin{align*}
    \integral{t_0}^T \E_{X_t | X_0 = x}\|s(X_t, t) - s'(X_t, t)\|^2 \dd t
    \leq \frac{2(\|x\|^2 + 8)}{\sigma_{t_0}^2} + 2D(\log(1 / \sigma_{t_0}^2) + 2(T - t_0)) .
\end{align*}
Hence, we verified the second statement of the Lemma as a byproduct. Nevertheless, the inequality \eqref{eq:scores_diff_exp_berns} will play an important role in future derivations.

It remains to bound
\begin{align*}
    &
    \integral{t_0}^T\E_{X_t | X_0 = x} \left\|s'(X_t, t) + \frac{X_t - m_t x}{\sigma_t^2}\right\|^2 \dd t
    \\&
    = \integral{t_0}^T \E_{X_t | X_0 = x}\left\|-\frac{X_t}{m_t^2(\sigma')^2 + \sigma_t^2} + \frac{m_t h'(X_t, t)}{m_t^2 (\sigma')^2 + \sigma_t^2} + \frac{X_t - m_t x}{\sigma_t^2}  \right\| \dd t
\end{align*}
to finish the proof. Let us introduce $Z = (X_t - m_t x) / \sigma_t$. Note that, conditionally on $X_0 = x$, the random vector $Z$ has a standard Gaussian distribution $\cN(0, I_D)$. This allows us to deduce that
\begin{align*}
    &\integral{t_0}^T\E_{X_t | X_0 = x} \left\|s'(X_t, t) + \frac{X_t - m_t x}{\sigma_t^2}\right\|^2 \dd t \\
    &
    \leq \quad 2\integral{t_0}^T \E_{X_t | X_0 = x} \frac{m_t^2 \|x - h'(X_t, t)\|^2}{\sigma_t^4} \dd t
    + 2\integral{t_0}^T \left(\frac{m_t^2 (\sigma')^2}{m_t^2 (\sigma')^2 + \sigma_t^2}\right)^2 \frac{\sigma_t^2 \E \|Z\|^2}{\sigma_t^4} \dd t \\
    &
    \leq 4 (\|x\|^2 + 4)\integral{t_0}^T \frac{m_t^2}{\sigma_t^4} \dd t + 2D \integral{t_0}^T \frac{\dd t}{\sigma_t^2} .
\end{align*}
The bound \eqref{eq:ints_aux_berns} yields that
\begin{align*}
    \integral{t_0}^T\E_{X_t | X_0 = x} \left\|s'(X_t, t) + \frac{X_t - m_t x}{\sigma_t^2}\right\|^2 \dd t 
    \leq \frac{2(\|x\|^2 + 4)}{\sigma_{t_0}^2} + D \log(1 / \sigma_{t_0}^2) + 2D (T - t_0).
\end{align*}
The last inequality, combined with \eqref{eq:loss_sq_berns}, \eqref{eq:scores_diff_exp_berns}, immediately implies that
\begin{align*}
    (\ell(s, x) - \ell(s', x))^2
    &\leq 48 \left( \frac{\|x\|^2 + 1}{\sigma_{t_0}^2} + D \log(1 / \sigma_{t_0}^2) + D (T - t_0) \right)
    \\&\quad
    \cdot \left(\integral{t_0}^T \E_{X_t | X_0 = x}\|s(X_t, t) - s'(X_t, t)\|^2 \dd t\right) .
\end{align*}
The proof is finished.

\endproof

\subsection{Proof of Lemma \ref{lem:loss_orlicz_norm_bound}}
\label{sec:lem_loss_orlicz_norm_bound_proof}

The proof follows standard techniques. However, before we move to the upper bound on the Orlicz norm of $\ell(s, X_0)$, we must first elaborate on properties of the loss function $\ell$. For this reason, we split the proof into two steps.

\medskip

\noindent
\textbf{Step 1: upper bound on $\ell(s, X_0)$.}
\quad
The goal of this step is to show that $\ell(s, X_0)$ grows as fast as $\cO(\|X_0\|^2)$ and specify the hidden constant. Let us fix an arbitrary $x \in \R^d$ and recall that
\[
    \ell(s, x)
    = \int\limits_{t_0}^T \left( \; \int\limits_{\R^D} \left\| s(y, t) - \nabla_y \log \sfp_t(y \,\vert\, x) \right\|^2 \sfp_t(y \,|\, x) \dd y \right) \dd t,
\]
where
\begin{align}
    \label{eq:cond_score_mt_sigmat_redef}
    \nabla_y \log \sfp_t(y \,\vert\, x) = -\frac{y - m_t x}{\sigma_t^2},
    \quad
    m_t = e^{-t},
    \quad \text{and} \quad
    \sigma_t^2 = 1 - e^{-2t}.
\end{align}
Since $\sfp_t(y \mid x)$ is the density of the Gaussian distribution $\cN(m_t x, \sigma_t^2)$, we have
\[
    \ell(s, x) = \int\limits_{t_0}^T \E_{Y \sim \mathcal N(m_t x, \sigma_t^2 I_D)} \left\| s(Y, t) + \frac{Y - m_t x}{\sigma_t^2} \right\|^2 \dd t.
\]
Due to the conditions of the lemma, the score function $s(y, t)$ has a form
\[
    s(y, t) = -\frac{y}{m_t^2\sigma^2 + \sigma_t^2} + \frac{m_t \, h(y, t)}{m_t^2\sigma^2 + \sigma_t^2},
\]
where $\|h(y, t)\| \leq 2$ and $\sigma \in [0, 1)$.
Using the Cauchy-Schwarz inequality, we obtain that
\begin{align*}
    \left\| s(Y, t) + \frac{Y - m_t x}{\sigma_t^2} \right\|^2
    &
    = \left\| -\left(\frac1{m_t^2\sigma^2 + \sigma_t^2} - \frac1{\sigma_t^2} \right) (Y - m_t x) + \frac{m_t x}{m_t^2\sigma^2 + \sigma_t^2} + \frac{m_t h(Y, t)}{m_t^2\sigma^2 + \sigma_t^2} \right\|^2
    \\&
    \leq 3 \left(\frac1{m_t^2\sigma^2 + \sigma_t^2} - \frac1{\sigma_t^2} \right)^2 \|Y - m_t x\|^2
    + \frac{3 m_t^2 (\|x\|^2 + \|h(Y, t)\|^2)}{(m_t^2\sigma^2 + \sigma_t^2)^2},
\end{align*}
and then
\begin{align*}
    \ell(s, x)
    &
    \leq \int\limits_{t_0}^T \E_{Y \sim \mathcal N(m_t x, \sigma_t^2 I_D)} \left( \frac{3 m_t^2 \|h(Y, t)\|^2}{(m_t^2\sigma^2 + \sigma_t^2)^2} + \frac{3 m_t^2 \|x\|^2}{(m_t^2\sigma^2 + \sigma_t^2)^2} \right) \dd t
    \\&\quad
    + 3 \int\limits_{t_0}^T \E_{Y \sim \mathcal N(m_t x, \sigma_t^2 I_D)} \left(\frac1{m_t^2\sigma^2 + \sigma_t^2} - \frac1{\sigma_t^2} \right)^2 \|Y - m_t x\|^2 \, \dd t
    \\&
    \leq 3 \int\limits_{t_0}^T \E_{Y \sim \mathcal N(m_t x, \sigma_t^2 I_D)} \left( \frac{4 m_t^2}{m_t^2\sigma^2 + \sigma_t^2} + \frac{m_t^2 \|x\|^2}{(m_t^2\sigma^2 + \sigma_t^2)^2} + \frac{m_t^4 \sigma^4 \, \|Y - m_t x\|^2}{\sigma_t^4 (m_t^2\sigma^2 + \sigma_t^2)^2} \right) \dd t
    \\&
    = 12 \int\limits_{t_0}^T \frac{m_t^2}{m_t^2\sigma^2 + \sigma_t^2} \dd t + 3 \|x\|^2 \int\limits_{t_0}^T \frac{m_t^2 }{(m_t^2\sigma^2 + \sigma_t^2)^2} \dd t + 3 D \int\limits_{t_0}^T \frac{m_t^4 \sigma^4}{\sigma_t^2 (m_t^2\sigma^2 + \sigma_t^2)^2} \dd t.
\end{align*}
The expression in the right-hand side can be simplified if we note that
\[
    \frac{m_t^2}{m_t^2\sigma^2 + \sigma_t^2} \leq \frac{m_t^2}{\sigma_t^2},
    \quad
    \frac{m_t^2}{(m_t^2\sigma^2 + \sigma_t^2)^2}
    \leq \frac{m_t^2}{\sigma_t^4},
    \quad \text{and} \quad
    \frac{m_t^4 \sigma^4}{\sigma_t^2 (m_t^2\sigma^2 + \sigma_t^2)^2}
    \leq \frac{1}{\sigma_t^2}.
\]
Indeed, it holds that
\begin{align*}
    \ell(s, x)
    &
    \leq 12 \int\limits_{t_0}^T \frac{m_t^2}{m_t^2\sigma^2 + \sigma_t^2} \dd t + 3 \|x\|^2 \int\limits_{t_0}^T \frac{m_t^2 }{(m_t^2\sigma^2 + \sigma_t^2)^2} \dd t + 3 D \int\limits_{t_0}^T \frac{m_t^4 \sigma^4}{\sigma_t^2 (m_t^2\sigma^2 + \sigma_t^2)^2} \dd t
    \\&
    \leq 12 \int\limits_{t_0}^T \frac{m_t^2}{\sigma_t^2} \dd t + 3 \|x\|^2 \int\limits_{t_0}^T \frac{m_t^2}{\sigma_t^4} \dd t + 3 D \int\limits_{t_0}^T \frac{\dd t}{\sigma_t^2}
    \\&
    = 12 \int\limits_{t_0}^T \frac{e^{-2t} \dd t}{1 - e^{-2t}} + 3 \|x\|^2 \int\limits_{t_0}^T \frac{e^{-2t} \dd t}{(1 - e^{-2t})^2} + 3 D \int\limits_{t_0}^T \frac{\dd t}{1 - e^{-2t}}.
\end{align*}
Let us elaborate on the integrals in the right-hand side. Substituting $(1 - e^{-2t})$ with $u$, we observe that
\begin{align*}
    \int\limits_{t_0}^T \frac{e^{-2t} \dd t}{1 - e^{-2t}}
    &
    = \int\limits_{1 - e^{-2 t_0}}^{1 - e^{-2 T}} \frac{\dd u}{2u}
    = \frac12 \left( \log(1 - e^{-2 T}) - \log(1 - e^{-2 t_0}) \right) 
    \leq -\frac12 \log(1 - e^{-2 t_0})
    \\
    \int\limits_{t_0}^T \frac{e^{-2t} \dd t}{(1 - e^{-2t})^2}
    &
    = \int\limits_{1 - e^{-2 t_0}}^{1 - e^{-2 T}} \frac{\dd u}{2 u^2} 
    = \frac12 \left( \frac{1}{1 - e^{-2 t_0}} - \frac{1}{1 - e^{-2 T}} \right)
    \leq \frac{1}{2(1 - e^{-2 t_0})}
\end{align*}
and
\begin{align*}
    \int\limits_{t_0}^T \frac{\dd t}{1 - e^{-2t}}
    &
    = \int\limits_{t_0}^T \frac{e^{-2t} \dd t}{e^{-2t} (1 - e^{-2t})}
    = \int\limits_{1 - e^{-2 t_0}}^{1 - e^{-2 T}} \frac{\dd u}{2 u (1 - u)}
    = \int\limits_{1 - e^{-2 t_0}}^{1 - e^{-2 T}} \frac{\dd u}{2 u } + \int\limits_{1 - e^{-2 t_0}}^{1 - e^{-2 T}} \frac{\dd u}{2(1 - u)}
    \\&
    \leq -\frac12 \log(1 - e^{-2 t_0}) + (T - t_0). 
\end{align*}
Hence, we obtain that
\begin{align*}
    \ell(s, x)
    &
    \leq 12 \int\limits_{t_0}^T \frac{e^{-2t} \dd t}{1 - e^{-2t}} + 3 \|x\|^2 \int\limits_{t_0}^T \frac{e^{-2t} \dd t}{(1 - e^{-2t})^2} + 3 D \int\limits_{t_0}^T \frac{\dd t}{1 - e^{-2t}}
    \\&
    \leq -6 \log(1 - e^{-2 t_0}) + \frac{3 \|x\|^2 \, e^{-2 t_0}}{2(1 - e^{-2 t_0})} - \frac{3 D}2 \log(1 - e^{-2 t_0}) + 3D (T - t_0)
    \\&
    = -\frac32 (4 + D) \log(1 - e^{-2 t_0}) + \frac{3 \|x\|^2 \, e^{-2 t_0}}{2(1 - e^{-2 t_0})} + 3D (T - t_0).
\end{align*}
In view of the definition of $\sigma_t^2$ (see \eqref{eq:cond_score_mt_sigmat_redef}), we conclude that
\begin{align}
    \label{eq:loss_quadr_bound}
    \ell(s, x) \leq \frac{3}{2}(4 + D)\log(\sigma_{t_0}^{-2}) + \frac{3\|x\|^2}{2 \sigma_{t_0}^2} + 3D(T - t_0),
    \quad \text{for all } x \in \R^D.
\end{align}

\medskip

\noindent
\textbf{Step 2: upper bound on the Orlicz norm}
\quad
The current step aims to provide the upper bound for the desired Orlicz norm through the analysis of its exponential moment.
The bound \eqref{eq:loss_quadr_bound} yields that for any $\lambda > 0$ the exponential moment of $\lambda \ell(s, X_0)$ does not exceed
\begin{align*}
    \E_{X_0}\exp\{\lambda \ell(s, X_0)\}
    \leq \exp \left\{\frac{3\lambda}{2}(4 + D)\log(\sigma_{t_0}^{-2}) + 3\lambda D(T - t_0)\right\} \E_{X_0}\exp\left\{\frac{3\lambda \|X_0\|^2}{2\sigma_{t_0}^2}\right\}.
\end{align*}
According to Assumption \ref{as:data_distr}, we have that $X_0 = g^*(U) + \sdata Z$, where $U \sim \mathrm{Un}([0, 1]^d)$ and $Z \sim \cN(0, I_D)$ are independent.
Moreover, it holds that $\|g^*\|_{L^\infty([0, 1]^d)} \leq 1$.
Therefore, we obtain that
\begin{align*}
    \E_{X_0}\exp\left\{\frac{3\lambda \|X_0\|^2}{2\sigma_{t_0}^2}\right\}
    \leq \exp\left\{\frac{3\lambda}{\sigma_{t_0}^2}\right\}\E_Z \exp\left\{\frac{3\lambda \sdata^2 \|Z\|^2}{\sigma_{t_0}^2}\right\} .
\end{align*}
Since $\|Z\|^2 \sim \chi^2(D)$, it follows from Remark \ref{rem:chi-squared_concentration} that $\|Z\|^2$ is sub-exponential random variable with parameters $(2\sqrt{D}, 4)$.
Therefore, using \citep[Proposition 2.7.1]{vershynin2018high}, we conclude that there exists an absolute constant $C_\psi$ such that
\begin{align*}
    \E_Z \exp\left\{\frac{3\lambda \sdata^2 \|Z\|^2}{\sigma_{t_0}^2}\right\}
    \leq \exp\left\{\frac{\lambda C_\psi D \sdata^2}{\sigma_{t_0}^2}\right\}
\end{align*}
for sufficiently small $\lambda$.
Hence, we obtain that
\begin{align*}
    \log \E_{X_0}\exp\{\lambda \ell(s, X_0)\}
    \leq \frac{3\lambda}{2}(4 + D)\log(\sigma_{t_0}^{-2}) + 3\lambda D(T - t_0) + \frac{3\lambda}{\sigma_{t_0}^2} + \frac{\lambda C_\psi D \sdata^2}{\sigma_{t_0}^2} .
\end{align*}
Taking into account that $\ell(s, X_0)$ is non-negative almost surely and using \citep[Proposition 2.7.1]{vershynin2018high}, we conclude that
\begin{align*}
    \|\ell(s, X_0)\|_{\psi_1} \lesssim D \log(\sigma_{t_0}^{-2}) + D(T - t_0) + \frac{D\sdata^2 + 1}{\sigma_{t_0}^2},
\end{align*}
where the hidden constant does not depend on $s$ and the problem parameters.
The proof is complete.

\endproof

\subsection{Proof of Lemma \ref{lem:cov_number_unbounded_eval}}
\label{sec:lem_cov_number_unbounded_eval_proof}
To improve clarity, we divided the proof into several steps.

\noindent
\textbf{Step 1: proximity of loss functions evaluation.}\quad
First, according to Lemma \ref{lem:losses_diff_bound}, for any $s_1, s_2 \in \cS(L, W, S, B)$ of the form
\begin{align*}
    s_j(y, t) = -\frac{y}{m_t^2 \sigma_{(j)}^2 + \sigma_t^2} + \frac{m_t \clip{f_j(y, t)}{2}}{m_t^2\sigma_{(j)}^2 + \sigma_t^2}, \quad (y, t) \in \R^D \times [t_0, T], \; j \in \{1, 2\}
\end{align*}
and for any $x \in \R^D$, it holds that
\begin{align}
    \label{eq:loss_diff_sq_s2_cover}
    \notag
    &(\ell(s_1, x) - \ell(s_2, x))^2 \\
    &
    \leq \frac{48(1 + \|x\|^2)D(T - t_0)\log(\sigma_{t_0}^{-2})}{\sigma_{t_0}^2}\integral{t_0}^T\E_{X_t | X_0 = x}\|s_1(X_t, t) - s_2(X_t, t)\|^2 \dd t .
\end{align}
Our goal now is to bound the last term in the above inequality.
The triangle inequality suggests that
\begin{align*}
    &\E_{X_t |X_0= x}\|s_1(X_t, t) - s_2(X_t, t)\|^2
    \lesssim \E_{X_t|X_0=x} \left\| \frac{m_t \clip{f_1(X_t, t)}{2}}{m_t^2\sigma_1^2 + \sigma_t^2} - \frac{m_t \clip{f_2(X_t, t)}{2}}{m_t^2\sigma_2^2 + \sigma_t^2} \right\|^2 \\
    &\quad + \E_{X_t |X_0=x}[\|X_t\|^2]\left\{(m_t^2\sigma_1^2 + \sigma_t^2)^{-1} - (m_t^2\sigma_2^2 + \sigma_t^2)^{-1}\right\}^2 .
\end{align*}
Mean value theorem in conjunction with \eqref{eq:cond_distr} implies that
\begin{align}
    \label{eq:score_diff_X_t_X_0_x_cover}
    \notag
    &\E_{X_t |X_0= x}\|s_1(X_t, t) - s_2(X_t, t)\|^2
    \lesssim \frac{D(\|x\|^2 + 1) |\sigma_1 - \sigma_2|^2 m_t^4 }{\sigma_t^8}
    \\&\quad
    + \frac{m_t^2}{\sigma_t^4} \E_{X_t |X_0= x}\|\clip{f_1(X_t, t)}{2} - \clip{f_2(X_t, t)}{2}\|^2 .
\end{align}
Applying the union bound, we obtain that for any $R \geq 1$, the following inequality holds:
\begin{align*}
    \p(\|X_t - m_t x\| \geq R \, | \, X_0 = x) \leq \p_{Z \sim \cN(0, I_D)}(\sigma_t\|Z\| \geq R)
    \leq D\exp\left\{-\frac{R^2}{2D\sigma_t^2}\right\}.
\end{align*}
Thus, choosing $R = \sqrt{ 2D\log(D / \tau') }$ ensures that
\begin{align}
    \label{eq:clip_f_diff_X_t_X_0_cover}
    \E_{X_t |X_0= x}\|\clip{f_1(X_t, t)}{2} - \clip{f_2(X_t, t)}{2}\|^2
    \lesssim 1 \wedge (\|f_1 - f_2\|^2_{L^{\infty}(\cB(x, R))} + \tau'),
\end{align}
where $\tau' \in (0, 1)$ will be determined later in the proof.
Similarly, for any $R_\text{data} \geq 1$, the union bound together with Assumption \ref{as:data_distr} implies that
\begin{align*}
    \p(\|X_0\| \geq R_\text{data}) \leq \p(\sdata \|Z\| > R_\text{data} - 1)
    \leq D \exp\left\{-\frac{(R_\text{data} - 1)^2}{2D}\right\}
\end{align*}
Thus, setting $R_\text{data} = 1 + \sqrt{2D\log(D n / \delta \vee (1 / \tau'))}$ guarantees that $\p(\|X_0\| > R_\text{data}) \leq \delta / n \wedge \tau'$.
Assume that
\begin{align}
    \label{eq:sigma_f_acc_assn_cover}
    |\sigma_1 - \sigma_2| \vee \|f_1 - f_2\|_{L^\infty(\cB(0, R + R_\text{data}))} \leq \tau'.
\end{align}
Therefore, from \eqref{eq:score_diff_X_t_X_0_x_cover}, \eqref{eq:clip_f_diff_X_t_X_0_cover}, and \eqref{eq:sigma_f_acc_assn_cover}, we deduce that
\begin{align*}
    &\E_{X_t |X_0 = x}\|s_1(X_t, t) - s_2(X_t, t)\|^2 \\
    &\quad \lesssim \frac{D^2 m_t^2 \log(Dn / \delta \vee (1 / \tau'))}{\sigma_t^8} \left(\tau' \cdot \1(\|x\| \leq R_\text{data}) + (1 + \|x\|^2) \1(\|x\| > R_\text{data}) \right) .
\end{align*}
Substituting the derived bound into \eqref{eq:loss_diff_sq_s2_cover} and usning the observation that
\begin{align*}
    \integral{t_0}^T \frac{m_t^2 \, \dd t}{\sigma_t^8} \leq \frac{(T - t_0)}{\sigma_{t_0}^8},
\end{align*}
we conclude that
\begin{align}
    \label{eq:loss_diff_both_cases_cover}
    &|\ell(s_1, x) - \ell(s_2, x)| \\
    \notag
    &\quad \lesssim \frac{D^2(T - t_0)(1 + \|x\|^2) \sqrt{\log(Dn / \delta \vee (1 / \tau'))}}{\sigma_{t_0}^5}\left(\sqrt{\tau'} \cdot \1(\|x\| \leq R_\text{data}) + \1(\|x\| > R_\text{data}) \right) .
\end{align}
As a special case of the above bound, we have that
\begin{align}
    \label{eq:ell_diff_comp_cover}
    \|\ell(s_1, \cdot) - \ell(s_2, \cdot)\|_{L^\infty(\cB(0, R_\text{data}))}
    \lesssim \frac{D^3(T - t_0)\log(Dn / \delta \vee (1 / \tau'))\sqrt{\tau'}}{\sigma_{t_0}^5} .
\end{align}
In addition, from Hölder's inequality we obtain that
\begin{align*}
    &|\E_{X_0}[\ell(s_1, X_0) - \ell(s_2, X_0)]| \\
    &\quad \leq \|\ell(s_1, \cdot) - \ell(s_2, \cdot)\|_{L^\infty(\cB(0, R_\text{data}))} + \sqrt{\p(\|X_0\| \geq R_\text{data})} \cdot \sqrt{\E_{X_0}(\ell(s_1, X_0) - \ell(s_2, X_0))^2} .
\end{align*}
The combination of \eqref{eq:loss_diff_both_cases_cover}, \eqref{eq:ell_diff_comp_cover} and the observation that $\E_{X_0}[\|X_0\|^4] \lesssim D$ due to Assumption \ref{as:data_distr} implies that the above bound simplifies to
\begin{align*}
    |\E_{X_0}[\ell(s_1, X_0) - \ell(s_2, X_0)]|
    \lesssim \frac{D^3 (T - t_0)\log(Dn / \delta \vee (1 / \tau'))\sqrt{\tau'}}{\sigma_{t_0}^5}.
\end{align*}
Thus, choosing
\begin{align}
    \label{eq:tau_prime_def_cover}
    \tau' \asymp \left(\frac{\sigma_{t_0}^5 \tau}{D^3 (T - t_0)\log(Dn / \delta)}\right)^4 \in (0, 1)
\end{align}
ensures that
\begin{align}
    \label{eq:loss_diff_s1_s2_comp_pop_cover}
    \|\ell(s_1, \cdot) - \ell(s_2, \cdot)\|_{L^\infty(\cB(0, R_\text{data}))} \vee |\E_{X_0}[\ell(s_1, X_0) - \ell(s_2, X_0)]| \leq \tau / 2.
\end{align}

\noindent
\textbf{Step 2: covering number evaluation.}\quad
The following result elaborates on the covering number of the $\relu$ neural network class.
\begin{Lem}[\cite{suzuki2018adaptivity}, Lemma 3]
\label{lem:nn_cov_number_suzuki}
    For any $\tau > 0$ the covering number of $\NN(L, W, S, B)$ can be bounded by
    \begin{align*}
        \log \cN(\tau, \NN(L, W, S, B), \|\cdot\|_{L^\infty([0, 1]^D)})
        \lesssim SL\log\left(\tau^{-1}L(\|W\|_\infty + 1)(B \vee 1)\right).
    \end{align*}
\end{Lem}
By leveraging Lemma \ref{lem:nn_cov_number_suzuki}, we can infer that multiplying the weight matrix of the initial layer by $(K \vee T \vee 1)$, and dividing the input vector by the same value, leads to
\begin{align}
    \label{eq:cov_number_tau_prime_cover}
    \notag
    &\log \cN(\tau', \NN(L, W, S, B), \|\cdot\|_{L^\infty([-R - R_\text{data}, R + R_\text{data}]^D \times [t_0, T])}) \\
    &\quad \lesssim SL\log((1 / \tau')L(\|W\|_\infty + 1)(R + R_\text{data} + T)(B \vee 1)),
\end{align}
for any $\tau' \in (0, 1)$.
Let
\begin{align*}
    \cN_{\tau'} = \cN(\tau', \NN(L, W, S, B), \|\cdot\|_{L^\infty([-R - R_\text{data}, R + R_\text{data}]^D \times [t_0, T])})
\end{align*}
and let $\cF_{\tau'} = \{f_j : 1 \leq j \leq \cN_{\tau'}\}$ be the minimal $\tau'$-net of $\NN(L, W, S, B)$ with respect to $L^\infty$-norm on $[-R - R_\text{data}, R + R_\text{data}]^D \times [t_0, T]$.
Let also $\cH_{\tau'}$ be the minimal $\tau'$-net of $[0, 1]$ with respect to $\|\cdot\|_\infty$-norm.
We also note from the union bound and the choice of $R_\text{data}$ that
\begin{align*}
    \p(\|X_i\| \leq R_\text{data} \text{ for all } 1 \leq i \leq n)
    \geq 1 - \delta .
\end{align*}
Therefore, from \eqref{eq:sigma_f_acc_assn_cover} and \eqref{eq:loss_diff_s1_s2_comp_pop_cover} we deduce that
\begin{align*}
    \cS_\tau = \left\{s(y, t) = -\frac{y}{m_t^2\sigma^2 + \sigma_t^2} + \frac{m_t}{m_t^2 \sigma^2 + \sigma_t^2} \clip{f(y, t)}{2} : f \in \cF_{\tau'}, \; \sigma \in \cH_{\tau'} \right\}
\end{align*}
satisfies the the statement of the Lemma.
Specifically, there exists an event $\cE$ with $\p(\cE) \geq 1-\delta$, for which the following holds: for all $s \in \cS(L, W, S, B)$, there exists $s_\tau \in \cS_\tau$ satisfying
\begin{align*}
    |\E_{X_0}[\ell(s, X_0) - \ell(s_\tau, X_0)]| + |\widehat{\E}_{X_0}[\ell(s, X_0) - \ell(s_\tau, X_0)]| \leq \tau .
\end{align*}
Moreover, for $\cS_\tau$ we have from \eqref{eq:tau_prime_def_cover} and \eqref{eq:cov_number_tau_prime_cover} that
\begin{align*}
    \log |\cS_\tau| &\leq \log \cN(\tau', \NN(L, W, S, B), \|\cdot\|_{L^\infty([-R - R_\text{data}, R + R_\text{data}]\times [t_0, T])}) + \log(\tau', [0, 1], \|\cdot\|_\infty) \\
    &\lesssim SL \log(\tau^{-1} L(\|W\|_\infty + 1)(B \vee 1) D T \sigma_{t_0}^{-2}\log(n / \delta)) .
\end{align*}
The proof is finished.

\endproof

\section{Proof of Theorem \ref{lem:tv_bound}}
\label{sec:tv_bound_proof}
We first formulate a helper result which connects the total variation distance and $L^2$ score estimation error.

\begin{Lem}[\cite{azangulov2024convergence}, Theorem 2; \cite{chen2023sampling}, Appendix B.2]
\label{lem:tv_score_bound}
    If $X_0$ has a finite second moment, then
    \begin{align*}
       \mathrm{D}_\mathrm{KL}(\hat{Z}_{T - t_0} \| X_{t_0}) \lesssim D e^{-2T} + \integral{t_0}^T\E_{X_t} \|\hat{s}(X_t, t) - s(X_t, t)\|^2 \dd t.
    \end{align*}
\end{Lem}
Therefore, the combination of the triangle inequality and Pinsker's inequality yields
\begin{align*}
    \mathrm{TV}(\hat{Z}_{T - t_0}, X_0)
    \leq \mathrm{TV}(\hat{Z}_{T - t_0}, X_{t_0}) + \mathrm{TV}(X_{t_0}, X_0)
    \lesssim \sqrt{\mathrm{D}_\mathrm{KL}(\hat{Z}_{T - t_0}\| X_{t_0})} + \mathrm{TV}(X_{t_0}, X_0).
\end{align*}
Applying Lemma \ref{lem:tv_score_bound}, we obtain 
\begin{align}
    \label{eq:tv_score_xt0_decomp}
    \mathrm{TV}(\hat{Z}_{T - t_0}, X_0)
    \lesssim \sqrt{D}e^{-T} + \left\{\integral{t_0}^T\E_{X_t}\|\hat{s}(X_t, t) - s^*(X_t, t)\|^2 \dd t \right\}^{1/2}
    + \mathrm{TV}(X_{t_0}, X_0).
\end{align}
Next, we evaluate the last term of the above bound using Jensen's inequality and Assumption \ref{as:data_distr}.
This implies that
\begin{align*}
    \mathrm{TV}(X_{t_0}, X_0)
    &= \mathrm{TV}(m_t X_0 + \sigma_t Z, X_0) \\
    &\leq \E_{U\sim \mathrm{Un}([0, 1]^d)}\left[\mathrm{TV}(m_{t_0} g^*(U) + \tilde\sigma_{t_0}Z, g^*(U) + \sdata Z)\right],
\end{align*}
where $Z \sim \cN(0, I_D)$.
From Pinsker's inequality we obtain that
\begin{align*}
    \mathrm{TV}(X_{t_0}, X_0) \lesssim \E_{U\sim \mathrm{Un}([0, 1]^d)}
    \left[\left\{\mathrm{D}_\text{KL}\left(\cN(m_{t_0} g^*(U), \tilde\sigma_{t_0}^2I_D) \, \| \, \cN(g^*(U), \sdata^2 I_D)\right) \right\}^{1/2}\right].
\end{align*}
Now we substitute the closed-form expression for the KL divergence between Gaussians, which yields
\begin{align*}
     \mathrm{TV}(X_{t_0}, X_0) \lesssim
     \E_{U\sim \mathrm{Un}([0, 1]^d)}
     \left[\left\{ D\log\left(\frac{\sdata^2}{\tilde\sigma_{t_0}^2}\right) - D + \frac{(1 - m_{t_0})^2\|g^*(U)\|^2}{\sdata^2} + \frac{D \tilde\sigma_{t_0}^2}{\sdata^2} \right\}^{1/2} \right].
\end{align*}
Recall that, according to Assumption \ref{as:data_distr}, we have $\|g^*(U)\| \leq 1$ and $\sdata < 1$. In particular, this yields that $\sdata^2 \leq \tilde\sigma_{t_0}^2$.
Furthermore, we have $\tilde\sigma^2_{t_0} = m_{t_0}^2\sdata^2 + \sigma_{t_0}^2$, and for $t_0 \leq 1$, it holds that $\sigma_{t_0}^2 \asymp t_0$.
Combining these observations, we obtain that
\begin{align}\label{eq:tv_t0_0}
    \mathrm{TV}(X_{t_0}, X_0)
    \lesssim \left\{ -D +  \frac{t_0^2}{\sdata^2} + D\left(1 + \frac{t_0}{\sdata^2}\right) \right\}^{1/2}
    \lesssim \frac{\sqrt{D t_0}}{\sdata}.
\end{align}
Thus, by invoking Theorem \ref{thm:estim_enh_zero}, and substituting the bound \eqref{eq:tv_t0_0} into \eqref{eq:tv_score_xt0_decomp},
we deduce that, with probability at least $(1 - \delta)$, the following holds:
\begin{align*}
    &\mathrm{TV}(\hat{Z}_{T - t_0}, X_0) \\
    &
    \lesssim \sqrt{D}e^{-T}
    + \left\{\frac{D^{12 + 2\binom{d + \floor{\beta}}{d}} (T \vee 1)^2 L(t_0, n) }{\sigma_{t_0}^2} \right\}^{1/2}(n\sigma_{t_0}^2)^{-\frac{\beta}{2\beta + d}}\log^{1/2}(4 / \delta)
    + \frac{\sqrt{D t_0}}{\sdata}.
\end{align*}
Therefore, choosing $T \asymp \log n$ and $t_0 = \sdata^{\frac{2\beta + d}{3\beta + d}} n^{-\frac{\beta}{3\beta + d}} \leq 1$ ensures that
\begin{align*}
    \mathrm{TV}(\hat{Z}_{T - t_0}, X_0)
    \lesssim D^{6 + \binom{d + \floor{\beta}}{d}} \sdata^{-\frac{1}{3\beta + d}} n^{-\frac{\beta}{6\beta + 2d}} \log^{1/2}(4 / \delta) L(t_0, n) \log n
\end{align*}
with probability at least $(1 - \delta)$.
Finally, substituting the expression for $t_0$ into $L(t_0, n)$ yields the desired bound, thereby completing the proof.

\endproof

\section{Proof of Lemma \ref{lem:analytic}}
\label{sec:lem_analytic_proof}

The proof goes by the induction in $k$. Let us introduce
\[
    h(y) = (\sqrt{2\pi} \sigma)^D e^{-\|y\|^2 / (2 \sigma^2)} \sfp(y) = \int\limits_{[0, 1]^d} \exp\left\{ \frac{y^\top g(u)}{\sigma^2} - \frac{\|g(u)\|^2}{2 \sigma^2} \right\} \dd u
\]
and show that
\[
    \sup\limits_{y \in \R^D} \left\| \nabla^k \log h(y) \right\| \leq \frac{2^{k - 1} (k - 1)!}{\sigma^{2k}} \max\limits_{u \in [0, 1]^d} \|g(u)\|^{k+1}.
\]
The induction base is obvious. Indeed, it holds that
\begin{equation}
    \label{eq:log_h_grad}
    \nabla \log h(y) = \frac{1}{\sigma^2} \cdot \frac{\int_{[0, 1]^d} g(u_1) \exp\left\{ \frac{y^\top g(u_1)}{\sigma^2} - \frac{\|g(u_1)\|^2}{2\sigma^2} \right\} \dd u_1}{\int_{[0, 1]^d} \exp\left\{ \frac{y^\top g(u_1)}{\sigma^2} - \frac{\|g(u_1)\|^2}{2\sigma^2} \right\} \dd u_1},
\end{equation}
and then, due to the triangle inequality, it holds that
\[
    \sup\limits_{y \in \R^D} \left\| \nabla \log h(y) \right\|
    \leq \frac1{\sigma^2} \; \sup\limits_{y \in \R^D} \frac{\int_{[0, 1]^d} \|g(u_1)\| \exp\left\{ \frac{y^\top g(u_1)}{\sigma^2} - \frac{\|g(u_1)\|^2}{2\sigma^2} \right\} \dd u_1}{\int_{[0, 1]^d} \exp\left\{ \frac{y^\top g(u_1)}{\sigma^2} - \frac{\|g(u_1)\|^2}{2\sigma^2} \right\} \dd u_1}
    \leq \frac1{\sigma^2}.
\]
Assume that for some $k \in \mathbb N$ we have
\[
    \nabla^k \log h(y) = \frac1{\sigma^{2k}} \cdot \frac{\int_{[0, 1]^d} P_k(g(u_1), \dots, g(u_k)) \exp \left\{\sum\limits_{j = 1}^k \left( \frac{y^\top g(u_j)}{\sigma^2} - \frac{\|g(u_j)\|^2}{2 \sigma^2} \right) \right\} \dd u_1 \dots \dd u_k}{\int_{[0, 1]^d} \exp \left\{\sum\limits_{j = 1}^k \left( \frac{y^\top g(u_j)}{\sigma^2} - \frac{\|g(u_j)\|^2}{2 \sigma^2} \right) \right\} \dd u_1 \dots \dd u_k},
\]
where $P_k$ is a tensor of order $k$. Then it is straightforward to check that
\[
    \nabla^{k + 1} \log h(y) = \frac{\int_{[0, 1]^d} P_{k + 1}(g(u_1), \dots, g(u_{k + 1})) \exp \left\{\sum\limits_{j = 1}^{k + 1} \left( \frac{y^\top g(u_j)}{\sigma^2} - \frac{\|g(u_j)\|^2}{2 \sigma^2} \right) \right\} \dd u_1 \dots \dd u_{k + 1}}{\sigma^{2k + 2} \; \int_{[0, 1]^d} \exp \left\{\sum\limits_{j = 1}^{k + 1} \left( \frac{y^\top g(u_j)}{\sigma^2} - \frac{\|g(u_j)\|^2}{2 \sigma^2} \right) \right\} \dd u_1 \dots \dd u_{k + 1}},
\]
where
\begin{align*}
    P_{k+1}(g(u_1), \dots, g(u_{k + 1}))
    &
    = P_k(g(u_1), \dots, g(u_k)) \otimes (g(u_1) + \ldots + g(u_k))
    \\&\quad
    - k P_k(g(u_1), \dots, g(u_k)) \otimes g(u_{k + 1}).
\end{align*}
Due to the triangle inequality, it holds that
\[
    \| P_{k+1} \|
    \leq \|P_k\| \left\| \sum\limits_{j = 1}^k g(u_j) \right\|
    + k \|P_k\| \|g(u_{k + 1})\|
    \leq 2k \|P_k\| \max\limits_{u \in [0, 1]^d} \|g(u)\|.
\]
Taking into account the relation $P_1(g(x_1)) = g(x_1)$ following from \eqref{eq:log_h_grad}, we obtain that
\[
    \| P_{k+1} \| \leq 2^k k! \max\limits_{u \in [0, 1]^d} \|g(u)\|^{k+1}.
\]
Hence, we conclude that
\[
    \left\| \nabla^{k+1} \log h(y) \right\|
    \leq \frac{\| P_{k+1} \|}{\sigma^{2k + 2}}
    \leq \frac{2^k k!}{\sigma^{2k + 2}} \; \max\limits_{u \in [0, 1]^d} \|g(u)\|^{k+1}.
\]
\endproof

\section{Approximation properties of deep neural networks}
\label{sec:approximation_auxiliary_results}

This section collects useful results on approximation properties of deep neural networks. Lemmata \ref{lem:multi_oko} and \ref{lem:exp_minus_oko} concern approximation of basic functions. They are used as auxiliary results in the proof of Theorem \ref{th:score_approximation}. 

\begin{Lem}[\cite{oko2023diffusion}, Lemma F.6]
\label{lem:multi_oko}
    Let $d \geq 2$, $C \geq 1$ and $\eps' \in (0, 1]$. For any $\eps > 0$ there exists a ReLU-network $\varphi(x_1, \dots, x_d)$ with $L \lesssim \log d(\log\eps^{-1} + d\log C)$, $\|W\|_\infty = 48d$,  $S \lesssim d\log\eps^{-1} + d\log C$, and $B = C^d$ such that
    \begin{align}
        \left|\varphi(x_1', \dots, x_d') - \prod_{k=1}^d x_k\right| \leq \varepsilon + dC^{d-1}\eps',
    \end{align}
    for all $x \in [-C, C]^d$ and $x' \in \R^d$ such that $\|x - x'\|_\infty \leq \eps'$. Moreover, $\varphi(x_1', \dots, x_d') = 0$ if at least one $x_i' = 0$ and $|\varphi(x_1', \dots, x_d')| \leq C^d$.
    Furthermore, the proposition extends to an approximation of the product  $\prod_{k=1}^Ix_k^{\alpha_i}$
    for $\alpha_i \in \Z_+$, $i \in \{1, \dots, I\}$ and $\sum_{i=1}^I\alpha_i = d$.
\end{Lem}

\begin{Lem}[\cite{oko2023diffusion}, Lemma F.12]
\label{lem:exp_minus_oko}
    For any $\eps_0 > 0$, there exists a ReLU-network $\phi_\mathrm{exp} \in \NN(L, W ,S, B)$ such that 
    \begin{align*}
        \sup_{x, x' \geq 0} \left|e^{-x'} - \phi_\mathrm{exp}(x) \right| \leq \eps_0 + |x - x'|
    \end{align*}
    holds, where $L \lesssim \log^2(1/\eps_0)$, $\|W\|_\infty \lesssim \log(1/\eps_0)$, $S \lesssim \log^2(1/\eps_0)$, $\log B \lesssim \log^2(1/\eps_0)$.
    Moreover, for all $x \geq \log(3/\eps_0)$ it holds that $|\phi_\mathrm{exp}(x)| \leq \eps_0$.
\end{Lem}

Lemma \ref{lem:exp_minus_oko} has an obvious corollary.

\begin{Co}
    \label{co:exp}
    For any $\eps_0 > 0$ and $a \geq 0$, there is a ReLU-network $\phi \in \NN(L, W, S, B)$ with
    $L \lesssim \log^2(1/\eps_0)$, $\|W\|_\infty \lesssim \log(1/\eps_0)$, $S \lesssim \log^2(1/\eps_0)$, $\log B \lesssim \log^2(1/\eps_0) + |a| \vee 1$ such that
    \begin{align*}
        \sup_{x \geq 0, \, x' \geq -a} \left|\phi(x') - e^{-x} \right| \leq e^a \left(\eps_0 + |x - x'| \right).
    \end{align*}
\end{Co}
\begin{proof}
    Let $\phi_{\mathrm{exp}}$ be a ReLU-network from Lemma \ref{lem:exp_minus_oko} corresponding to the accuracy parameter $\eps_0$, and let $\phi(x) = e^a\phi_{\mathrm{exp}}(x + a)$. Obviously, $\phi$ has the configuration described in the corollary statement. Besides, Lemma \ref{lem:exp_minus_oko} yields that 
    \begin{align*}
        \sup_{x \geq 0, \, x' \geq -a}|\phi(x') - e^{-x}|
        = e^a \sup_{x \geq 0, \, x' \geq -a}|\phi_{\mathrm{exp}}(x' + a) - e^{-(x + a)}|
        \leq e^a(\eps + |x - x'|).
    \end{align*}

\end{proof}

In our proof, we also rely on the standard result of \cite{schmidt-hieber20} (Theorem 5). It plays a central role in the proof of Lemmata \ref{lem:exponent_integral_product_approximation} and \ref{lem:integral_approximation}.

\begin{Th}[\cite{schmidt-hieber20}, Theorem 5]
    \label{th:relu_approximation}
    For any function $f \in \cH^\alpha([0, 1]^r, H)$ and any integers $m \geq 1$ and $N \geq (\alpha + 1)^r \vee ((H + 1) e^r)$, there exists a network
    \[
        \widetilde f \in \NN(L, W, S, 1)
    \]
    with depth
    \[
        L = 8 + (m + 5) (1 + \lceil\log_2(r \vee \alpha)\rceil,
    \]
    width
    \[
        W = 6 (r \vee \lceil\alpha\rceil) N,
    \]
    and with at most
    \[
        S \leq 141 (r + \alpha + 1)^{3 + r} N (m + 6)
    \]
    non-zero parameters such that
    \[
        \|\widetilde f - f\|_{L^\infty([0, 1]^r)}
        \leq 6^r (2H + 1)(1 + r^2 + \alpha^2) N 2^{-m} +  3^\alpha H N^{-\alpha / r}.
    \]
\end{Th}

\begin{Rem}
    \label{rem:relu_approximation_corollary}
    Assume that $H$ from Theorem \ref{th:relu_approximation} is at least $1$ and take positive integers $\alpha \geq e (H + 1)^{1/r} - 1$, $N = (\alpha + 1)^r$, and
    \[
        m
        = \left\lceil(\alpha + r) \log_2(1 + \alpha) + r \log_2 6 + \log_2\left(1 + r^2 + \alpha^2 \right) \right\rceil.
    \]
    Then
    \[
        (2H + 1)(1 + r^2 + \alpha^2) 6^r N 2^{-m}
        \leq 3H (1 + r^2 + \alpha^2) 6^r N 2^{-m}
        \leq 3H (\alpha + 1)^{-\alpha}, 
    \]
    and the function $\widetilde f$ satisfies
    \[
        \|\widetilde f - f\|_{L^\infty([0, 1]^r)}
        \leq 3H (\alpha + 1)^{-\alpha} + 3^\alpha H (\alpha + 1)^{-\alpha}
        \leq H \left( \frac{3}{\alpha + 1} \right)^{\alpha + 1}.
    \]
\end{Rem}

Finally, Lemmata \ref{lem:m_t_over_sigma_t_approx}, \ref{lem:y_sq_over_sigma_approx}, and \ref{lem:mt_y_dot_a_sigt} are used in the proof of Theorem \ref{th:score_approximation} to approximate the functions $\cV(t)$, $V_{\bj, 0}$, and $V_{\bj, \bk}$, where $\bj \in \{1, \dots, N\}^d$ and $\bk \in \Z_+^d$, $1 \leq |\bk| \leq \lfloor\beta\rfloor$. We provide their proofs in Appendices \ref{sec:m_t_over_sigma_t_approx_proof}, \ref{sec:y_sq_over_sigma_approx_proof}, and \ref{sec:mt_y_dot_a_sigt_proof}, respectively.

\begin{Lem}\label{lem:m_t_over_sigma_t_approx}
    Let us fix an arbitrary $\gamma \in \{0, 1, 2\}$.
    Then, for any $\eps \in (0, 1]$, there exists a ReLU neural network $\chi_{\gamma, \eps} \in \NN(L, W, S, B)$ such that 
    \begin{align}\label{eq:mt_over_st_final_acc}
        \sup\limits_{t \in [t_0, T]} \left\|\chi_{\gamma, \eps}(t) - \frac{m_t^\gamma}{\tilde{\sigma}_t^2} \right\|_{L^\infty([t_0, T])} \leq \eps
    \end{align}
    and its configuration satisfies the inequalities 
    \begin{align}\label{eq:mt_over_st_final_cfg}
        L \vee \log B
        &\notag
        \lesssim \log^2(1 / \eps) + \log^2(\tilde{\sigma}_{t_0}^{-2}),
        \\
        \|W\|_\infty
        &
        \lesssim \left(\frac{1}{t_0 + \sdata^2}\vee 1\right)\left(\log^2(1 / \eps) + \log^2(\tilde{\sigma}_{t_0}^{-2}) \right),
        \\
        S
        &\notag
        \lesssim \left(\frac{1}{t_0 + \sdata^2}\vee 1\right)\left(\log^3(1 / \eps) + \log^3(\tilde{\sigma}_{t_0}^{-2}) \right).
    \end{align}
\end{Lem}

\begin{Lem}\label{lem:y_sq_over_sigma_approx}
    Let $y \in \R^D$, $M \geq 1$ and $t_0 \leq T$.
    Then, for any $\eps \in (0, 1]$ there exists a ReLU-network $\rho_\eps(y, t) \in \NN(L, W, S, B)$ such that
    \begin{align*}
        \left|\frac{\|y\|^2}{2\tilde{\sigma}_t^2} - \rho_\eps(y, t) \right| \leq \eps
        \quad \text{for all $\|y\|_\infty \leq M$, and $t \in [t_0, T]$.}
    \end{align*}
    Furthermore, the configuration of $\rho_\eps(y, t)$ satisfies
    \begin{align*}
        L \vee \log B &\lesssim \log^2(1/\eps) + \log^2(MD) + \log^2(\tilde{\sigma}_{t_0}^{-2}), \\
        \|W\|_\infty &\lesssim D\left(\frac{1}{t_0 + \sdata^2}\vee 1\right)\left(\log^2(1/\eps) + \log^2(MD) + \log^2(\tilde{\sigma}_{t_0}^{-2})\right), \\
        S &\lesssim D\left(\frac{1}{t_0 + \sdata^2}\vee 1\right)\left(\log^3(1/\eps) + \log^3(MD) + \log^3(\tilde{\sigma}_{t_0}^{-2})\right).
    \end{align*}
\end{Lem}

\begin{Lem}\label{lem:mt_y_dot_a_sigt}
    Let $y, a \in \R^D$, $M \geq 1$ and $t_0 \leq T$.
    Then, for any $\eps \in (0, 1]$ there exists a ReLU-network $\omega_\eps(y, t) \in \NN(L, W, S, B)$ such that
    \begin{align*}
        \left| \frac{m_t y^\top a}{\tilde{\sigma}_t^2} - \omega_\eps(y, t) \right| \leq \eps
        \quad \text{for any $\|y\|_\infty \leq M$ and $t \in [t_0, T]$.}
    \end{align*}
    In addition, $\omega_\eps(y, t)$ has the following configuration:
    \begin{align*}
        L \vee \log B &\lesssim \log^2(1/\eps) + \log^2(DM\|a\|_\infty \vee 1) + \log^2(\tilde{\sigma}_{t_0}^{-2}), \\
        \|W\|_\infty &\lesssim D\left(\frac{1}{t_0 + \sdata^2}\vee 1\right) \left(\log^2(1/\eps) + \log^2(DM\|a\|_\infty \vee 1) + \log^2(\tilde{\sigma}_{t_0}^{-2})\right), \\
        S &\lesssim D\left(\frac{1}{t_0 + \sdata^2}\vee 1\right) \left(\log^3(1/\eps) + \log^3(DM\|a\|_\infty \vee 1) + \log^3(\tilde{\sigma}_{t_0}^{-2})\right).
    \end{align*}
\end{Lem}

\subsection{Proof of Lemma \ref{lem:m_t_over_sigma_t_approx}}
\label{sec:m_t_over_sigma_t_approx_proof}

Let us introduce
\[
    \Delta_\sigma = -\frac{1}{2}\log(1 - \sdata^2) \geq 0.
\]
Note that $\Delta_\sigma$ is well defined due to the fact that $\sdata$ is strictly less than $1$.
Then, for any $\gamma \in \{0, 1, 2\}$ and any $t \in [t_0, T]$, it holds that 
\begin{align*}
    \frac{m_t^\gamma}{\tilde{\sigma}_t^2} = \frac{e^{-\gamma t}}{1 - e^{-2t}(1 - \sdata^2)}
    = \frac{e^{-\gamma t}}{1 - e^{-2(t - \frac{1}{2}\log(1 - \sdata^2))}}
    = \frac{e^{-\gamma t}}{1 - e^{-2(t + \Delta_\sigma)}}.
\end{align*}
We can represent the right-hand side as a converging series:
\begin{align*}
    \frac{e^{-\gamma t}}{1 - e^{-2(t + \Delta_\sigma)}} = e^{-\gamma t}\sum_{k=0}^\infty e^{-2k(t + \Delta_\sigma)}.
\end{align*}
Futhermore, if we take
\[
    r = \left\lceil \frac{\log(2 / \eps) + \log(\tilde{\sigma}_{t_0}^{-2})}{2(t_0 + \Delta_\sigma)} \right\rceil,
\]
then we obtain that
\begin{align}\label{eq:m_t_gamma_st_trunc}
    \sup_{t \in [t_0, T]}\left|\frac{m_t^\gamma}{\tilde{\sigma}_t^2} - e^{-\gamma t} \sum_{k=0}^{r - 1} e^{-2k(t + \Delta_\sigma)} \right|
    =\frac{e^{-\gamma t}e^{-2r(t + \Delta_\sigma)}}{1 - e^{-2(t + \Delta_\sigma)}}
    \leq \frac{e^{-\gamma t_0}e^{-2r(t_0 + \Delta_\sigma)}}{\tilde{\sigma}_{t_0}^2} \leq \frac{\eps}{2}.
\end{align}
The inequality \eqref{eq:m_t_gamma_st_trunc} means that it is enough to approximate each term in the sum
\[
    e^{-\gamma t} \sum\limits_{k = 0}^{r - 1} e^{-2k(t + \Delta_\sigma)}
    = e^{-2k \Delta_\sigma} \sum\limits_{k = 0}^{r - 1} e^{-(2k + \gamma) t}
\]
within the accuracy $\eps / (2r)$. According to Lemma \ref{lem:exp_minus_oko}, there is $\phi_{\exp} \in \NN(\tilde L, \tilde W, \tilde S, \tilde B)$ with
\[
    \tilde L \vee \tilde S \vee \log \tilde B
    \lesssim \log^2(1 / \eps) + \log^2(\tilde{\sigma}_{t_0}^{-2}) + \log^2(1 \vee 1 / (t_0 + \Delta_\sigma))
\]
and with
\[
    \|\tilde W\|_\infty
    \lesssim \log(1 / \eps) + \log(\tilde{\sigma}_{t_0}^{-2}) + \log(1 \vee 1 / (t_0 + \Delta_\sigma))
\]
such that 
\[
    \sup_{t \in [t_0, T]} \left| e^{-2k\Delta_\sigma} \phi_{\mathrm{exp}} ((\gamma + 2k)t) - e^{-\gamma t - 2k(t + \Delta_\sigma)} \right|
    \leq \frac{e^{-2k\Delta_\sigma}\eps}{2r}
    \quad \text{for all $k \in \{0, \dots, r - 1\}$}.
\]
Then the triangle inequality implies that
\begin{align}\label{eq:mt_st_q_approx}
    \sup_{t \in [t_0, T]} \left| \sum_{k=0}^{r - 1} e^{-\gamma t - 2k(t + \Delta_\sigma)} - \sum_{k=0}^{r - 1} e^{-2k\Delta_\sigma} \phi_{\mathrm{exp}}((\gamma + 2k)t) \right|
    \leq \sum_{k=0}^{r - 1} \frac{e^{-2k\Delta_\sigma} \eps}{2r}
    \leq \frac{\eps}{2}.
\end{align}
It is straightforward to observe that
\[
    \chi_{\gamma, \eps}(t) = \sum_{k=0}^{r - 1} e^{-2k\Delta_\sigma} \phi_{\mathrm{exp}}((\gamma + 2k)t)
\]
can be obtained by parallel stacking of the neural networks $e^{-2k\Delta_\sigma} \phi_{\mathrm{exp}} ((\gamma + 2k)t)$ where $k$ runs over $\{0, 1, \dots, r - 1\}$. Hence, 
\[
    \sum_{k=0}^{r - 1} e^{-2k\Delta_\sigma} \phi_{\mathrm{exp}}((\gamma + 2k)t)
    \in \NN(L, W, S, B)
\]
with configuration parameters satisfying the bounds
\[
    L \vee S \vee \log B
    \lesssim \log^2(1 / \eps) + \log^2(\tilde{\sigma}_{t_0}^{-2}) + \log^2(1 \vee 1 / (t_0 + \Delta_\sigma)),
\]
and
\[
    \|W\|_\infty
    \lesssim \log(1 / \eps) + \log(\tilde{\sigma}_{t_0}^{-2}) + \log(1 \vee 1 / (t_0 + \Delta_\sigma)).
\]
It only remains to note that Jensen's inequality and the definition of $\Delta_\sigma$ yield that 
\begin{align*}
    \log(1 \vee 1 / (t_0 + \Delta_\sigma)) \lesssim \log(1 \vee 1 / (t_0 + \sdata^2)) \lesssim \log(\tilde{\sigma}_{t_0}^{-2}).
\end{align*}
In other words, the configuration of $\chi_{\gamma, \eps}$
has the required form \eqref{eq:mt_over_st_final_cfg}.
Finally, applying the triangle inequality once again and taking \eqref{eq:m_t_gamma_st_trunc} and \eqref{eq:mt_st_q_approx} into account, we obtain that
\begin{align*}
    \sup_{t \in [t_0, T]} \left| \chi_{\gamma, \eps}(t) - \frac{m_t^\gamma}{\tilde\sigma_t^2} \right|
    \leq \frac{\eps}2 + \sup_{t \in [t_0, T]} \left| \sum_{k=0}^{r - 1} e^{-\gamma t - 2k(t + \Delta_\sigma)} - \sum_{k=0}^{r - 1} e^{-2k\Delta_\sigma} \phi_{\mathrm{exp}}((\gamma + 2k)t) \right|
    \leq \eps.
\end{align*}
This concludes the proof.

\endproof

\subsection{Proof of Lemma \ref{lem:y_sq_over_sigma_approx}}
\label{sec:y_sq_over_sigma_approx_proof}
Let $\phi_{\mathrm{mult}}(x_1', x_2')$ be the multiplication network from Lemma 
\ref{lem:multi_oko} such that
\[
    \left| \phi_{\mathrm{mult}}(x_1', x_2') - x_1 x_2 \right| \leq \frac{\eps}{2 D C} + 2 C \left(|x_1 - x_1'| \vee |x_2 - x_2'| \right)
    \quad \text{for all $x_1, x_2 \in [-C, C]$,}
\]
where $C = (DM^2/2) \vee \tilde{\sigma}_{t_0}^{-2} \vee 1$.
Clearly, $\phi_{\mathrm{mult}}$ belongs to the class $\NN(L_{\mathrm{mult}}, W_{\mathrm{mult}}, S_{\mathrm{mult}}, B_{\mathrm{mult}})$ with
\begin{equation}\label{eq:h2_y_st_cfg}
\begin{split}
    &
    L_{\mathrm{mult}} \lesssim \log(1/\eps) + \log(DM^2 \vee \tilde{\sigma}_{t_0}^{-2}) 
    \lesssim \log(1/\eps) + \log(MD) + \log(\tilde{\sigma}_{t_0}^{-2}),
    \\&
    \|W_{\mathrm{mult}}\|_\infty \lesssim 1,
    \\&
    S_{\mathrm{mult}}
    \lesssim \log(1/\eps) + \log(MD) + \log(\tilde{\sigma}_{t_0}^{-2}),
    \\&
    \log B_{\mathrm{mult}}
    = \log C^2 \lesssim \log(MD) + \log(\tilde{\sigma}_{t_0}^{-2}).
\end{split}
\end{equation}
In particular, since $\|y\|_\infty \leq M$ by the conditions of the lemma, it holds that
\[
    \sup\limits_{\|y\|_\infty \leq M} \left| \phi_{\mathrm{mult}}(y_j, y_j)  - y_j^2 \right|
    \leq \frac{\eps}{2 D C}
    \quad \text{for all $j \in \{1, \dots, D\}$ and all $y_j \in [-M, M]$.}
\]
More importantly, we can use $\phi_{\mathrm{mult}}$ to approximate the product $\|y\|^2 / (2 \tilde\sigma_t^2)$. Indeed, note that for all $\|y\|_\infty \leq M$ and all $t \in [t_0, T]$ both $\|y\|^2 / 2$ and $\tilde\sigma_t^{-2}$ belong to $[-C, C]$.
Let us take the neural network $\chi_{0, \eps_0}$ defined in Lemma \ref{lem:multi_oko} with $\eps_0 = \eps / 4C$ and denote
\[
    \rho_\eps(y, t) = \phi_{\mathrm{mult}}\left( \chi_{0, \eps_0}(t), \frac12 \sum\limits_{j = 1}^D \phi_{\mathrm{mult}}(y_j, y_j)\right).
\]
Then it is straightforward to check that
\begin{align*}
    \left| \rho_\eps(y, t) - \frac{\|y\|^2}{2 \tilde\sigma_t^2} \right|
    &
    \leq \frac{\eps}{2C} + 2C \max\left\{ \sup\limits_{t_0 \leq t \leq T} \left|\chi_{0, \eps_0}(t) - \frac1{\tilde\sigma_t^2} \right|, \frac12 \left| \sum\limits_{j = 1}^D \left(\phi_{\mathrm{mult}}(y_j, y_j) - y_j^2 \right) \right| \right\}
    \\&
    \leq \frac{\eps}2 + 2C \max\left\{ \eps_0, \frac{D}2 \cdot \frac{\eps}{2DC} \right\}
    \\&
    \leq \frac{\eps}2 + \frac{\eps}2
    = \eps
\end{align*}
for all $\|y\|_\infty \leq M$ and all $t \in [t_0, T]$.
It only remains to specify the configuration of $\rho_\eps(y, t)$.
First, note that
\[
    \frac12 \sum\limits_{j = 1}^D \phi_{\mathrm{mult}}(y_j, y_j)
\]
is obtained by parallel stacking of $D$ neural networks with configuration defined in \eqref{eq:h2_y_st_cfg}. This means that it has depth $L_{\mathrm{mult}}$, width $D \|W_{\mathrm{mult}}\|_\infty$, $D S_{\mathrm{mult}}$ non-zero weights, and the weight magnitude $B_{\mathrm{mult}}$. Recalling the configuration \eqref{eq:mt_over_st_final_cfg} of $\chi_{0, \eps_0}(t)$, we conclude that
\[
    \phi_{\mathrm{mult}}\left( \chi_{0, \eps_0}(t), \frac12 \sum\limits_{j = 1}^D \phi_{\mathrm{mult}}(y_j, y_j)\right)
    \in \NN(L, W, S, B),
\]
where the parameters $L$, $W$, $S$, and $B$ fulfil the inequalities
\begin{align*}
    L \vee \log B
    &
    \lesssim \log^2(1/\eps) + \log^2(MD) + \log^2(\tilde{\sigma}_{t_0}^{-2}),
    \\
    \|W\|_\infty(\hat{f}_\eps)
    &
    \lesssim D\left(\frac{1}{t_0 + \sdata^2}\vee 1\right)\left(\log^2(1/\eps) + \log^2(MD) + \log^2(\tilde{\sigma}_{t_0}^{-2})\right),
    \\
    S
    &
    \lesssim D\left(\frac{1}{t_0 + \sdata^2}\vee 1\right)\left(\log^3(1/\eps) + \log^3(MD) + \log^3(\tilde{\sigma}_{t_0}^{-2})\right).
\end{align*}
\endproof

\subsection{Proof of Lemma \ref{lem:mt_y_dot_a_sigt}}
\label{sec:mt_y_dot_a_sigt_proof}

The proof follows a similar approach to that of Lemma \ref{lem:y_sq_over_sigma_approx}.
Let $\psi_{\mathrm{mult}}(x_1', x_2')$ be the multiplication network from Lemma \ref{lem:multi_oko} such that
\[
    \left| \psi_{\mathrm{mult}}(x_1', x_2') - x_1 x_2 \right| \leq \frac{\eps}{2 D C} + 2 C \left(|x_1 - x_1'| \vee |x_2 - x_2'| \right)
    \quad \text{for all $x_1, x_2 \in [-C, C]$,}
\]
where $C = \tilde{\sigma}_{t_0}^{-2} \vee DM\|a\|_\infty$.
According to Lemma \ref{lem:multi_oko}, $\psi_{\mathrm{mult}} \in \NN(L_{\mathrm{mult}}, W_{\mathrm{mult}}, S_{\mathrm{mult}}, B_{\mathrm{mult}})$ with
\begin{equation}\label{eq:ya_st_h2_cfg}
\begin{split}
    &
    L_{\mathrm{mult}} \lesssim \log(1/\eps) + \log(\tilde{\sigma}_{t_0}^{-2}) + \log(DM\|a\|_\infty \vee 1),
    \\&
    \|W_{\mathrm{mult}}\|_\infty \lesssim 1,
    \\&
    S_{\mathrm{mult}} \lesssim \log(1/\eps) + \log(\tilde{\sigma}_{t_0}^{-2}) + \log(DM\|a\|_\infty \vee 1),
    \\&
    \log B_{\mathrm{mult}}
    \lesssim \log(\tilde{\sigma}_{t_0}^{-2}) + \log(DM\|a\|_\infty \vee 1).
\end{split}
\end{equation}
We are going to use $\psi_{\mathrm{mult}}$ to approximate the product of $y^\top a$ and $m_t / \tilde\sigma_t^2$. Note that the conditions of the lemma ensure that
\[
    |y^\top a| \leq D \|y\|_\infty \|a\|_\infty \leq D M \|a\|_\infty
    \quad \text{and} \quad
    0 \leq \frac{m_t}{\tilde\sigma_t^2} \leq \frac{1}{\tilde\sigma_{t_0}^2},
\]
or, in other words, both $y^\top a$ and $m_t / \tilde\sigma_t^2$ belong to $[-C, C]$ for all admissible $y$ and $t$. Let $\chi_{1, \eps / 2}(t)$ be the neural network from Lemma \ref{lem:m_t_over_sigma_t_approx}. Then
\[
    \omega_\eps(y, t) = \psi_{\mathrm{mult}} \left(y^\top a, \chi_{1, \eps / 2}(t) \right)
\]
satisfies
\begin{align*}
    \left| \omega_\eps(y, t) - \frac{m_t y^\top a}{\tilde\sigma_t^2} \right|
    \leq \frac{\eps}{2} + 2 C \sup\limits_{t_0 \leq t \leq T} \left|\chi_{1, \eps_1}(t) - \frac{m_t}{\tilde\sigma_t^2} \right|
    \leq \frac{\eps}{2} + \frac{\eps}{2}
    = \eps.
\end{align*}
Moreover, taking into account \eqref{eq:ya_st_h2_cfg} and the configuration of $\chi_{1, \eps / 2}(t)$ (see \eqref{eq:mt_over_st_final_cfg}), we conclude that
$\omega_\eps(y, t)$ lies in $\NN(L, W, S, B)$ with
\begin{align*}
    L \vee \log B
    &
    \lesssim \log^2(1/\eps) + \log^2(DM\|a\|_\infty \vee 1) + \log^2(\tilde{\sigma}_{t_0}^{-2}),
    \\
    \|W\|_\infty
    &
    \lesssim D\left(\frac{1}{t_0 + \sdata^2}\vee 1\right) \left(\log^2(1/\eps) + \log^2(DM\|a\|_\infty \vee 1) + \log^2(\tilde{\sigma}_{t_0}^{-2})\right),
    \\
    S
    &
    \lesssim D\left(\frac{1}{t_0 + \sdata^2}\vee 1\right) \left(\log^3(1/\eps) + \log^3(DM\|a\|_\infty \vee 1) + \log^3(\tilde{\sigma}_{t_0}^{-2})\right).
\end{align*}
\endproof

\section{Tools from probability theory}

This section collects a couple of useful results from probability theory used in the proof of Theorem \ref{th:score_approximation}. 

\begin{Prop}[\cite{azangulov2024convergence}, Proposition 23]
    \label{prop:azangulov}
    Let $\sfP$ and $\sfQ$ be arbitrary compactly supported measures such that $W_2(\sfP, \sfQ) < \infty$, where $W_2$ stands for the Kantorovich distance. Given independent random elements $X \sim \sfP$, $Y \sim \sfQ$, $Z \sim \cN(0, I_D)$ in $\R^D$, let $X_t = c_t X + \sigma_t Z_D$ and $Y_t = c_t Y + \sigma_t Z_D$ be two random processes initialized at $X$ and $Y$, respectively. For any $t \geq 0$, let $\sfp_t$ and $\sfq_t$ stand for the probability density functions (with respect to the Lebesgue measure in $\R^D$) of $X_t$ and $Y_t$, respectively. Then, for any $t_{\max} \geq t_{\min} \geq 0$, it holds that
    \[
        \int\limits_{t_{\min}}^{t_{\max}} \int\limits_{\R^D} \left\| \nabla \log \sfp_t(x) - \nabla \log \sfq_t(x) \right\|^2 \, \sfp_t(x)\, \dd x \, \dd t
        \leq W_2^2(\sfP, \sfQ) \, \frac{c_{t_{\min}}^2}{4 \sigma_{t_{\min}}^2}.
    \]
\end{Prop}

\begin{Rem}
    In \citep{azangulov2024convergence}, the authors assumed that $c_t = e^{-t}$ and $\sigma_t^2 = 1 - e^{-2t}$. However, careful inspection of the proof reveals that Proposition 23 from \citep{azangulov2024convergence} remains valid for arbitrary $c_t$ and $\sigma_t$.
\end{Rem}

\begin{Prop}[\cite{wainwright19}, Proposition 2.2]
    \label{prop:sub-exp_concentration}
    Suppose that $X$ is a sub-exponential random variable with parameters $(\nu, b)$, that is
    \[
        \E e^{\lambda (X - \E X)} \leq e^{\nu^2 \lambda^2 / 2}
        \quad \text{for all $\lambda$ such that $|\lambda| \leq 1/b$.}
    \]
    Then, for any $t \geq 0$, it holds that
    \[
        \p\left(X \geq \E X + t \right)
        \leq \exp\left\{-\frac12 \left( \frac{t^2}{\nu^2} \land \frac{t}{b} \right) \right\}.
    \]
\end{Prop}

\begin{Rem}
    \label{rem:chi-squared_concentration}
    According to \citep[Example 2.5]{wainwright19}, chi-squared random variable with $D$ degrees of freedom is sub-exponential with parameters $(2\sqrt{D}, 4)$. This yields that, if $X \sim \chi^2(D)$, then
    \[
        \p\left(X \geq D + t \right)
        \leq \exp\left\{-\frac18 \left( \frac{t^2}{D} \land t \right) \right\}
        \quad \text{for all $t \geq 0$.}
    \]
\end{Rem}

\end{document}